\documentclass[11pt]{article}
\usepackage[margin=1in]{geometry}

\usepackage{caption}
\usepackage{subcaption}
\usepackage{amsmath, amssymb} 
\usepackage{tikz}
\usepackage{mathtools}
\usepackage{amsthm}
\usepackage{enumitem}

\usepackage[authoryear,round]{natbib}

\usepackage{algorithm}
\usepackage{algorithmic}

\newtheorem{theorem}{Theorem}

\newtheorem{lemma}{Lemma}
\newtheorem{proposition}{Proposition}

\newtheorem{corollary}{Corollary}

\newtheorem{remark}{Remark}
\newtheorem{definition}{Definition}

\newtheorem{example}{Example}

\usepackage[dvipsnames]{xcolor}        

\usepackage[colorlinks=true, linkcolor={rgb,255:red,0;green,20;blue,115}, citecolor={rgb,255:red,0;green,20;blue,115}]{hyperref}

\usepackage{cleveref}

\crefformat{subfigure}{#2Fig.~#1#3}
\Crefformat{subfigure}{#2Fig.~#1#3}

\crefname{theorem}{Theorem}{Theorems}
\crefname{lemma}{Lemma}{Lemmas}
\crefname{property}{Property}{Properties}
\crefname{assumption}{Assumption}{Assumptions}
\crefname{proposition}{Proposition}{Propositions}
\crefname{property}{Property}{Properties}
\crefname{corollary}{Corollary}{Corollaries}
\crefname{figure}{Fig.}{Figs.}
\crefname{section}{Section}{Sections}
\crefname{definition}{Definition}{Definitions}
\crefname{table}{Table}{Tables}

\crefname{algorithm}{Algorithm}{Algorithms}
\crefname{example}{Example}{Examples}
\crefname{remark}{Remark}{Remarks}
\crefname{appendix}{Appendix}{Appendices}
\crefname{section}{Section}{Sections}

\usepackage{float}
\usepackage{placeins}

\usepackage{titletoc}

\usepackage[textsize=tiny]{todonotes}

\usepackage{adjustbox}

\usepackage{centernot}

\newcommand{\CI}{\mathrel{\perp\mspace{-10mu}\perp}}
\newcommand{\nCI}{\centernot{\CI}}

\newdimen\arrowsize

\pgfarrowsdeclare{arcsq}{arcsq}
{
	\arrowsize=0.5pt
	\advance\arrowsize by .5\pgflinewidth
	\pgfarrowsleftextend{-10\arrowsize-.10\pgflinewidth}
	\pgfarrowsrightextend{.10\pgflinewidth}
}
{
	\arrowsize=0.5pt
	\advance\arrowsize by .5\pgflinewidth
	\pgfsetdash{}{0pt} 
	\pgfsetroundjoin   
	\pgfsetroundcap    
	\pgfpathmoveto{\pgfpoint{0\arrowsize}{0\arrowsize}}
	\pgfpatharc{-90}{-130}{7\arrowsize}
	\pgfusepathqstroke
	\pgfpathmoveto{\pgfpointorigin}
	\pgfpatharc{90}{130}{7\arrowsize}
	\pgfusepathqstroke
}
\pgfarrowsdeclare{arcsqblue}{arcsqblue}
{
	\arrowsize=0.5pt 
	\advance\arrowsize by .5\pgflinewidth
	\pgfarrowsleftextend{-12\arrowsize-.10\pgflinewidth}
	\pgfarrowsrightextend{.10\pgflinewidth}
}
{
	\arrowsize=0.8pt 
	\advance\arrowsize by .5\pgflinewidth
	\pgfsetdash{}{0pt} 
	\pgfsetroundjoin   
	\pgfsetroundcap    
	\pgfsetcolor{blue} 
	\pgfsetlinewidth{1.5pt} 
	\pgfpathmoveto{\pgfpoint{0\arrowsize}{0\arrowsize}}
	\pgfpatharc{-90}{-130}{6\arrowsize}
	\pgfusepathqstroke
	\pgfpathmoveto{\pgfpointorigin}
	\pgfpatharc{90}{130}{6\arrowsize}
	\pgfusepathqstroke
}

\tikzset{
    every path/.style={line width=0.75pt},
    every node/.style={font=\normalsize,inner sep=0pt,
        text=black,
        text opacity=1},
    selection/.style={
        draw,
        rectangle,
        fill=black!50,
        minimum width=5.5mm,
        minimum height=5.5mm,text=white,
    },
    latent/.style={
        draw,
        circle,
        dash pattern={on 3.0pt off 2.5pt},
        fill=black!50,
        line width=1.0pt, minimum size=5.5mm,text=white,
    },
    obs/.style={
        circle, draw=Gray, fill=Gray!10, line width=0.8pt, minimum size=5.5mm
    },
    geneobs/.style={
        rectangle,
        rounded corners=2pt,
        draw=Gray,
        fill=Gray!10,
        line width=0.8pt,
        minimum width=7.5mm,
        minimum height=5.5mm
    },
   target/.style={
        circle, 
        draw=OrangeRed, 
        fill=OrangeRed!10, 
        line width=1.5pt, 
        minimum size=5.5mm, 
        inner sep=0pt
   },
   genetarget/.style={
        rectangle,
        rounded corners=2pt,
        draw=OrangeRed,
        fill=OrangeRed!10,
        line width=1.5pt,
        minimum width=7.5mm,
        minimum height=5.5mm,
        inner sep=0pt
    },
    mb/.style={
        circle, 
        draw=RoyalBlue, 
        fill=RoyalBlue!10, 
        line width=1pt, 
        minimum size=5.5mm, 
        inner sep=0pt
    }
    }
\usetikzlibrary{arrows.meta}

\makeatletter
\pgfarrowsdeclare{circle}{circle}
{
    \pgfarrowsleftextend{0pt}
    \pgfarrowsrightextend{6\pgflinewidth}
}
{
    \pgfpathcircle{\pgfqpoint{1.5pt}{0pt}}{1.5pt} 
    \pgfusepathqstroke 
}
\makeatother

\usepackage{xspace}
\newcommand{\ie}{i.e.\@\xspace}           
\newcommand{\method}{{\textsc{LoCaLS}}\@\xspace}  

\DeclareMathOperator{\An}{\mathrm{An}}
\DeclareMathOperator{\Anplus}{\mathrm{An^{+}}}
\DeclareMathOperator{\Ant}{\mathrm{Ant}}
\DeclareMathOperator{\Antplus}{\mathrm{Ant^{+}}}
\DeclareMathOperator{\Pa}{\mathrm{Pa}}

\DeclareMathOperator{\MB}{\mathrm{MB}}
\DeclareMathOperator{\MBplus}{\mathrm{MB^{+}}}

\DeclareMathOperator{\ArrColl}{\mathrm{ArrColl}}
\DeclareMathOperator{\Sepset}{\mathrm{Sepset}}
\DeclareMathOperator{\Sepsets}{\mathbf{Sepsets}}

\DeclareMathOperator{\target}{T}

\DeclareMathOperator{\MAG}{\mathcal{M}}
\DeclareMathOperator{\PAG}{\mathcal{P}}
\DeclareMathOperator{\DAG}{\mathcal{D}}
\newcommand{\aug}[1][]{(#1)^a}

\newcommand{\vars}[1][V]{\mathbf{#1}}

\newcommand{\g}[1][G]{\mathcal{#1}}  

\newcommand{\LocalP}[1][X]{\mathcal{L}_{#1}}

\usepackage{threeparttable}
\usepackage{booktabs}    
\usepackage{multirow}    
\usepackage{xcolor}      
\usepackage{colortbl}    
\usepackage{pifont}      
\usepackage{fontawesome5}
\usepackage{arydshln} 
\usepackage{array}
\newcolumntype{C}[1]{>{\centering\arraybackslash}m{#1}}

\newcommand{\cmark}{\textcolor{teal}{\ding{51}}} 
\newcommand{\xmark}{\textcolor{gray!30}{\ding{55}}} 

\newcommand{\starterm}{\ding{72}}

\definecolor{lightblue}{RGB}{220,240,244}

\newcommand{\mycommfont}[1]{%
    {\small\ttfamily\textcolor{gray}{#1}}%
}

\newcommand{\algcomment}[1]{%
    \STATE \mycommfont{$\triangleright$~#1}%
}

\newlength{\localfigbodyheight}
\newcommand{\fixedtikzbox}[1]{%
    \begin{minipage}[b][\localfigbodyheight][c]{\linewidth}
        \centering
        \makebox[\linewidth][c]{#1}%
    \end{minipage}%
}

\newlength{\seqfigbodyheight}
\newcommand{\fixedseqtikzbox}[1]{%
    \begin{minipage}[c][\seqfigbodyheight][c]{\linewidth}
        \centering
        \makebox[\linewidth][c]{#1}%
    \end{minipage}%
}

\date{}

\begin{document}

\title{Local Causal Structure Learning in the Presence of Latent Variables and Selection Bias}

\author{Zheng Li\textsuperscript{1}, Hao Zhang\textsuperscript{1}, Ruxin Wang\textsuperscript{1}, Ruichu Cai\textsuperscript{3,4}, Kun Zhang\textsuperscript{5,6}, and Feng Xie\textsuperscript{2}%
}

\maketitle

\begingroup
\makeatletter
\insert\footins{%
    \reset@font\footnotesize
    \interlinepenalty\interfootnotelinepenalty
    \splittopskip\footnotesep
    \splitmaxdepth\dp\strutbox
    \floatingpenalty\@MM
    \hsize\columnwidth
    \@parboxrestore
    \noindent
    \textsuperscript{1}Shenzhen Institutes of Advanced Technology, Chinese Academy of Sciences, Shenzhen, China.
    \textsuperscript{2}Department of Applied Statistics, Beijing Technology and Business University, Beijing, China.
    \textsuperscript{3}School of Computer Science, Guangdong University of Technology, Guangzhou, China.
    \textsuperscript{4}Peng Cheng Laboratory, Shenzhen, China.
    \textsuperscript{5}Department of Philosophy, Carnegie Mellon University, Pittsburgh, USA.
    \textsuperscript{6}Mohamed bin Zayed University of Artificial Intelligence, Abu Dhabi, United Arab Emirates.
    Corresponding authors: Hao Zhang and Feng Xie (e-mails: h.zhang10@siat.ac.cn; fengxie@btbu.edu.cn).\par
}
\makeatother
\endgroup

\begin{abstract}
Discovering the direct causes and effects of a target variable from observational data is a fundamental problem in causal discovery, with broad applications in domains such as gene regulatory analysis and biomedical research.
Existing causal discovery methods either learn a global causal structure, which incurs substantial computational cost, or assume the absence of latent variables and selection bias, assumptions that are often violated in real-world settings.
Motivated by these challenges, we study local causal structure learning in the presence of latent variables and selection bias.
Specifically, we first characterize a local region that enables target-specific causal discovery without recovering the entire global structure. We then establish a theoretical bridge between causal information learned from the observed distribution induced on this local region and the corresponding information in the global causal structure. Building on these foundations, we propose \method, a local causal structure learning algorithm that is sound and complete under standard assumptions and identifies the same direct causes and effects of a target variable as those identifiable by global causal discovery methods, while allowing for latent variables and selection bias.
Extensive experiments on random and real-world structures demonstrate that the proposed method consistently achieves higher structural accuracy than existing local methods while requiring substantially less computational effort than state-of-the-art global methods. Furthermore, applications to two real-world gene expression datasets reveal biologically plausible target-specific causal structures, demonstrating its practical applicability in large-scale biological data analysis.
\end{abstract}


\section{Introduction}
\label{sec:introduction}


Learning causal relations, commonly known as causal discovery, has attracted increasing attention across many fields, including computer science \citep{peters2017elements,pearl2018theoretical,scholkopf2022causality,chen2025mecd,zhu2026cfsm}, social science \citep{spirtes2000causation}, epidemiology \citep{hernan2010causal}, biology \citep{glymour2019review}, and neuroscience \citep{smith2011network,sanchez2019estimating}. Such causal relationships are useful for predicting the behavior of a system under external interventions, which is a crucial step in both understanding and manipulating complex systems~\citep{pearl2009causality}. 
In observational studies, researchers are often interested in \emph{discovering the causes and effects of a target variable rather than discovering a global causal structure}~\citep{walters2007genome,peter2011gene,geng2019evaluation,ma2023local}, such as finding the causes of lung cancer, the effects of smoking, the causes and effects of hypertension, and the regulatory causes and effects of a particular gene~\citep{geng2019evaluation}.

\begin{figure}[t!]
    \centering
    \includegraphics[trim=0.0cm 3.5cm 0.0cm 0.0cm, clip, width=0.55\textwidth, page=1]{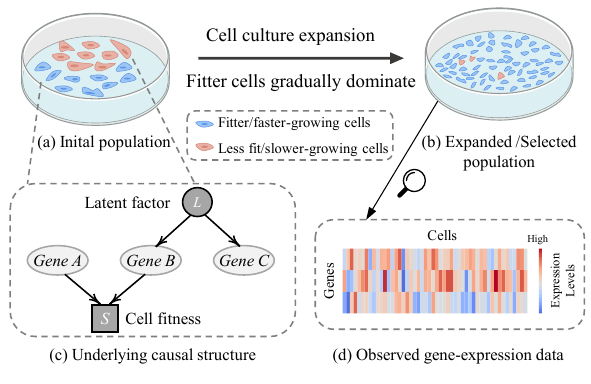}
  
    \caption{A Motivating example of local causal discovery in gene expression analysis under latent variables and selection bias.}
    \label{fig:motivating_examples}
    \vspace{-6mm}
\end{figure}

Driven by such practical demands, a growing body of work has focused on local causal discovery to efficiently identify the direct causes and effects of a target variable. Among the early contributions, \cite{yin2008partial} and \cite{zhou2010discover} introduced the PCD-by-PCD algorithm, where PCD denotes parents, children, and some descendants. This method leverages the MMPC algorithm~\citep{tsamardinos2006max} to extract the target's initial PCD set, and sequentially expands the structural boundary by computing the PCD sets of newly discovered adjacent variables. Throughout this expansion, local V-structures are identified, enabling the progressive orientation of edges incident to the target via Meek's rules~\citep{meek1995causal}. Building upon this foundation, \cite{wang2014discovering} proposed a more efficient MB-by-MB (Markov Blanket) approach. Additional significant contributions to this field include the CMB (Causal Markov Blanket) algorithm \citep{gao2015local}, the ELCS (Efficient Local Causal Structure) algorithm \citep{yang2021towards}, the PSL (Partial Bayesian Network Structure Learning) algorithm~\citep{ling2022psl}, and the GraN-LCS (GradieNt-based LCS) algorithm \citep{liang2023gradient}. 
Despite their success across various domains, the reliability of these methods typically relies on two restrictive assumptions:

\begin{itemize}[leftmargin=15pt]
    \item \textbf{No latent variables:} There is no latent confounding among the observed variables; that is, no unobserved variable serves as a common cause of two or more measured variables.
    \item \textbf{No selection bias:} The observed data are assumed to be collected without unobserved selection effects, so that the sample is representative of the population of interest.
\end{itemize}

However, in many real-world settings, such as clinical studies~\citep{richardson1998chain,zhang2008completeness} and gene expression studies~\citep{wille2004sparse,boquest2005isolation,frot2019robust,versteeg2022local}, these assumptions are often violated. 
\begin{example}[Latent Variables and Selection Bias in Gene Expression Analysis]
Consider a gene-expression study in which a population of cells is cultured and expanded over multiple generations before expression profiling~\citep{boquest2005isolation,kemmeren2014large,versteeg2022local}, as illustrated in \cref{fig:motivating_examples}(a)--(b). During this culture-expansion process, cells with higher survivability or faster growth rates gradually dominate the assayed population. Consequently, the measured gene-expression profiles are affected by selection bias, because the observed samples no longer represent the original cell population~\citep{hernan2004structural,boquest2005isolation,versteeg2022local}. Meanwhile, standard gene-expression assays do not capture all biologically relevant regulators. For instance, unmeasured upstream regulatory activities, such as the protein-level activity of transcription factors or transcription-factor complexes, can jointly regulate multiple observed genes but may remain unmeasured in expression profiles~\citep{stegle2012using,buettner2015computational,wei2018protein,buccitelli2020mrnas}.

The causal structure in \cref{fig:motivating_examples}(c) summarizes these two mechanisms, where $A$, $B$, and $C$ represent three distinct genes. The variable $S$ represents an unobserved selection variable that encodes cell fitness or survivability, and is directly affected by the expressions of genes $A$ and $B$. Thus, conditioning on the selected population can open the collider path $A \to S \leftarrow B$ and induce a spurious association between $A$ and $B$. The latent variable $L$ represents an unmeasured upstream regulatory activity and acts as a common cause of $B$ and $C$, thereby inducing latent confounding between the two observed genes. As a result, the observed gene-expression data in \cref{fig:motivating_examples}(d) contain both selection-induced and latent-confounding-induced associations. These mechanisms can further create apparent dependencies among $A$, $B$, and $C$ through paths involving both $S$ and $L$. If latent variables and selection bias are disregarded, such dependencies may be incorrectly identified as direct or indirect causal relations among the observed genes; for example, one could falsely conclude that $A$ is an (indirect) cause of both $B$ and $C$~\citep{versteeg2022local}.
\end{example}
This example shows that latent variables and selection bias are not merely technical complications, but pervasive features of real-world data collection processes. 
Therefore, overcoming these challenges is essential for reliable causal discovery in realistic settings~\citep{spirtes2016causal}.

A substantial body of work has studied causal discovery in the presence of latent variables and selection bias. \cite{spirtes1995fci} proposed the seminal FCI (Fast Causal Inference) algorithm, which recovers an equivalence class of causal structures in this setting by leveraging conditional independence (CI) relations derived from observational data. Subsequent work has largely focused on improving the efficiency of global discovery procedures. Representative examples include RFCI (Really Fast Causal Inference)~\citep{colombo2012rfci} and FCI$^{+}$~\citep{claassen2013learning}, which improve efficiency by relaxing certain CI search steps, making them more suitable for sparse, high-dimensional settings. More recently, \cite{rohekar2021iterative} proposed the ICD (Iterative Causal Discovery) algorithm, which reduces the number of CI tests by restricting conditioning sets according to the topological distance between the testing variables. Additionally, \cite{akbari2021recursive} introduced L-MARVEL, which reduces CI testing complexity by identifying and recursively removing a particular class of variables. Other related developments include \citep{pellet2008finding,claassen2011logical,claassen2013learning,mokhtarian2023novel,mokhtarian2025recursive}. Despite these important advances, existing methods in this line are primarily designed for global causal discovery, aiming to recover a causal structure over all observed variables. As a result, they can be unnecessarily expensive when the goal is only to identify the direct causes and effects of a single target variable.

Taken together, these limitations naturally raise the following important yet underexplored question:
\begin{center}
\begin{minipage}{0.95\linewidth}
~~\starterm~\textit{
How can we efficiently and accurately identify the direct causes and effects of a target variable from observational data when latent variables and selection bias may both be present?}
\end{minipage}
\end{center}
To address this question, we propose a novel local causal discovery algorithm that can efficiently identify the direct causes and effects of a target variable in the presence of latent confounding and selection bias. Specifically, we make the following contributions:
\begin{itemize}[leftmargin=15pt]
    \item We establish a theoretical foundation for local causal discovery in the presence of latent variables and selection bias. Specifically, we characterize an identifiable local region, and show which target-specific adjacencies and orientation information learned from this region are guaranteed to be consistent with the results of global causal discovery.
    
    \item Building on these foundations, we propose \method, a local causal structure learning algorithm for identifying the direct causes and effects of a target variable in the presence of latent confounding and selection bias. We prove that \method is sound and complete with respect to the target-specific information identifiable from the global learning.

    \item We conduct extensive experiments on synthetic random structures and real-world benchmark structures. The results show that \method achieves strong target-specific structural accuracy with higher efficiency compared to global methods, while outperforming existing local methods.
    We further apply \method to two real-world gene expression datasets. The results demonstrate its practical utility for recovering biologically plausible local causal structures in high-dimensional biological systems.

\end{itemize}


\section{Related Work}
\label{sec:related-work}

This work studies \emph{local causal structure} (LCS) learning in the presence of latent variables and selection bias. It is closely related to three lines of research: \emph{LCS learning}, \emph{global causal structure} (GCS) learning, and \emph{Markov blanket} (MB) learning. We briefly review these areas below and clarify how our problem setting differs from existing work. For comprehensive surveys of causal structure learning and MB learning, see~\citep{spirtes2016causal,heinze2018causal,glymour2019review,geng2019evaluation,yu2020causality,kitson2023survey,zanga2022survey}.

\textbf{LCS learning.}
Existing methods for LCS learning can be broadly divided into two categories. The first category is based on small local causal patterns such as Y-structures. Representative methods include the LCD algorithm~\citep{cooper1997simple} and its variants~\citep{silverstein2000scalable,mani2004causal,mani2006theoretical,versteeg2022local}. These methods typically infer causal relations from small local configurations involving only three or four variables. They are useful for detecting certain local causal patterns, but they do not aim to recover all causal edges incident to a target variable.
The second category aims to identify the direct causes and effects of a target variable without reconstructing the entire causal graph. Representative methods include PCD-by-PCD~\citep{yin2008partial,zhou2010discover}, MB-by-MB~\citep{wang2014discovering}, and a series of extensions~\citep{gao2015local,yang2021towards,ling2022psl,liang2023gradient,dai2024local,ling2025hybrid,zheng2026local}. 
Most methods in this line are developed under causal sufficiency and assume that the observed data are not affected by selection bias. More recently, \cite{xie2024local} and related methods~\citep{li2025local,ling2025local} extended local causal discovery to settings with latent variables. However, these methods do not explicitly handle selection bias.
Therefore, the problem considered in this paper is more general: we aim to learn the target-specific local causal structure while allowing both latent variables and selection bias.

\textbf{GCS learning.}
Under the assumptions that there are no latent variables and no selection bias, classical global causal discovery methods such as the IC algorithm~\citep{verma1990equivalence} and the PC algorithm~\citep{spirtes1991PC} recover a Markov equivalence class of DAGs from conditional independence relations. Numerous extensions have since been developed to improve computational efficiency, statistical robustness, and scalability~\citep{chickering2002GES,harris2013pc,colombo2014order,le2016multipc,ramsey2017fGES,rohekar2018bayesian,zheng2018notears,yu2019dag,mokhtarian2021recursive,mokhtarian2022learning,shiragur2024causal,zhang2024towards,dong2025intersecting}. When latent variables and selection bias are allowed, the seminal FCI algorithm~\citep{spirtes1995fci,zhang2008completeness} extends PC to this more general setting and is sound and complete under standard assumptions. Building on FCI, several methods have been proposed to improve the efficiency of global discovery, including RFCI~\citep{colombo2012learning} and FCI$^{+}$~\citep{claassen2013learning}. Other relevant developments include~\citep{pellet2008finding,claassen2011logical,ogarrio2016hybrid,tsirlis2018scoring,chen2023causal,wang2023sound}, as well as more recent iterative and recursive techniques~\citep{rohekar2021iterative,rohekar2023temporal,akbari2021recursive,mokhtarian2025recursive}. These methods are designed for global causal discovery, namely, recovering a causal structure over all observed variables. By contrast, when the goal is only to identify the local causal structure around a single target variable, recovering a global structure may be unnecessarily costly, both computationally and statistically.

\textbf{MB learning.}
MB learning is closely related to LCS learning because the Markov blanket can provide a compact local search space for identifying the direct causes and effects of a target variable. A large number of MB learning algorithms have been developed, including the GS (Grow-Shrink) algorithm~\citep{margaritis1999bayesian}, IAMB~\citep{tsamardinos2003algorithms}, and its variants such as inter-IAMB~\citep{tsamardinos2003algorithms}, Fast-IAMB~\citep{yaramakala2005speculative}, and KIAMB~\citep{pena2007towards}. Another important line of work adopts divide-and-conquer strategies, including MMMB (Min-Max MB)~\citep{tsamardinos2003time}, HITON-MB~\citep{aliferis2003hiton}, PCMB~\citep{pena2007towards}, and STMB~\citep{gao2016efficient}. Other advances include~\citep{pellet2008using,yu2019multi,wu2019accurate,ling2022light}. However, most of these methods are developed under the assumptions that there are no latent variables and no selection bias. \cite{yu2018mining} proposed M3B for MB discovery in the presence of latent confounders, but it does not handle selection bias or recover the target-specific causal structure. Therefore, although MB learning is useful for narrowing the local search space, existing MB methods do not address the problem of recovering the local causal structure of a target variable in the presence of both latent variables and selection bias.

To the best of our knowledge, no existing local causal structure learning method provides soundness and completeness guarantees for identifying target-specific causal relations in the presence of both latent variables and selection bias.

\section{Preliminaries}
\label{sec:preliminaries}

This section introduces the notation and graphical background used throughout the paper. 
We begin in \cref{sec:scm} by describing the structural causal model (SCM) associated with an underlying DAG, where latent variables and selection variables may be present. 
We then review maximal ancestral graphs (MAGs) and partial ancestral graphs (PAGs) in \cref{sec:mag-pag}, which provide the standard graphical representations of the observed-variable structure in the presence of latent variables and selection bias. 
Finally, in \cref{sec:problem-formulation}, we formalize the local causal structure learning problem considered in this paper.

Throughout this paper, individual variables, or vertices, are denoted by uppercase letters, e.g., $V$, while sets of variables, or vertices, are denoted by bold uppercase letters, e.g., $\mathbf{V}$. Graphs are denoted by calligraphic letters, e.g., $\g$. Additionally, \cref{Appendix:Preliminaries} provides a summary of the main notation used throughout the paper in \cref{table:list-symbols}, together with additional terminology and formal definitions.

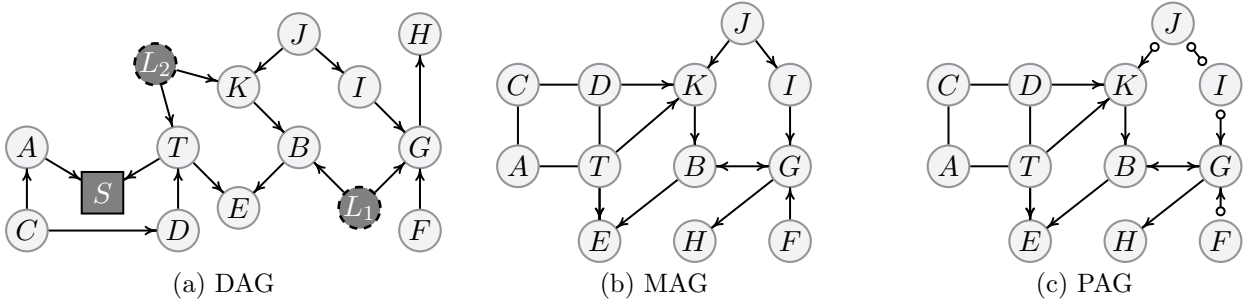
\begin{figure}[t!]
    \centering
    \captionsetup[subfigure]{font=small, skip=2pt}

    \begin{subfigure}[c]{0.31\linewidth}
        \centering
        \fixedtikzbox{%
        \begin{tikzpicture}[x=1.0cm,y=1.0cm]
              \draw (0.0, 0.0) node(A) [obs] {$A$};
              \draw (2.0, 0.0) node(T) [obs] {$T$};
              \draw (3.6, 0.0) node(B) [obs] {$B$};
              \draw (5.2, 0.0) node(G) [obs] {$G$};
              
              \draw (0.0, -1.1) node(C) [obs] {$C$};
              \draw (1.0, -0.6) node(S) [selection] {$S$};
              \draw (2.0, -1.1) node(D) [obs] {$D$};
              \draw (2.8, -0.8) node(E) [obs] {$E$};
              \draw (4.4, -0.8) node(L1) [latent] {$L_{1}$};
              \draw (5.2, -1.1) node(F) [obs] {$F$};
              \draw (1.7, 1.1) node(L2) [latent] {$L_{2}$};
              \draw (2.8, 0.8) node(K) [obs] {$K$};
              \draw (4.4, 0.8) node(I) [obs] {$I$};
            
              \draw (5.2, 1.5) node(H) [obs] {$H$};
            
              \draw (3.6, 1.5) node(J) [obs] {$J$};
              \draw[-arcsq] (A) -- (S);
              \draw[-arcsq] (T) -- (S);
              \draw[-arcsq] (C) -- (A);
              \draw[-arcsq] (C) -- (D);
              \draw[-arcsq] (D) -- (T);
            
              \draw[-arcsq] (T) -- (E);
              \draw[-arcsq] (B) -- (E);
            
              \draw[-arcsq] (L1) -- (B);
              \draw[-arcsq] (L1) -- (G);
              \draw[-arcsq] (F) -- (G);
            
              \draw[-arcsq] (L2) -- (T);
              \draw[-arcsq] (L2) -- (K);
            
              \draw[-arcsq] (J) -- (K);
              \draw[-arcsq] (J) -- (I);
            
              \draw[-arcsq] (K) -- (B);
              \draw[-arcsq] (I) -- (G);
            
              \draw[-arcsq] (G) -- (H);
        \end{tikzpicture}%
        }
        \caption{DAG}
        \label{fig:dag}
    \end{subfigure}
    \hfill
    \begin{subfigure}[c]{0.31\linewidth}
        \centering
        \fixedtikzbox{%
        \begin{tikzpicture}[x=0.9cm,y=0.9cm]
              \draw (0.0, 0.0) node(A) [obs] {$A$};
              \draw (1.2, 0.0) node(T) [obs] {$T$};
              \draw (2.6, 0.0) node(B) [obs] {$B$};
              \draw (4.0, 0.0) node(G) [obs] {$G$};
              
              \draw (0.0, 1.2) node(C) [obs] {$C$};
              \draw (1.2, 1.2) node(D) [obs] {$D$};
              \draw (1.2, -1.1) node(E) [obs] {$E$};
              \draw (2.6, -1.1) node(H) [obs] {$H$};
              \draw (4.0, -1.1) node(F) [obs] {$F$};
              \draw (2.6, 1.2) node(K) [obs] {$K$};
              \draw (4.0, 1.2) node(I) [obs] {$I$};
              \draw (3.3, 2.1) node(J) [obs] {$J$};
              \draw[-] (C) -- (A);
              \draw[-] (C) -- (D);
              \draw[-] (D) -- (T);
              \draw[-] (A) -- (T);
            
              \draw[-arcsq] (D) -- (K);
              \draw[-arcsq] (T) -- (K);
              \draw[-arcsq] (T) -- (E); 
            
              \draw[-arcsq] (T) -- (E);
              \draw[-arcsq] (B) -- (E);
              \draw[arcsq-arcsq] (B) -- (G);
            
              \draw[-arcsq] (F) -- (G);

              \draw[-arcsq] (J) -- (K);
              \draw[-arcsq] (J) -- (I);
            
              \draw[-arcsq] (K) -- (B);
              \draw[-arcsq] (I) -- (G);
            
              \draw[-arcsq] (G) -- (H);
        \end{tikzpicture}%
        }
        \caption{MAG}
        \label{fig:mag}
    \end{subfigure}
    \hfill
    \begin{subfigure}[c]{0.31\linewidth}
        \centering
        \fixedtikzbox{%
        \begin{tikzpicture}[x=0.9cm,y=0.9cm]
              \draw (0.0, 0.0) node(A) [obs] {$A$};
              \draw (1.2, 0.0) node(T) [obs] {$T$};
              \draw (2.6, 0.0) node(B) [obs] {$B$};
              \draw (4.0, 0.0) node(G) [obs] {$G$};
              
              \draw (0.0, 1.2) node(C) [obs] {$C$};
              \draw (1.2, 1.2) node(D) [obs] {$D$};
              \draw (1.2, -1.1) node(E) [obs] {$E$};
              \draw (2.6, -1.1) node(H) [obs] {$H$};
              \draw (4.0, -1.1) node(F) [obs] {$F$};
              \draw (2.6, 1.2) node(K) [obs] {$K$};
              \draw (4.0, 1.2) node(I) [obs] {$I$};
              \draw (3.3, 2.1) node(J) [obs] {$J$};
              \draw[-] (C) -- (A);
              \draw[-] (C) -- (D);
              \draw[-] (D) -- (T);
              \draw[-] (A) -- (T);
            
              \draw[-arcsq] (D) -- (K);
              \draw[-arcsq] (T) -- (K);
              \draw[-arcsq] (T) -- (E); 
            
              \draw[-arcsq] (T) -- (E);
              \draw[-arcsq] (B) -- (E);
              \draw[arcsq-arcsq] (B) -- (G);
            
              \draw[circle-arcsq] (F) -- (G);

              \draw[circle-arcsq] (J) -- (K);
              \draw[circle-circle] (J) -- (I);
            
              \draw[-arcsq] (K) -- (B);
              \draw[circle-arcsq] (I) -- (G);
            
              \draw[-arcsq] (G) -- (H);
        \end{tikzpicture}%
        }
        \caption{PAG}
        \label{fig:pag}
    \end{subfigure}

    \caption{(a) The underlying causal DAG for a subnetwork of the ANDES network \citep{conati1997line}, where $L_1$ and $L_2$ are latent variables and $S$ represents an unobserved selection variable. (b) The MAG corresponding to the DAG in (a). (c) The PAG representing the Markov equivalence class of the MAG in (b).}
    \vspace{-1.0em}
    \label{fig:andes-dag-mag-pag}
\end{figure}

\subsection{SCMs with Latent Variables and Selection Bias}
\label{sec:scm}

We consider a structural causal model (SCM)~\citep{pearl2000causality} over a set of variables
$
\mathbf{V}=\mathbf{O}\cup\mathbf{L}\cup\mathbf{S},
$
where $\mathbf{O}$, $\mathbf{L}$, and $\mathbf{S}$ denote the observed variables, latent variables, and unobserved selection variables, respectively. The SCM is associated with an underlying directed acyclic graph (DAG) $\DAG$ over $\mathbf{V}$. Each variable $V_i\in\mathbf{V}$ is generated according to a structural equation
$
V_i = f_i\bigl(\Pa(V_i,\DAG),u_i\bigr),
$
where $\Pa(V_i,\DAG)$ denotes the parents of $V_i$ in $\DAG$, and $u_i$ is an exogenous error term, and the exogenous errors are mutually independent. Accordingly, the joint distribution over $\mathbf{V}$ factorizes according to the underlying DAG:
\[
P(\mathbf{V})
=
P(\mathbf{O},\mathbf{L},\mathbf{S})
=
\prod_{V_i\in\mathbf{V}}
P\bigl(V_i\mid \Pa(V_i,\DAG)\bigr).
\]

Following the standard setting for causal discovery in the presence of latent variables and selection bias~\citep{spirtes1995fci,richardson2002ancestral,zhang2008completeness,colombo2012learning}, the latent variables $\mathbf{L}$ are unobserved and marginalized out, while the selection variables $\mathbf{S}$ are unobserved and conditioned upon. Thus, the distribution available from observational data is not the population marginal $P(\mathbf{O})$, but the selected distribution
\[ 
P_{\mathrm{obs}}(\mathbf{O}) 
= P(\mathbf{O}\mid \mathbf{S}=\mathbf{s}) 
= \int P(\mathbf{O},\mathbf{L}\mid \mathbf{S}=\mathbf{s})\,d\mathbf{L}, 
\]
where $\mathbf{s}$ denotes the selection event under which samples are included in the observed dataset. 
For discrete variables, the integral over $\mathbf{L}$ is replaced by summation. Therefore, the conditional independence relations available from data are those among the observed variables under $P_{\mathrm{obs}}(\mathbf{O})$.

The formulation above describes the causal system at the level of the underlying DAG $\DAG$. However, after marginalizing latent variables and conditioning on selection variables, the causal and (conditional) independence relations among the observed variables $\mathbf{O}$ are generally no longer adequately represented by a DAG over $\mathbf{O}$~\citep{spirtes1995fci,richardson2002ancestral}.
We next introduce MAGs and PAGs, which provide the standard graphical language for representing observed-variable causal structures in this setting.

\subsection{Ancestral Graphs}
\label{sec:mag-pag}

We first introduce the relevant graphical notions. We consider simple graphs, meaning that there are no self-loops or multiple edges between two vertices. The two ends of an edge are called \emph{marks}, which can be a tail (`$\--$') or an arrowhead (`$>$'). In partial graphs such as PAGs, a circle mark (`$\circ$') is additionally used to represent an unresolved mark. For convenience, we use an asterisk (`$*$') to denote any allowed edge mark.

\begin{definition}[\textbf{Mixed Graph}]
\label{def:mixed-graph}
A \emph{mixed graph} is a graph that can contain three types of edges: directed edges ($\rightarrow$), bi-directed edges ($\leftrightarrow$), and undirected edges ($\--$), with at most one edge between any two vertices.
\end{definition}

In a mixed graph, $X$ is called a \emph{parent} of $Y$ and $Y$ a \emph{child} of $X$ if $X\rightarrow Y$; $X$ is called a \emph{spouse} of $Y$ if $X\leftrightarrow Y$; and $X$ is called an \emph{undirected neighbor} of $Y$ if $X\--Y$. A \emph{directed path} from $X$ to $Y$ is a path of the form $X\rightarrow \cdots \rightarrow Y$. We say that $X$ is an \emph{ancestor} of $Y$, and $Y$ a \emph{descendant} of $X$, if there is a directed path from $X$ to $Y$. A non-endpoint vertex on a path is called a \emph{collider} if both adjacent edges on the path have arrowheads pointing into it; otherwise, it is called a \emph{non-collider}. 
A path $\pi=\langle V_0,\ldots,V_k\rangle$ is \emph{uncovered (unshielded)} if $V_{i-1}$ and $V_{i+1}$ are non-adjacent for every $1\le i\le k-1$. 
It is called an \emph{uncovered (unshielded) collider path} if, in addition, every non-endpoint vertex $V_i$ on $\pi$ is a collider.

\begin{definition}[\textbf{$\mathbf{m}$-Separation}]
\label{def:m-separation}
In a mixed graph, let $\mathbf{X}$, $\mathbf{Y}$, and $\mathbf{Z}$ be disjoint sets of vertices. A path between vertices $X$ and $Y$ is \emph{$m$-connecting (active)} given $\mathbf{Z}$ if every non-collider on the path is not in $\mathbf{Z}$, and every collider on the path has a descendant in $\mathbf{Z}$. The sets $\mathbf{X}$ and $\mathbf{Y}$ are \emph{$m$-separated} by $\mathbf{Z}$ in $\g$ if there is no $m$-connecting path between any $X\in\mathbf{X}$ and $Y\in\mathbf{Y}$ given $\mathbf{Z}$.
\end{definition}

The notion of $m$-separation generalizes the classical $d$-separation criterion for DAGs to mixed graphs.

\begin{definition}[\textbf{Ancestral Graph}]
\label{def:ancestral-graph}
An \emph{ancestral graph} is a mixed graph satisfying the following conditions:
\begin{enumerate}[leftmargin=22pt,itemsep=0pt,topsep=0pt,parsep=0pt,label=(\roman*)]
    \item there is no directed cycle, \ie, there is no directed edge $X\rightarrow Y$ such that $Y$ is an ancestor of $X$;
    \item there is no almost directed cycle, i.e., there is no bi-directed edge $X\leftrightarrow Y$ such that one endpoint is an ancestor of the other;
    \item for any undirected edge $X\--Y$, neither $X$ nor $Y$ has any parents or spouses.
\end{enumerate}
\end{definition}

Intuitively, the conditions in \cref{def:ancestral-graph} ensure that the arrows and tails in the ancestral graph do not indicate conflicting ancestral relationships.

\begin{definition}[\textbf{Maximal Ancestral Graph}]
\label{def:mag}
A \emph{maximal ancestral graph} (MAG, denoted by $\MAG$) is an ancestral graph such that, for every pair of non-adjacent vertices, there exists a set of vertices that $m$-separates them.
\end{definition}

DAGs are special cases of MAGs.
Given an underlying DAG $\DAG$ over $\mathbf{O}\cup\mathbf{L}\cup\mathbf{S}$, marginalizing over $\mathbf{L}$ and conditioning on $\mathbf{S}=\mathbf{s}$, we can uniquely induce a MAG $\MAG$ over $\mathbf{O}$~\citep{zhang2008completeness}.
For example, the DAG in \cref{fig:dag} induces the MAG in \cref{fig:mag} over the observed variables $\mathbf{O}=\{A,B,C,D,E,F,G,H,I,J,K,T\}$ by marginalizing out the latent variables $\mathbf{L}=\{L_1,L_2\}$ and conditioning on the selection variable $\mathbf{S}=\{S\}$.
Importantly, this MAG preserves the conditional independence statements of $P_{\mathrm{obs}}(\mathbf{O})$ and ancestral relations among $\mathbf{O}$ in $\DAG$~\citep{richardson2002ancestral,zhang2006causal,zhang2008completeness}.~\footnote{The construction from a DAG with latent and selection variables to the corresponding induced MAG, together with its preservation properties, is detailed in \cref{Appendix:DAG-to-MAG}.}

Notably, the MAG $\MAG$ is generally not uniquely identifiable from $P_{\mathrm{obs}}(\mathbf{O})$. Instead, one can at most identify its Markov equivalence class.

\begin{definition}[\textbf{Markov Equivalence}]
\label{def:markov-equivalence}
Two MAGs over the same vertex set are \emph{Markov equivalent} if they entail the same set of $m$-separation relations. Given a MAG $\MAG$, we denote its Markov equivalence class by $[\MAG]$.
\end{definition}

An edge mark in $\MAG$ is called \emph{invariant} if the same mark appears in all MAGs in $[\MAG]$. The invariant information in $[\MAG]$ can be represented by a partial ancestral graph.

\begin{definition}[\textbf{Partial Ancestral Graph}]
\label{def:pag}
A \emph{partial ancestral graph} (PAG, denoted by $\PAG$)  represents a Markov equivalence class $[\MAG]$ of MAGs. It has the same adjacencies as $\MAG$ and any member of $[\MAG]$. A tail `$-$' or arrowhead `$>$' appears in $\PAG$ if the corresponding mark is invariant in $[\MAG]$; otherwise, a circle `$\circ$' appears. \footnote{In the terminology of \citep{zhang2006causal,zhang2008completeness}, such PAGs are referred to as \emph{complete} (or \emph{maximally oriented}) partial ancestral graphs. In this paper, we do not consider PAGs that fail to display all invariant edge marks in the equivalence class of MAGs, and we therefore use the term `PAG' to mean a `maximally oriented PAG'.}
\end{definition}

The above definition describes the global graphical object that can be recovered from $P_{\mathrm{obs}}(\mathbf{O})$ under the standard causal Markov and faithfulness conditions~\citep{spirtes2000causation}.
For example, the PAG in \cref{fig:pag} represents the Markov equivalence class induced by the MAG in \cref{fig:mag}. 
It retains all edge marks shared by every MAG in the class and uses circle marks for those that remain ambiguous. 
Thus, the PAG provides a maximally informative summary of the causal information identifiable from $P_{\mathrm{obs}}(\mathbf{O})$.

\subsection{Problem Formulation}
\label{sec:problem-formulation}

We now state the target-specific problem considered in this paper.
Let $\DAG$ be the underlying causal DAG over $\mathbf{O}\cup\mathbf{L}\cup\mathbf{S}$, and let $\MAG$ be the induced MAG over the observed variables $\mathbf{O}$ obtained by marginalizing latent variables and conditioning on unobserved selection variables~\citep{richardson2002ancestral,zhang2008completeness}. Let $\PAG$ denote the PAG representing the Markov equivalence class $[\MAG]$.

Under the standard causal Markov and faithfulness conditions~\citep{spirtes2000causation}, given a target variable $T\in\mathbf{O}$, our goal is to learn the local causal structure around $T$ from the observational distribution $P_{\mathrm{obs}}(\mathbf{O})$ without recovering the entire global structure over $\mathbf{O}$. In particular, we aim to identify the direct causes and effects of $T$ that are identifiable from $P_{\mathrm{obs}}(\mathbf{O})$. 
The desired local structure should be consistent with the target-specific causal information obtainable from global causal discovery, \ie, PAG $\PAG$, while avoiding the cost of learning the full graph.

\section{Local Causal Structure Learning}
\label{sec:local-learning}

In this section, we study how to learn the local causal structure around a target variable in the presence of latent variables and selection bias. 
To address this problem, we need to answer the following two key questions.

\begin{center}
\begin{minipage}{0.95\linewidth}
\starterm~\textit{\textbf{Q1.} How can we avoid learning the entire global structure over all observed variables?}
\end{minipage}
\end{center}
To answer \textit{\textbf{Q1}}, \cref{sec:markov-blanket} introduces the Markov blanket in MAGs under latent variables and selection bias.
The Markov blanket serves as an instrumental local region around the target variable, allowing us to learn a local structure from the marginal distribution over this region instead of performing causal discovery over all observed variables.

\begin{center}
\begin{minipage}{0.95\linewidth}
\starterm~\textit{\textbf{Q2.} Given such a local region, which structural information learned from it is guaranteed to be consistent with the global PAG?}
\end{minipage}
\end{center}
To answer \textit{\textbf{Q2}}, \cref{sec:foundations-local-learning} establishes a theoretical bridge between the local structure learned from Markov-blanket-based marginal information and the corresponding global PAG. 
These results characterize which local adjacencies and orientation information are sound, and in some cases complete, with respect to the global structure.

Building on these theoretical results, \cref{sec:algorithm} presents \method, a \textbf{Lo}cal \textbf{Ca}usal structure learning method designed to allow for \textbf{L}atent variables and \textbf{S}election bias, together with its theoretical guarantee and complexity analysis.

\subsection{Markov Blanket under Latent Variables and Selection Bias}
\label{sec:markov-blanket}

The concept of the \emph{Markov blanket} (MB) was first introduced by~\cite{pearl1988Probaili} and has since become a foundational principle for feature selection and dimensionality reduction, enabling more efficient and robust model construction~\citep{guyon2003introduction, pellet2008using, gao2016efficient,dong2025intersecting}. Intuitively, the Markov blanket of a target variable $X$ is the minimal set of variables that contains all the relevant information about $X$ not already provided by any other variable~\citep{aliferis2010local}. This fundamental property is formally stated in~\cref{def:Markov-Blanket}.

\begin{definition}[\textbf{Markov Blanket}]
\label{def:Markov-Blanket}
Let $\mathbf{V}$ be a set of variables and let $X\in\mathbf{V}$. 
A set $\mathbf{B}\subseteq \mathbf{V}\setminus\{X\}$ is called the \emph{Markov blanket} of $X$, denoted by $\MB(X)$, if it satisfies the following two conditions:
\begin{enumerate}[label=(\roman*),leftmargin=18pt,itemsep=1pt,topsep=2pt]
    \item \textbf{MB Property:} 
    for every $V\in \mathbf{V}\setminus(\mathbf{B}\cup\{X\})$,
    \[
        X \CI V \mid \mathbf{B}.
    \]
    
    \item \textbf{Minimality:}~\footnote{While some literature distinguishes between a ``Markov blanket'' (which may not require minimality) and a ``Markov boundary'' (the strictly minimal set), we adopt the convention where the Markov blanket inherently implies minimality.} 
    no proper subset $\mathbf{B}'\subsetneq \mathbf{B}$ satisfies the above property, i.e., there does not exist $\mathbf{B}'\subsetneq \mathbf{B}$ such that
    \[
        X \CI V \mid \mathbf{B}'
        \quad
        \text{for all }
        V\in \mathbf{V}\setminus(\mathbf{B}'\cup\{X\}).
    \]
\end{enumerate}
\end{definition}

While~\cref{def:Markov-Blanket} characterizes the MB through conditional independence relations, it is often useful in causal graphical models to have an explicit graphical characterization. Assuming the absence of latent variables and selection bias, the graphical formulation of the MB for a vertex $X$ in a DAG takes the form given in~\cref{def:app-MB-DAG}.

\begin{definition}[\textbf{Markov Blanket for DAGs}~\citep{pearl1988Probaili, pearl2000causality}]
\label{def:app-MB-DAG}
    Under causal faithfulness, in a DAG $\DAG$, the Markov blanket of a vertex $X$ is unique and consists of its parents, children, and the parents of the children of $X$. 
\end{definition}

\begin{remark}
    The idealized assumptions of no latent variables and no selection bias are often violated in real-world settings. Previous studies have established graphical characterizations of the MB of a vertex $X$ in the presence of latent variables, together with theoretical guarantees~\citep{richardson2003markov,pellet2008finding,yu2018mining}. However, these works do not account for selection bias. To fill this gap, we introduce~\cref{def:MB-MAG}.
\end{remark}

\begin{definition}[\textbf{Markov Blanket for MAGs}]
\label{def:MB-MAG}
    Under causal faithfulness, allowing for latent variables and selection bias, consider a MAG $\MAG$ over observed variables $\vars[O]$. The Markov blanket of a vertex $X$, denoted $\MB(X,\MAG)$, is the union of:
    \begin{enumerate}[label=(\roman*)]
        \item the parents, children, and undirected neighbors of $X$;
        \item the parents of the children of $X$;
        \item the district of $X$ (the set of vertices reachable from $X$ using only bidirected edges) and the districts of $X$'s children; and
        \item the parents of every vertex in these districts.
    \end{enumerate} 
\end{definition}

\begin{figure}[htp!]
\centering
\begin{tikzpicture}[x=0.55pt,y=0.55pt,yscale=-1,xscale=1]

\draw (65,88) node(A) [obs] {$A$};
\draw (125,63) node(B) [mb] {$B$};
\draw (234,77) node(C) [obs] {$C$};
\draw (145,114) node (L) [mb] {$L$};
\draw (44,143) node(D) [obs] {$D$};
\draw (146,162) node(E) [obs] {$E$};
\draw (291,110) node(F) [mb] {$F$};
\draw (79,199) node(G) [obs] {$G$};
\draw (211,132) node(H) [mb] {$H$};
\draw (261,43) node(I) [mb] {$I$};
\draw (219,181) node(J) [mb] {$J$};
\draw (292,157) node(K) [obs] {$K$};
\draw (192,62) node(X) [target] {$X$};
\draw (89,144) node(M) [mb] {$M$};
\draw (357,92) node(O) [obs] {$O$};
\draw (357,149) node(P) [obs] {$P$};

\draw[-] (M) -- (A);
\draw[-arcsq] (M) -- (L);
\draw[arcsq-arcsq] (L) -- (B);
\draw[arcsq-arcsq] (L) -- (H);
\draw[-arcsq] (L) -- (E); 
\draw[-arcsq] (G) -- (D);
\draw[-arcsq] (G) -- (E);
\draw[arcsq-arcsq] (E) -- (J);
\draw[-arcsq] (X) -- (H);
\draw[-] (X) -- (I);
\draw[arcsq-arcsq] (H) -- (F);
\draw[-arcsq] (F) -- (C);
\draw[-arcsq] (J) -- (F);
\draw[-arcsq] (K) -- (J);
\draw[-arcsq] (F) -- (O);
\draw[-arcsq] (P) -- (O);

\end{tikzpicture}

\caption{An illustration of the MB in a MAG. The target variable is $X$, and the blue vertices constitute $\MB(X,\MAG)$.}
\vspace{-1.0em}
    \label{fig:Markov-blanket}
\end{figure}
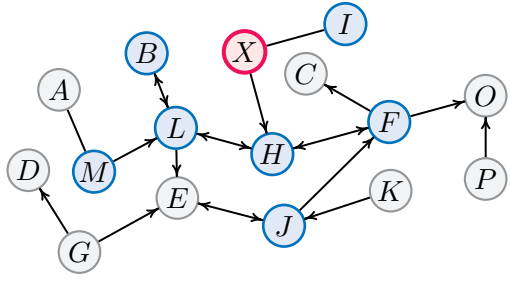

\begin{example}
 Consider the MAG $\MAG$ in \cref{fig:Markov-blanket}, where $X$ is the target variable of interest. The Markov blanket of $X$ in $\MAG$ is $\MB(X, \MAG) = \{B, F, H, L, M, I, J\}$. This set includes the child $H$, the undirected neighbor $I$, the district members associated with $X$ and its child $H$, namely $\{B,F,L\}$, and the parents of vertices in these districts, namely $\{M,J\}$.

\end{example}

Next, we justify \cref{def:MB-MAG} theoretically by showing that the proposed set satisfies the MB property and the minimality condition, as established in \cref{lemma:MB-MAG-property,lemma:MB-MAG-minimality}.

\begin{lemma}[\textbf{MB Property}]
    \label{lemma:MB-MAG-property}
    Let $\MAG$ be a MAG over observed variables $\vars[O]$. For any $X\in\vars[O]$, $X$ is m-separated from all vertices in $\vars[O]\setminus (\MB(X,\MAG)\cup\{X\})$ given $\MB(X,\MAG)$.
\end{lemma}


\begin{lemma}[\textbf{Minimality}]
    \label{lemma:MB-MAG-minimality}
    Let $\MAG$ be a MAG over observed variables $\vars[O]$. For any $X\in\vars[O]$, no proper subset $\mathbf{B}'$ of $\MB(X,\MAG)$ m-separates $X$ from all vertices in $\vars[O]\setminus (\mathbf{B}'\cup\{X\})$.
\end{lemma}


Combining~\cref{lemma:MB-MAG-property} and~\cref{lemma:MB-MAG-minimality}, we can conclude that the defined $\MB(X,\MAG)$ is indeed the MB of $X$ in $\MAG$.

Furthermore, under the causal Markov and faithfulness assumptions, the MB defined in \cref{def:MB-MAG} admits the following conditional-independence characterization.

\begin{proposition}
    \label{proposition:MB-MAG}
    Let $\MAG$ be the MAG induced over the observed variables $\vars[O]$, and let $P_{\mathrm{obs}}(\mathbf{O})$ be Markov and faithful to $\MAG$. For any distinct $X, Y \in \vars[O]$, the following holds:
    \begin{align} 
        Y \in \MB(X,\MAG) \Leftrightarrow X \nCI Y \mid \vars[O] \setminus \{X, Y\} 
    \end{align} 
\end{proposition}

\cref{proposition:MB-MAG} indicates that the MB of a vertex $X$ in a MAG $\MAG$ can be identified by testing the conditional independence relations between $X$ and other vertices in $\vars[O]$ given the rest of the observed variables. This equivalence is crucial for learning the MB from data, as it allows us to use conditional independence tests to determine the MB of a target variable. Based on this theoretical guarantee, algorithms such as the TC algorithm~\citep{pellet2008using} can be applied in practice to identify the MB efficiently in the presence of latent variables and selection bias, with a linear number of conditional independence tests with respect to the number of observable variables.

\begin{remark}
The results in this subsection answer \textit{\textbf{Q1}}. 
For target-specific causal discovery, the Markov blanket provides a compact local region, allowing us to avoid global structure learning over all variables.
Notably, the PAG $\PAG$ represents the Markov equivalence class of $\MAG$ and therefore entails the same conditional independence relations among observed variables. 
By \cref{proposition:MB-MAG}, the same Markov blanket is characterized equivalently by the MAG, the PAG, and the observed distribution:
\[
\MB(X,\MAG)=\MB(X,\PAG)=\MB(X,P_{\mathrm{obs}}(\mathbf{O})).
\]
Therefore, in the remainder of the paper, we simply write $\MB(X)$ for this common Markov blanket, and define
\[
\MBplus(X)=\MB(X)\cup\{X\}
\]
as the Markov-blanket-based instrumental local region. 
\end{remark}

\subsection{Foundations of Local Structure Learning}
\label{sec:foundations-local-learning}

Throughout this subsection, $\g[L]_X$ denotes the PAG learned, in the oracle setting, from the marginal selected distribution $P_{\mathrm{obs}}(\MBplus(X))$ over $\MBplus(X)$ by a sound and complete PAG-learning procedure.
Having identified $\MBplus(X)$ as a Markov-blanket-based local region, we now turn to \textit{\textbf{Q2}}: which structural information learned from this region is guaranteed to be consistent with the global PAG? 
This question is nontrivial because learning a causal structure from the marginal distribution over $\MBplus(X)$ does not, in general, automatically reproduce the structure obtained by global causal discovery over all observed variables. 
We therefore establish theoretical results that connect the locally learned PAG $\g[L]_X$ to the global PAG $\PAG$.

Specifically, we focus on the following question:
\begin{center}
\begin{minipage}{0.98\linewidth}
\starterm~\textit{Which \emph{adjacency relations} and \emph{orientation information} in the locally learned structure $\g[L]_{X}$ are guaranteed to be consistent with the global PAG $\PAG$?}
\end{minipage}
\end{center}

We first establish the consistency of locally learned adjacencies. 
A key observation is that, in a MAG, whether a target variable $X$ can be $m$-separated from another observed variable can be determined using only the MB of $X$.

\begin{theorem}
    \label{theorem:local-learning-edge}
    Let $\MAG$ be the underlying MAG over $\mathbf{O}$, and let $X\in\mathbf{O}$ be a target variable.
    For any $V \in \MB(X)$, there exists a set
    $\mathbf{Z}\subseteq \mathbf{O}\setminus\{X,V\}$ that $m$-separates $X$ and $V$
    if and only if there exists a set
    $\mathbf{Z}'\subseteq \MB(X)\setminus\{V\}$ that $m$-separates $X$ and $V$.
\end{theorem}


\cref{theorem:local-learning-edge} shows that the separability between $X$ and any vertex in its MB can be decided within the local region $\MBplus(X)$. 
Together with the MB property, this yields a local characterization of adjacencies incident to $X$. 
Specifically, for any $V\in \mathbf{O}\setminus\{X\}$, if there exists a separating set $\mathbf{Z}\subseteq \mathbf{O}\setminus\{X,V\}$ for $X$ and $V$, then \cref{theorem:local-learning-edge} guarantees the existence of a local separating set $\mathbf{Z}'\subseteq \MBplus(X)\setminus\{X,V\}$. 
Conversely, if no such local separating set exists within $\MBplus(X)\setminus\{X,V\}$, then no separating set exists within $\mathbf{O}\setminus\{X,V\}$. 
Therefore, the adjacency between $X$ and any vertex can be determined from the MB-based local region.

This leads to the following consistency result for the locally learned structure $\g[L]_{X}$.

\begin{corollary}
    \label{corollary:local-learning-edge}
    Let $\g[L]_{X}$ denote the inferred local structure over $\MBplus(X)$, and let $V\in\MB(X)$. 
    Then $X$ and $V$ are adjacent in $\g[L]_{X}$ if and only if they are adjacent in the global PAG $\PAG$.
\end{corollary}

\cref{corollary:local-learning-edge} implies that all adjacencies incident to $X$ in the locally learned graph $\g[L]_X$ are exactly the same as those incident to $X$ in the global PAG. 
Thus, local learning over $\MBplus(X)$ preserves the target-specific adjacency information that would be obtained by global causal discovery.


Next, we show that certain structures carrying orientation information in $\g[L]_{X}$ are also consistent with the global PAG. 

\begin{theorem}
\label{theorem:local-learning-collider-paths}
Let $\g[L]_{X}$ denote the inferred local structure over $\MBplus(X)$, and let 
$V_i$ $(1 \le i \le |\MB(X)|)$ be vertices in $\MB(X)$. 
Then an uncovered collider path 
\[
X *\!\!\rightarrow V_{1} \leftrightarrow \dots \leftarrow\!\!* V_{i}
\]
exists in $\g[L]_{X}$ if and only if the same uncovered collider path exists
in the global PAG.
\end{theorem}

\cref{theorem:local-learning-collider-paths} indicates that the uncovered collider paths between $X$ and any vertex in $\MB(X)$ in the locally learned $\g[L]_{X}$ are consistent with the global PAG. Conversely, if there is an uncovered collider path between $X$ and any vertex in $\MB(X)$ in the global PAG, then the same uncovered collider path must exist in the locally learned $\g[L]_{X}$.


\begin{theorem}
    \label{theorem:local-learning-collider-triples}
    Let $\g[L]_{X}$ denote the inferred local structure over $\MBplus(X)$, and let 
    $V_i$ $(1 \le i \le |\MB(X)|)$ be vertices in $\MB(X)$. 
    If an unshielded collider triple
    \[
    V_1 *\!\!\rightarrow X \leftarrow\!\!* V_2
    \]
    exists in $\g[L]_{X}$, then the same triple exists in the global PAG.
\end{theorem}

\cref{theorem:local-learning-collider-triples} states that any unshielded collider triple with $X$ as a collider in the locally learned $\g[L]_{X}$ must also exist in the global PAG.


Together, \cref{theorem:local-learning-collider-paths} and \cref{theorem:local-learning-collider-triples} show that the collider-orientation information involving the current target $X$ and identified in the locally learned $\g[L]_{X}$ is sound with respect to the global PAG.

\begin{remark}
It is worth noting that these results do not cover all collider information.
First, collider information in $\g[L]_X$ that does not involve the current target $X$ may be unreliable. 
Second, some collider information involving $X$ in the global PAG may not be recoverable from $\MBplus(X)$. 
The following two examples illustrate these cases using the global PAG in \cref{fig:pag}.
    \begin{itemize}[leftmargin=15pt]
        \item \emph{\textbf{Spurious.}} Collider information in $\g[L]_X$ that does not involve the current target $X$ may be spurious. For example, let $B$ be the target variable, and consider the locally learned graph $\g[L]_B$ shown in \cref{fig:inferred-local-pag-B}. From the marginal distribution over $\MBplus(B)$, one may identify $ T *\!\!\rightarrow K \leftarrow\!\!* I $ because $T \CI I \mid \emptyset$ but $T \nCI I \mid K$. However, this collider is not consistent with the global PAG. In the global structure, the corresponding collider information is $ T *\!\!\rightarrow K \leftarrow\!\!* J, $ rather than $ T *\!\!\rightarrow K \leftarrow\!\!* I. $
        \item \emph{\textbf{Missing.}} Not all collider information involving the current target $X$ can necessarily be recovered from the marginal distribution over $\MBplus(X)$. 
        For example, let $E$ be the target variable.
        In this case, $\MBplus(E)=\{B,E,T\}$, and the locally learned graph $\g[L]_E$ is shown in \cref{fig:inferred-local-pag-E}. Although the global PAG contains the uncovered collider $ T *\!\!\rightarrow E \leftarrow\!\!* B, $ this collider cannot be identified from the marginal distribution over $\MBplus(E)$, as the separating set for $T$ and $B$ (i.e., $K$) is not contained in $\MBplus(E)$.
    \end{itemize}
\label{remark:non-complete-local-V-structures}
\end{remark}

\begin{figure}[hpt!]
    \centering
    \captionsetup[subfigure]{font=small, skip=2pt}

    \begin{subfigure}[t]{0.48\linewidth}
        \centering
        \begin{tikzpicture}
            \draw[fill=red!80, fill opacity=0.35, draw=none]
                (2.6, 0.0) ellipse [x radius=0.5cm, y radius=0.4cm];

            \draw (1.2, 0.0) node(T) [obs] {$T$};
            \draw (2.6, 0.0) node(B) [obs] {$B$};
            \draw (4.0, 0.0) node(G) [obs] {$G$};

            \draw (1.2, -1.1) node(E) [obs] {$E$};
            \draw (4.0, -1.1) node(F) [obs] {$F$};

            \draw (2.6, 1.2) node(K) [obs] {$K$};
            \draw (4.0, 1.2) node(I) [obs] {$I$};

            \draw[circle-arcsq,red] (T) -- (K);
            \draw[circle-arcsq] (T) -- (E); 
            \draw[circle-arcsq] (B) -- (E);
            \draw[arcsq-arcsq] (B) -- (G);
            \draw[circle-arcsq] (K) -- (B);
            \draw[circle-arcsq] (F) -- (G);  
            \draw[circle-arcsq] (I) -- (G);
            \draw[circle-arcsq,red] (I) -- (K);
        \end{tikzpicture}
        \caption{$\g[L]_{B}$}
        \label{fig:inferred-local-pag-B}
    \end{subfigure}
    \hfill
    \begin{subfigure}[t]{0.48\linewidth}
        \centering
        \begin{tikzpicture}
            \draw[fill=red!80, fill opacity=0.35, draw=none]
                (1.2, -1.1) ellipse [x radius=0.5cm, y radius=0.4cm];

            \draw (1.2, 0.0) node(T) [obs] {$T$};
            \draw (2.6, 0.0) node(B) [obs] {$B$};
            \draw (1.2, -1.1) node(E) [obs] {$E$};

            \draw[circle-circle,red] (T) -- (B);
            \draw[circle-circle] (T) -- (E);
            \draw[circle-circle] (E) -- (B);
        \end{tikzpicture}
        \caption{$\g[L]_{E}$}
        \label{fig:inferred-local-pag-E}
    \end{subfigure}

    \caption{Causal graphs accompanying \cref{remark:non-complete-local-V-structures}. (a) The locally learned structure $\g[L]_{B}$; (b) the locally learned structure $\g[L]_{E}$. The red shaded region highlights the current target variable, and red edges denote local results whose consistency with the global learning results is not guaranteed.}
    \label{fig:non-complete-example}
\end{figure}
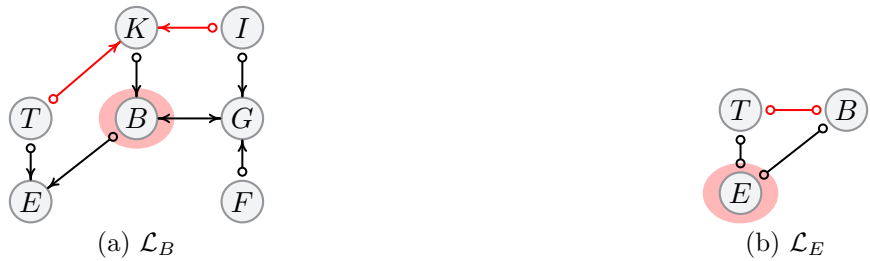

\subsection{The \method Algorithm}
\label{sec:algorithm}

Building on the theoretical results in \cref{sec:markov-blanket,sec:foundations-local-learning}, we present the \method algorithm in \cref{alg:local-structure-learn}. The algorithm starts from the target variable $T$ and sequentially learns local structures over Markov-blanket-based regions.

\begin{algorithm}[hpt!]
   \caption{\method}
   \label{alg:local-structure-learn}
   \begin{algorithmic}[1]
   \REQUIRE Target variable $T$, observed data over $\mathbf{O}$.
   \ENSURE The learned local structure around $T$ in $\hat{\PAG}$.
   
   \STATE Initialize $\mathrm{Waitlist}\gets \{T\}$, $\mathrm{Donelist}\gets \emptyset$, $\hat{\PAG}\gets \emptyset$, and $\Sepsets\gets \emptyset$.
   
   \REPEAT
        \STATE $X$ $\leftarrow$ the first variable of $\mathrm{Waitlist}$.
        
        \STATE $\MBplus(X)\gets \textsc{MBLearn}(X)$.
        \STATE $(\g[L]_X,\Sepsets_X)\gets \textsc{LocalLearn}(\MBplus(X))$.
        \STATE $\Sepsets\gets \Sepsets\cup \Sepsets_X$.
        
        \STATE $\hat{\PAG}\gets \textsc{PreserveConsistentInfo}(\hat{\PAG},\g[L]_X,X)$.
        \STATE $\hat{\PAG}\gets \textsc{OrientPAG}(\hat{\PAG},\Sepsets)$.
        
        \STATE Add $X$ to $\mathrm{Donelist}$.
        \STATE $\mathrm{Waitlist}\gets \textsc{UpdateWaitlist}(T,\hat{\PAG},\mathrm{Donelist})$.
        
   \UNTIL{one of the stopping rules $\mathcal{R}1$--$\mathcal{R}3$ is satisfied}
   
   \STATE \textbf{return} the local structure around $T$ in $\hat{\PAG}$.
   \end{algorithmic}
\end{algorithm}

The algorithm maintains four key objects during this sequential procedure.
First, $\mathrm{Waitlist}$ contains variables whose local structures may provide useful information for determining the local causal structure around the original target $T$. Second, $\mathrm{Donelist}$ records variables that have already been processed. Third, $\hat{\PAG}$ stores the locally learned structural information that is guaranteed to be consistent with the global PAG. Finally, $\Sepsets$ stores the separating sets obtained during local structure learning and is used for subsequent PAG orientation.

The algorithm is initialized with $\mathrm{Waitlist}=\{T\}$, $ \mathrm{Donelist}=\emptyset$, $\hat{\PAG}=\emptyset$, and $\Sepsets=\emptyset$. At each iteration, it selects a variable $X$ from $\mathrm{Waitlist}$, identifies the Markov-blanket-based local region $\MBplus(X)$ using \textsc{MBLearn}~\citep{pellet2008using}, and applies a constraint-based learner \textsc{LocalLearn}~\citep{spirtes1995fci,akbari2021recursive,rohekar2021iterative} to the marginal distribution over $\MBplus(X)$ to obtain a local structure $\g[L]_X$ and separating sets $\Sepsets_X$. The algorithm then updates $\Sepsets$, preserves in $\hat{\PAG}$ only the locally learned structural information justified by \cref{sec:foundations-local-learning} through \textsc{PreserveConsistentInfo}, and applies the complete PAG orientation rules \textsc{OrientPAG} summarized in \cref{app:orientation-rules}. After $X$ is processed, it is added to $\mathrm{Donelist}$, and $\mathrm{Waitlist}$ is updated to retain variables that may still help determine the local structure around the original target $T$. This procedure continues until one of the stopping rules in \cref{theorem:completeness-stop} is satisfied. The full pseudocode is shown in \cref{alg:local-structure-learn}.

We next describe the conceptual stopping conditions. If one updates $\mathrm{Waitlist}$ in a coarse manner, for example by adding all unprocessed vertices connected to $T$ in the current graph $\hat{\PAG}$, then the following three stopping rules are sufficient for termination. 
\begin{theorem}[\textbf{Stopping Rules}] 
    \label{theorem:completeness-stop} 
    Let $T$ be the target variable and let $\mathrm{Waitlist}$ be the current set of variables to be processed. If any of the following rules is satisfied, then the local structure identified for $T$, including its direct causes and direct effects, is equivalent to the corresponding target-specific structure obtained from global causal discovery: 
\begin{description} 
    \item[$\mathcal{R}1.$] All edge marks incident to the target $T$ have been determined. 
    \item[$\mathcal{R}2.$] The set $\mathrm{Waitlist}$ is empty. \item[$\mathcal{R}3.$] Every path from any unprocessed variable to $T$ is not a potentially anterior path toward $T$. 
\end{description} 
A path $\pi=\langle V_1,V_2,\dots,T\rangle$ is said to be not potentially anterior toward $T$ if it contains a vertex $V_i$ such that the edge adjacent to $V_i$ along $\pi$ has an arrowhead into $V_i$, i.e., is of the form $V_i \leftarrow *$.
\end{theorem}

Rules $\mathcal{R}1$ and $\mathcal{R}2$ are both direct stopping criteria, meaning that all causal information of interest has been identified, or all variables have been fully processed.
By contrast, $\mathcal{R}3$ captures a structural redundancy condition. 
Roughly speaking, it states that continuing to learn the local structures of the remaining unprocessed variables cannot provide additional information for determining the local structure around $T$. 
Intuitively, if the remaining unprocessed variables have no potentially directed path to $T$, then their causal information cannot propagate toward $T$ and therefore cannot further refine the target-specific local structure.

In practice, we use a more selective waitlist update. Instead of adding all vertices connected to $T$, \textsc{UpdateWaitlist} only retains vertices that have a potentially anterior path to $T$ in the current graph $\hat{\PAG}$. Consequently, vertices excluded by $\mathcal{R}3$ are automatically removed from consideration. Thus, with this refined update strategy, the algorithm only needs to explicitly check $\mathcal{R}1$ and $\mathcal{R}2$. Detailed implementation is shown in \cref{alg:update-waitlist} in \cref{app:update-waitlist}. The \cref{app:example-R3} provides an example that terminates by $\mathcal{R}3$ under a coarse update but by $\mathcal{R}2$ when using \textsc{UpdateWaitlist}.

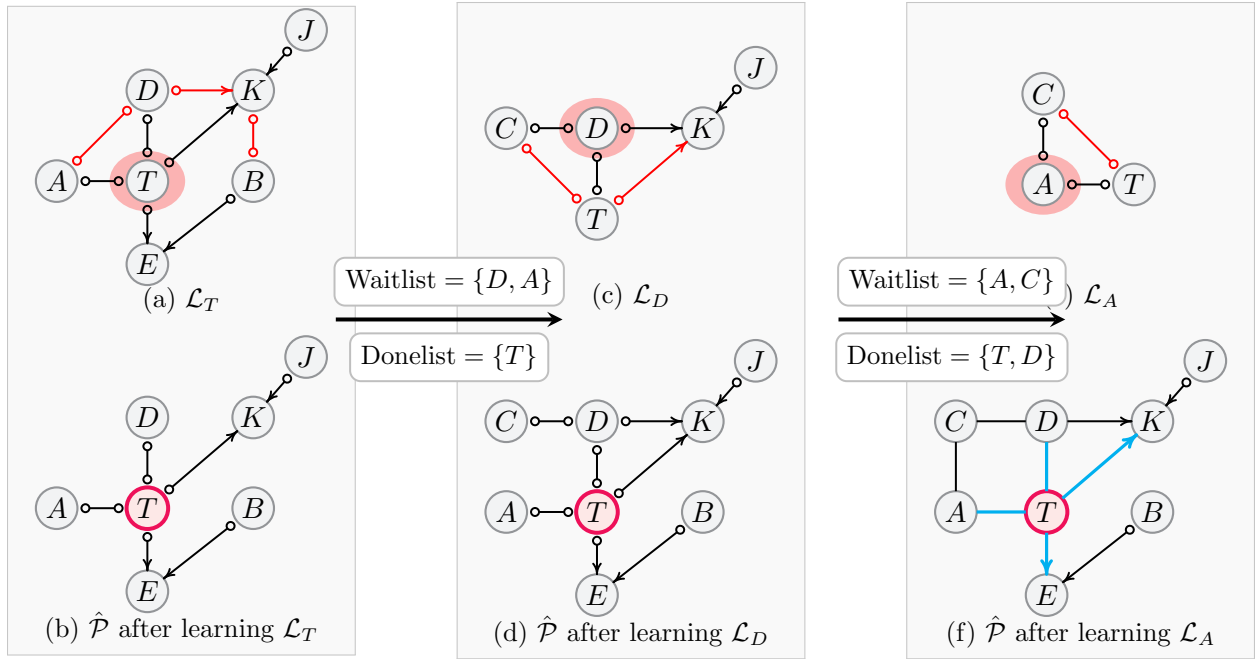
\begin{figure*}[t!]
    \centering
    \captionsetup[subfigure]{font=small, skip=2pt}
    {\renewcommand\fbox[1]{\fcolorbox{gray!45}{gray!05}{#1}}%
    \begin{tikzpicture}
    \node[inner sep=0pt, outer sep=0pt] (figgrid) {%
    \begin{minipage}{\textwidth}
    \centering

    \fbox{%
    \begin{minipage}[c]{0.265\linewidth}
        \centering
        \begin{subfigure}[c]{\linewidth}
            \centering
            \fixedtikzbox{%
            \begin{tikzpicture}[x=1.0cm,y=1.0cm]
                \draw[fill=red!80, fill opacity=0.35, draw=none]
                    (1.2, 0.0) ellipse [x radius=0.5cm, y radius=0.4cm];

                \draw (0.0, 0.0) node(A) [obs] {$A$};
                \draw (1.2, 0.0) node(T) [obs] {$T$};
                \draw (2.6, 0.0) node(B) [obs] {$B$};

                \draw (1.2, 1.2) node(D) [obs] {$D$};
                \draw (1.2, -1.1) node(E) [obs] {$E$};

                \draw (2.6, 1.2) node(K) [obs] {$K$};
                \draw (3.3, 2.0) node(J) [obs] {$J$};

                \draw[circle-circle] (D) -- (T);
                \draw[circle-arcsq,red] (D) -- (K);
                \draw[circle-circle,red] (D) -- (A);

                \draw[circle-arcsq] (J) -- (K);
                \draw[circle-circle] (A) -- (T);

                \draw[circle-arcsq] (T) -- (K);
                \draw[circle-arcsq] (T) -- (E); 
                \draw[circle-arcsq] (B) -- (E);
                \draw[circle-circle,red] (K) -- (B);
            \end{tikzpicture}%
            }
            \caption{$\g[L]_{T}$}
            \label{fig:inferred-local-pag-T}
        \end{subfigure}

        \vspace{3mm}

        \begin{subfigure}[c]{\linewidth}
            \centering
            \fixedtikzbox{%
            \begin{tikzpicture}[x=1.0cm,y=1.0cm]
                \draw (0.0, 0.0) node(A) [obs] {$A$};
                \draw (1.2, 0.0) node(T) [target] {$T$};
                \draw (2.6, 0.0) node(B) [obs] {$B$};

                \draw (1.2, 1.2) node(D) [obs] {$D$};
                \draw (1.2, -1.1) node(E) [obs] {$E$};

                \draw (2.6, 1.2) node(K) [obs] {$K$};
                \draw (3.3, 2.0) node(J) [obs] {$J$};

                \draw[circle-circle] (D) -- (T);
                \draw[circle-arcsq] (J) -- (K);
                \draw[circle-circle] (A) -- (T);

                \draw[circle-arcsq] (T) -- (K);
                \draw[circle-arcsq] (T) -- (E); 
                \draw[circle-arcsq] (B) -- (E);
            \end{tikzpicture}%
            }
            \caption{$\hat{\PAG}$ after learning $\g[L]_{T}$}
            \label{fig:hat-pag-after-T}
        \end{subfigure}
    \end{minipage}}
    \hfill
    \fbox{%
    \begin{minipage}[c]{0.265\linewidth}
        \centering
        \begin{subfigure}[c]{\linewidth}
            \centering
            \fixedtikzbox{%
            \begin{tikzpicture}[x=1.0cm,y=1.0cm]
                \draw[fill=red!80, fill opacity=0.35, draw=none]
                    (1.2, 1.2) ellipse [x radius=0.5cm, y radius=0.4cm];

                \draw (1.2, 0.0) node(T) [obs] {$T$};

                \draw (0.0, 1.2) node(C) [obs] {$C$};
                \draw (1.2, 1.2) node(D) [obs] {$D$};

                \draw (2.6, 1.2) node(K) [obs] {$K$};
                \draw (3.3, 2.0) node(J) [obs] {$J$};

                \draw[circle-circle] (C) -- (D);
                \draw[circle-circle] (D) -- (T);
                \draw[circle-circle,red] (C) -- (T);

                \draw[circle-arcsq] (D) -- (K);
                \draw[circle-arcsq,red] (T) -- (K);
                \draw[circle-arcsq] (J) -- (K);
            \end{tikzpicture}%
            }
            \caption{$\g[L]_{D}$}
            \label{fig:inferred-local-pag-D}
        \end{subfigure}

        \vspace{4mm}

        \begin{subfigure}[c]{\linewidth}
            \centering
            \fixedtikzbox{%
            \begin{tikzpicture}[x=1.0cm,y=1.0cm]
                \draw (0.0, 0.0) node(A) [obs] {$A$};
                \draw (1.2, 0.0) node(T) [target] {$T$};
                \draw (2.6, 0.0) node(B) [obs] {$B$};

                \draw (0.0, 1.2) node(C) [obs] {$C$};
                \draw (1.2, 1.2) node(D) [obs] {$D$};
                \draw (1.2, -1.1) node(E) [obs] {$E$};

                \draw (2.6, 1.2) node(K) [obs] {$K$};
                \draw (3.3, 2.0) node(J) [obs] {$J$};

                \draw[circle-circle] (D) -- (T);
                \draw[circle-arcsq] (D) -- (K);

                \draw[circle-circle] (C) -- (D);
                \draw[circle-arcsq] (J) -- (K);
                \draw[circle-circle] (A) -- (T);

                \draw[circle-arcsq] (T) -- (K);
                \draw[circle-arcsq] (T) -- (E); 
                \draw[circle-arcsq] (B) -- (E);
            \end{tikzpicture}%
            }
            \caption{$\hat{\PAG}$ after learning $\g[L]_{D}$}
            \label{fig:hat-pag-after-D}
        \end{subfigure}
    \end{minipage}}
    \hfill
    \fbox{%
    \begin{minipage}[c]{0.265\linewidth}
        \centering
        \begin{subfigure}[c]{\linewidth}
            \centering
            \fixedtikzbox{%
            \begin{tikzpicture}[x=1.0cm,y=1.0cm]
                \draw[fill=red!80, fill opacity=0.35, draw=none]
                    (0.0, 0.0) ellipse [x radius=0.5cm, y radius=0.4cm];

                \draw (0.0, 0.0) node(A) [obs] {$A$};
                \draw (1.2, 0.0) node(T) [obs] {$T$};
                \draw (0.0, 1.2) node(C) [obs] {$C$};

                \draw[circle-circle,red] (C) -- (T);
                \draw[circle-circle] (A) -- (T);
                \draw[circle-circle] (A) -- (C);
            \end{tikzpicture}%
            }
            \caption{$\g[L]_{A}$}
            \label{fig:inferred-local-pag-A}
        \end{subfigure}

        \vspace{4mm}

        \begin{subfigure}[c]{\linewidth}
            \centering
            \fixedtikzbox{%
            \begin{tikzpicture}[x=1.0cm,y=1.0cm]
                \draw (0.0, 0.0) node(A) [obs] {$A$};
                \draw (1.2, 0.0) node(T) [target] {$T$};
                \draw (2.6, 0.0) node(B) [obs] {$B$};

                \draw (0.0, 1.2) node(C) [obs] {$C$};
                \draw (1.2, 1.2) node(D) [obs] {$D$};
                \draw (1.2, -1.1) node(E) [obs] {$E$};

                \draw (2.6, 1.2) node(K) [obs] {$K$};
                \draw (3.3, 2.0) node(J) [obs] {$J$};

                \draw[-] (C) -- (A);
                \draw[-] (C) -- (D);
                \draw[-,ProcessBlue,line width=1.2pt] (D) -- (T);
                \draw[-,ProcessBlue,line width=1.2pt] (A) -- (T);

                \draw[-arcsq] (D) -- (K);
                \draw[-arcsq,ProcessBlue,line width=1.2pt] (T) -- (K);
                \draw[-arcsq,ProcessBlue,line width=1.2pt] (T) -- (E); 

                \draw[circle-arcsq] (B) -- (E);
                \draw[circle-arcsq] (J) -- (K);
            \end{tikzpicture}%
            }
            \caption{$\hat{\PAG}$ after learning $\g[L]_{A}$}
            \label{fig:hat-pag-after-A}
        \end{subfigure}
    \end{minipage}}
    \end{minipage}};

    \begin{scope}[overlay]
        \draw[-stealth, line width=1.5pt]
            ([xshift=4.35cm,yshift=-4.20cm]figgrid.north west) --
            ([xshift=7.35cm,yshift=-4.20cm]figgrid.north west);
        \node[font=\small, fill=white, inner sep=4pt,align=center,draw=black!25,rounded corners=4pt,]
            at ([xshift=5.85cm,yshift=-3.7cm]figgrid.north west) {$\mathrm{Waitlist}=\{D,A\}$};
        \node[font=\small, fill=white, inner sep=4pt,align=center,draw=black!25,rounded corners=4pt,]
            at ([xshift=5.85cm,yshift=-4.7cm]figgrid.north west) {$\mathrm{Donelist}=\{T\}$};

        \draw[-stealth, line width=1.5pt]
            ([xshift=11cm,yshift=-4.20cm]figgrid.north west) --
            ([xshift=14cm,yshift=-4.20cm]figgrid.north west);
        \node[font=\small, fill=white, inner sep=4pt,align=center,draw=black!25,rounded corners=4pt,]
            at ([xshift=12.5cm,yshift=-3.7cm]figgrid.north west) {$\mathrm{Waitlist}=\{A,C\}$};
        \node[font=\small, fill=white, inner sep=4pt,align=center,draw=black!25,rounded corners=4pt,]
            at ([xshift=12.5cm,yshift=-4.7cm]figgrid.north west) {$\mathrm{Donelist}=\{T,D\}$};
    \end{scope}
    \end{tikzpicture}
    }

    \caption{Illustration of the sequential procedure for identifying the local structure of the target variable $T$ in \cref{example:completed-algorithm}. The underlying SCM with latent and selection variables is shown in \cref{fig:dag}, and the corresponding ground-truth PAG is shown in \cref{fig:pag}. The first row shows the sequentially learned local structures, and the second row shows the corresponding updates to the global PAG. In each local structure, the red shaded region indicates the current target variable, and the red edges denote local results whose consistency with the global PAG is not guaranteed.}

    \label{fig:inferred-local-pag-process}
\end{figure*}

\begin{example}
\label{example:completed-algorithm}
We illustrate \method using the example in \cref{fig:inferred-local-pag-process}. 
Assume that oracle conditional independence tests are available, and let $T$ be the target variable of interest; consider the ground-truth global PAG shown in \cref{fig:pag}. 
The algorithm starts with $\mathrm{Waitlist}=\{T\}$, $\mathrm{Donelist}=\emptyset$, and $\hat{\PAG}=\emptyset$. 
The learning process proceeds as follows:
\begin{itemize}[leftmargin=15pt]
    \item \textbf{Processing $T$.} 
    The algorithm first selects $T$ from $\mathrm{Waitlist}$ and identifies its MB-based local region as
    $
    \MBplus(T)=\{A,B,D,E,J,K,T\}.
    $
    Using the marginal distribution over $\MBplus(T)$, it learns the local structure $\g[L]_{T}$ shown in \cref{fig:inferred-local-pag-T}. 
    From $\g[L]_{T}$, \method preserves the local information whose consistency with the global PAG is guaranteed. 
    The red edges in \cref{fig:inferred-local-pag-T} denote local results whose consistency with the global PAG is not guaranteed and are therefore not preserved. 
    After this update, the current graph $\hat{\PAG}$ is shown in \cref{fig:hat-pag-after-T}. 
    The algorithm then updates
    $
    \mathrm{Donelist}=\{T\}
    $
    and
    $
    \mathrm{Waitlist}=\{D,A\}.
    $

    \item \textbf{Processing $D$.} 
    The algorithm next selects $D$ from $\mathrm{Waitlist}$ and identifies
    $
    \MBplus(D)=\{C,D,J,K,T\}.
    $
    It then learns the local structure $\g[L]_{D}$ over $\MBplus(D)$, as shown in \cref{fig:inferred-local-pag-D}. 
    Again, only the locally learned structures that are guaranteed to be consistent with the global PAG are incorporated into $\hat{\PAG}$, while the red edge in $\g[L]_{D}$ is discarded. 
    This yields the updated graph shown in \cref{fig:hat-pag-after-D}. 
    The algorithm then updates
    $
    \mathrm{Donelist}=\{T,D\}
    $
    and
    $
    \mathrm{Waitlist}=\{A,C\}.
    $

    \item \textbf{Processing $A$.} 
    The algorithm then selects $A$ from $\mathrm{Waitlist}$ and identifies
    $
    \MBplus(A)=\{A,C,T\}.
    $
    It learns the local structure $\g[L]_{A}$ shown in \cref{fig:inferred-local-pag-A}. 
    After incorporating the globally consistent local information from $\g[L]_{A}$ and applying the orientation rules, \method obtains the updated graph shown in \cref{fig:hat-pag-after-A}. 
    At this point, the local structure around the original target $T$ has been determined: the edge marks incident to $T$ are determined as
    $
    A\!-\!T,\quad D\!-\!T,\quad T\rightarrow K,\quad T\rightarrow E.
    $
    Therefore, stopping rule $\mathcal{R}1$ is satisfied, and the algorithm returns the local structure around $T$ in $\hat{\PAG}$.
\end{itemize}
\end{example}

We now state the overall correctness guarantee of \method. 
The theorem shows that the local output of \method agrees with the target-specific information that would be obtained by global causal discovery.

\begin{theorem}[\textbf{Soundness and Completeness of \method}]
\label{theorem:locals-correctness}
Assume the causal Markov and faithfulness conditions, oracle access to the conditional independence relations in $P_{\mathrm{obs}}(\mathbf{O})$.
Then \method identifies the direct causes and effects of the target variable $T$ that are identifiable from the global PAG $\PAG$. 
Equivalently, the edges incident to $T$ and their identifiable endpoint marks in the output $\hat{\PAG}$ are consistent with those in the global PAG $\PAG$.
\end{theorem}


\subsection{Complexity Analysis}
\label{sec:complexity-analysis}

Let $|\vars[O]|$ denote the number of observed variables, and let $r$ denote the number of local structures learned sequentially by \method. The computational cost of \method mainly consists of two components: identifying local search spaces via Markov blanket (MB) learning, and performing local causal structure learning within these search spaces.

For the first component, if an MB discovery algorithm such as TC~\citep{pellet2008using} is used, the cost of learning the MBs of $r$ variables is bounded by $\mathcal{O}(r|\vars[O]|)$. For the second component, suppose a constraint-based causal discovery algorithm is used as the local learner, such as FCI algorithm~\citep{spirtes1995fci}. Then the total cost over the $r$ local subproblems is bounded by
$
\mathcal{O}\!\left(rL(|\MB|)\right),
$
where $|\MB|$ denotes the maximum MB size encountered during the procedure, and $L(\cdot)$ denotes the time complexity of the local learner as a function of the size of the local search space.

Therefore, the overall complexity of \method is bounded by
$
\mathcal{O}\!\left(r\left(|\vars[O]|+L(|\MB|)\right)\right).
$
Since the MB size is typically much smaller than the total number of observed variables, i.e., $|\MB|\ll |\vars[O]|$, \method can substantially reduce the computational burden compared with applying a global causal discovery algorithm to all observed variables.





\section{Experiments}
\label{sec:Experiments}

\subsection{Baseline Methods}
\label{sec:baseline-methods}

We compare \method with a broad collection of causal structure learning methods, including both global and local algorithms. 
For global causal discovery, we consider two types of baselines. 
The first type assumes the absence of latent variables and selection bias, including PC~\citep{spirtes1991PC} and EEMBI-PC~\citep{dong2025intersecting}. 
The second type allows for latent variables and selection bias, including FCI~\citep{spirtes1995fci}, RFCI~\citep{colombo2012rfci}, FCI$^+$~\citep{claassen2013learning}, L-MARVEL~\citep{akbari2021recursive}, and ICD~\citep{rohekar2021iterative}. 
These global methods serve as important references for assessing whether \method can achieve comparable structural accuracy while avoiding the cost of learning the entire causal structure.

For local causal structure learning, we include representative methods such as PCD-by-PCD~\citep{yin2008partial}, MB-by-MB~\citep{wang2014discovering}, CMB~\citep{gao2015local}, PSL~\citep{ling2022psl}, and GraN-LCS~\citep{liang2023gradient}, which were developed under the assumptions of no latent variables and no selection bias. 
We also include LatentLCD~\citep{ling2025local}, which allows for latent variables but does not account for selection bias. 
These local methods are designed to recover target-specific structures, but none of them can handle both latent variables and selection bias, which is the setting addressed by \method.

Additional details of all compared methods, including their learning scope, supported assumptions, and implementation settings, are provided in \cref{tab:baseline-methods} in \cref{app:supplementary-baseline-methods}.
Our source code is available at \href{https://github.com/zhengli0060/LoCaLS}{github.com/zhengli0060/LoCaLS}.

\begin{figure*}[t!]
    \centering

    \begin{subfigure}{1.0\linewidth}
        \centering
        \includegraphics[width=\linewidth]{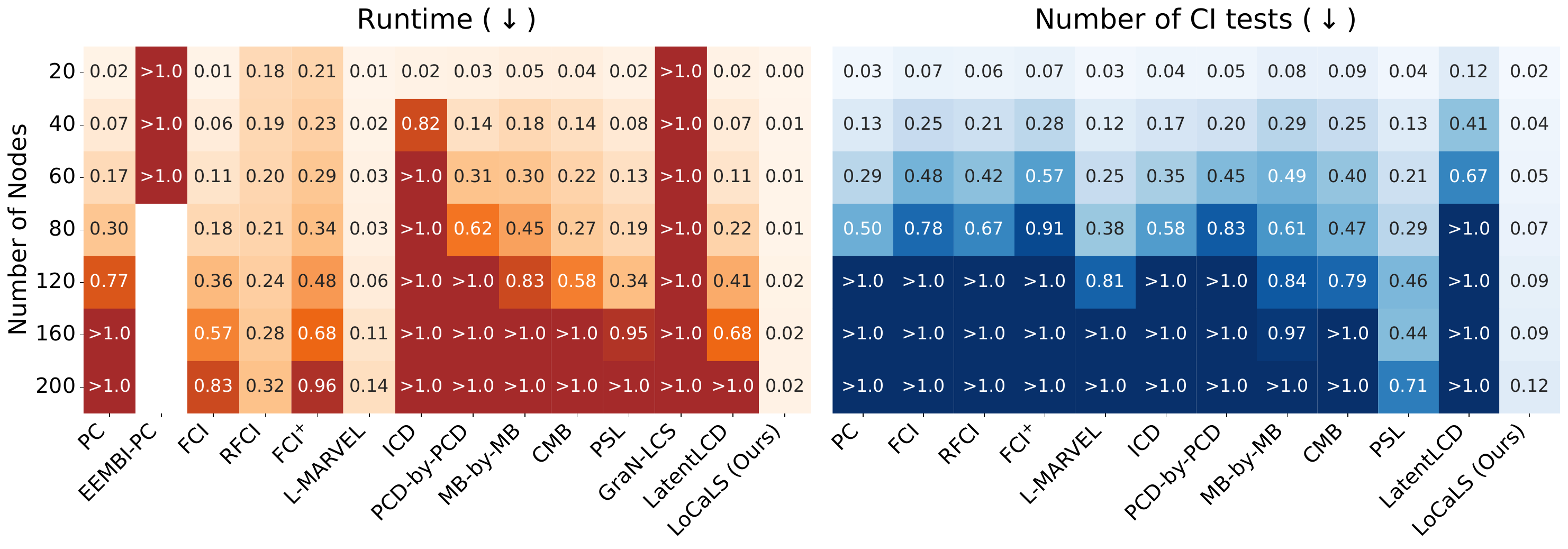}
        \label{fig:vary-nodes-runtime-ci-test}
    \end{subfigure}

    \begin{subfigure}{1.0\linewidth}
        \centering
        \includegraphics[width=\linewidth]{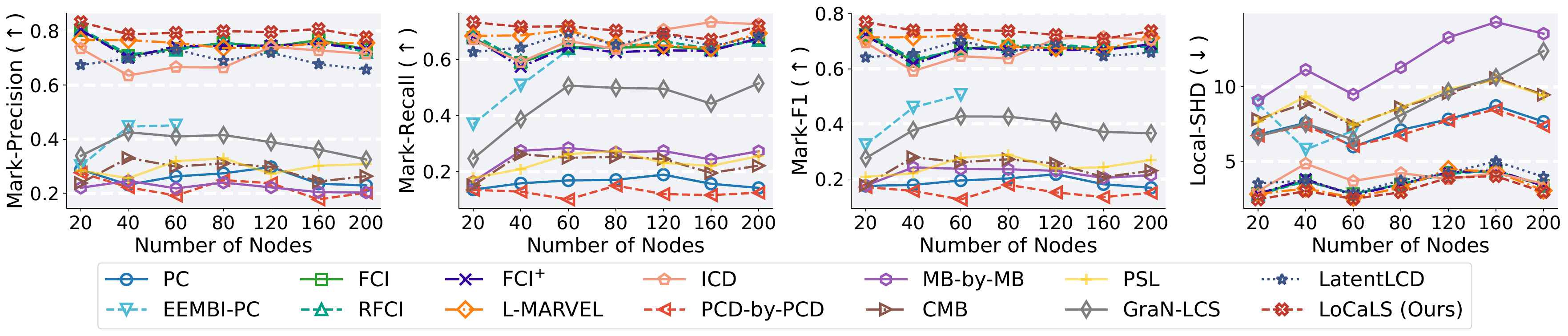}
        \vspace{-0.3cm}
        \label{fig:vary-nodes-prf}
    \end{subfigure}

    \caption{
    Results on random structures with varying dimensionality.
    The top panel reports normalized runtime and CI-test counts; runtime values are normalized by $5$ seconds and CI-test counts by $8\times10^3$, with values above these thresholds shown as $>1.0$. 
    The bottom panel reports Mark-Precision, Mark-Recall, Mark-F1, and Local-SHD.
    }
    \label{fig:ER-varying-vertices-results}
    \vspace{-0.5cm}
\end{figure*}

\subsection{Evaluation Metrics}
\label{sec:evaluation-metrics}

We evaluate all methods from two perspectives: target-specific structural accuracy and efficiency. 
Since the goal is to recover the local causal structure around a target variable $T$, all structural accuracy metrics are computed only on the edge marks incident to $T$.

For structural accuracy, we report Mark-Precision, Mark-Recall, Mark-F1, and Local-SHD. 
Mark-Precision measures how many predicted local edge marks are correct, while Mark-Recall measures how many true local edge marks are successfully recovered. 
Mark-F1 summarizes the trade-off between these two quantities. 
Local-SHD further measures the total structural discrepancy around the target variable, penalizing missing adjacencies, false adjacencies, and incorrect edge marks. 
Thus, higher Mark-Precision, Mark-Recall, and Mark-F1 indicate better recovery, whereas lower Local-SHD indicates smaller local structural error.
The formal definitions of all metrics are provided in \cref{app:supplementary-evaluation-metrics}.
For efficiency, we report the number of conditional independence (CI) tests and runtime. 
Fewer CI tests and lower runtime indicate higher computational efficiency.

\begin{figure*}[t!]
    \centering
    \begin{subfigure}{1.0\linewidth}
        \centering
        \includegraphics[width=\linewidth]{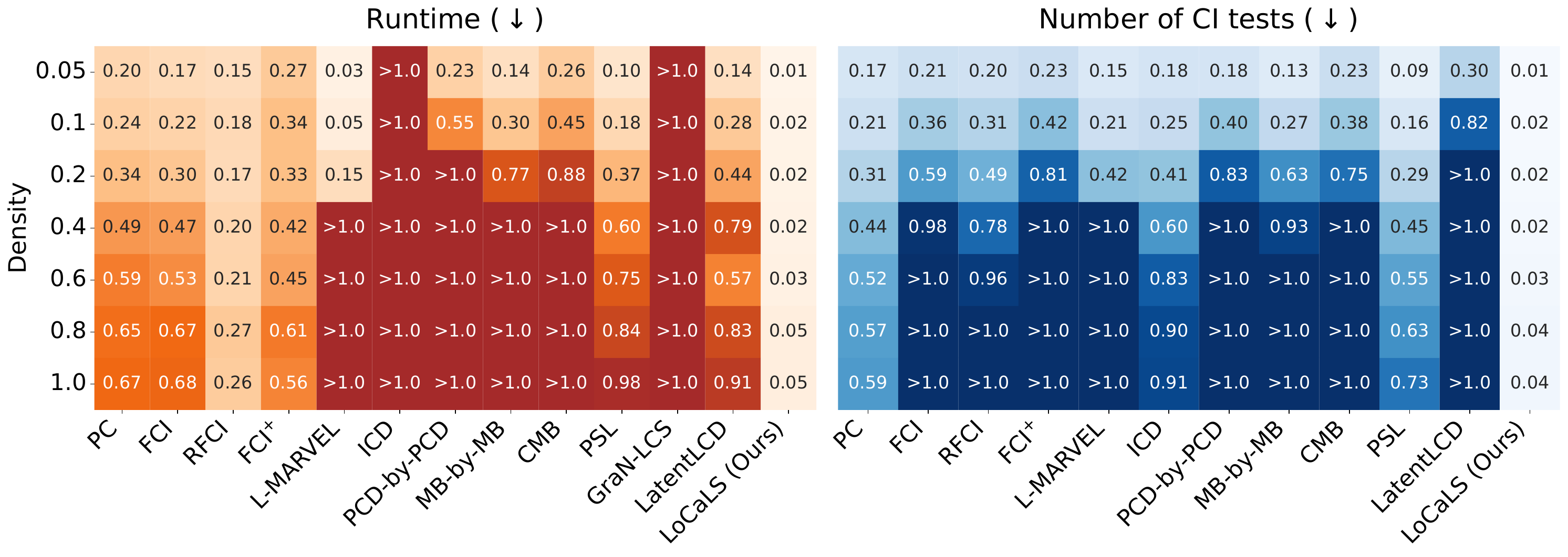}
        \label{fig:vary-density-runtime-ci-test}
    \end{subfigure}
    \begin{subfigure}{1.0\linewidth}
        \centering
        \includegraphics[width=\linewidth]{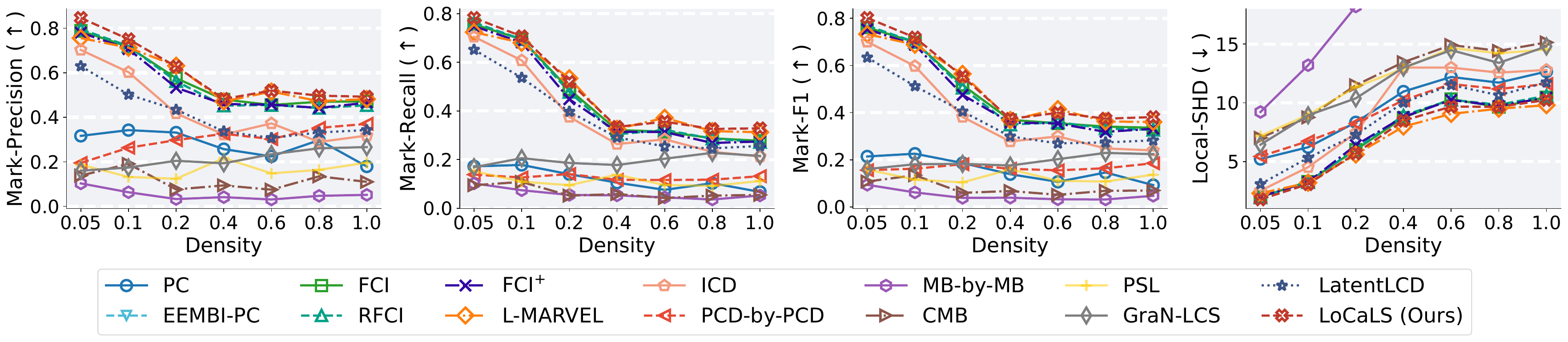}
        \label{fig:vary-density-prf}
        \vspace{-0.3cm}
    \end{subfigure}
    \caption{Results on random structures with varying graph density. The top panel reports normalized runtime and CI-test counts; runtime values are normalized by $10$ seconds and CI-test counts by $3\times10^4$, with values above these thresholds shown as $>1.0$. The bottom panel reports Mark-Precision, Mark-Recall, Mark-F1, and Local-SHD.}
    \label{fig:ER-varying-density-results}
    \vspace{-0.5cm}
\end{figure*}

\subsection{Random Structures with Varying Dimensions}
\label{sec:performance-varying-dimensions}

We first evaluate the performance of different methods as the problem dimensionality increases. Specifically, we generate random underlying causal structures from the Erdős--Rényi model $\mathrm{ER}(n,d)$~\citep{erd6s1960evolution}, where $n$ denotes the number of variables in $\vars[V]$ and $d$ specifies the expected average degree. In this experiment, we fix $d=2$ and vary $n$ over $\{20,40,60,80,120,160,200\}$. Following the convention in~\citep{colombo2012rfci,akbari2021recursive,rohekar2021iterative,wang2014discovering,ling2025local}, each generated underlying causal structure is parameterized as a linear Gaussian structural causal model. The causal coefficients are sampled uniformly from $[-1,-0.5]\cup[0.5,1]$, and all noise terms are independently drawn from a standard Gaussian distribution.

To simulate the presence of latent variables and selection bias, we randomly select $5\%$ of the variables as latent variables and another $5\%$ as unobserved selection variables, \ie, $|\vars[L]| = |\vars[S]| = 5\% \cdot n$.
Latent variables are selected from vertices in the underlying DAG that have at least two children, so that marginalizing them can induce latent confounding among observed variables.
Selection variables are selected from vertices that have at least two parents, so that conditioning on them can induce selection effects.
Following the selection bias mechanism in~\citep{versteeg2022local}, we first generate an initial pool of samples and then retain samples whose aggregate selection score $\sum \mathbf{s}$ falls within the rank interval $[0.5,0.9]$ of the initial samples.
This sampling procedure is repeated until the desired sample size is reached.
The sample size is fixed at $1000$ for all settings, and all reported results are averaged over $50$ independently generated datasets. Target variables are selected from vertices with relatively large neighborhoods.\footnote{Summary statistics of the underlying structures are provided in \cref{app:supplementary-experiments-varying-dimensions}.}

The results are summarized in \cref{fig:ER-varying-vertices-results}, with detailed runtime and CI-test results provided in \cref{tab:vary-dimension-runtime,tab:vary-dimension-ci-test} of \cref{app:supplementary-experiments-varying-dimensions}.
EEMBI-PC is excluded from some plots due to excessive runtime.  
Overall, \method achieves strong target-specific structural accuracy while remaining highly efficient across all dimensions. 
Compared with global methods, \method requires substantially fewer CI tests and shorter runtime, while achieving comparable or better structural accuracy. 
Compared with existing local methods, \method attains substantially better structural accuracy while preserving high efficiency, since these methods do not account for latent variables and selection bias.

As the number of dimensions increases, the task becomes more challenging for all methods. 
Nevertheless, the efficiency advantage of \method over global methods becomes more pronounced in higher-dimensional settings, while its structural accuracy advantage over most baselines remains substantial. 
For example, when $n=200$, \method performs fewer than $10^3$ CI tests, whereas most competing methods require several thousand to tens of thousands of CI tests. 
The runtime results show a similar trend: \method remains highly efficient even in the largest setting, whereas many baselines become substantially more expensive.

\begin{figure*}[t!]
    \centering

    \begin{subfigure}{1.0\linewidth}
        \centering
        \includegraphics[width=\linewidth]{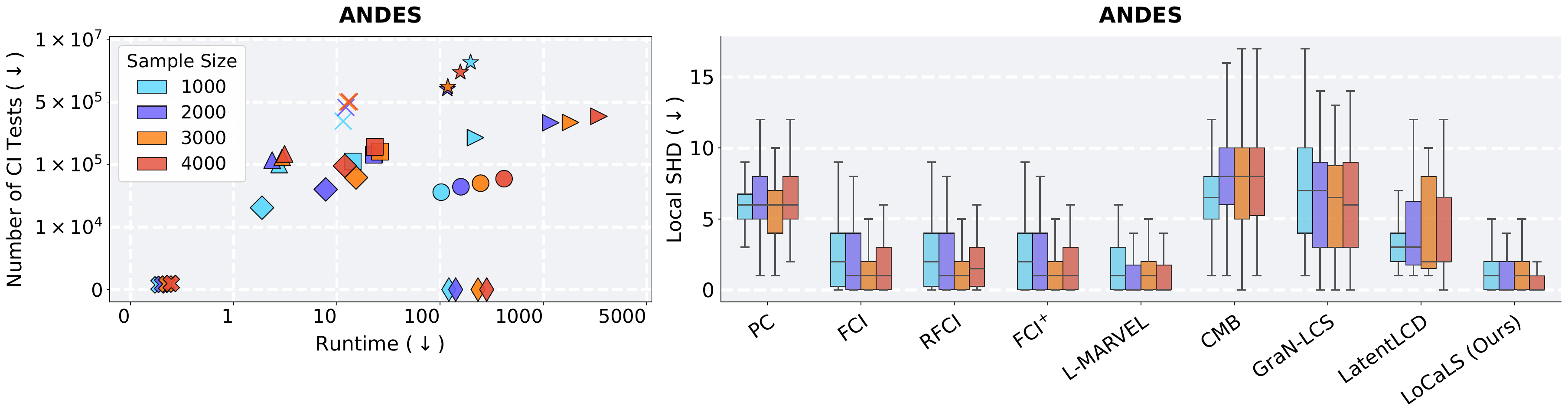}
        \label{fig:andes-main}
    \end{subfigure}
    \begin{subfigure}{1.0\linewidth}
        \centering
        \includegraphics[width=\linewidth]{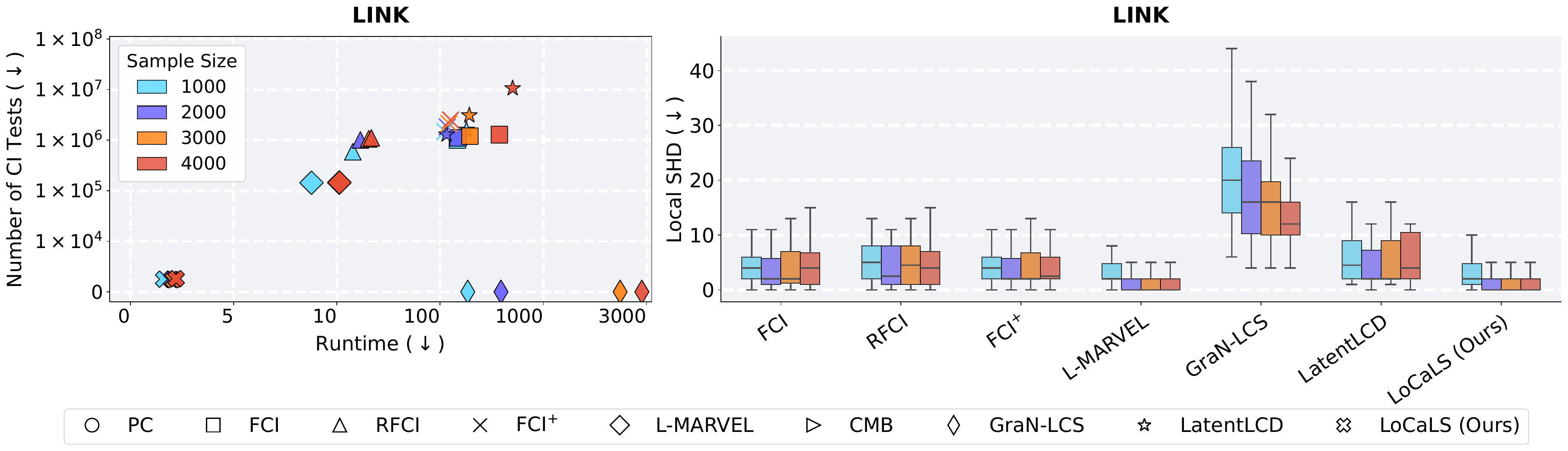}
        \vspace{-0.6cm}
        \label{fig:link-main}
    \end{subfigure}
    \caption{
    Results on the ANDES and LINK benchmark structures under varying sample sizes. The top and bottom panels correspond to ANDES and LINK, respectively. In each panel, the left plot reports runtime and CI-test counts, and the right plot reports Local-SHD. The symbol $\downarrow$ indicates that lower values are better. 
    }
    \label{fig:benchmark-results}
    \vspace{-0.5cm}
\end{figure*}

\subsection{Random Structures with Varying Density}
\label{sec:performance-varying-density}

We next evaluate the performance of \method and the baseline methods as the density of the underlying graph increases. 
Specifically, we generate random underlying causal structures from the Erdős--Rényi model $\mathrm{ER}(n,p)$~\citep{erd6s1960evolution}, where $p$ denotes the probability that an edge is present between a pair of vertices. Following the dense regime considered in~\citep{mokhtarian2025recursive,akbari2021recursive}, we use $\log n/n$ as the reference upper density level. To obtain a gradual transition from sparse to dense graphs, we introduce a density scaling factor $s$ and set \[ p = s\cdot \frac{\log n}{n}. \] In our experiments, we fix $n=100$ and vary $s$ over $\{0.05,0.1,0.2,0.4,0.6,0.8,1.0\}$. Here, $s$ is referred to as the density level in the plots. All other settings are the same as those in \cref{sec:performance-varying-dimensions}.

The results are summarized in \cref{fig:ER-varying-density-results}, with detailed numerical runtime and CI-test results provided in \cref{tab:vary-density-runtime,tab:vary-density-ci-test} of \cref{app:supplementary-experiments-varying-density}. As the graph density increases, causal discovery becomes increasingly challenging: all methods show some degree of degradation in structural accuracy, either through lower Mark-Precision, Mark-Recall, or Mark-F1, or through larger Local-SHD. Despite this difficulty, \method generally achieves better target-specific accuracy across different density levels. In particular, it outperforms existing local methods that do not account for latent variables and selection bias, and remains competitive with global methods designed for this more general setting. 
In terms of efficiency, \method shows a clear advantage. Across all density levels, it requires substantially fewer CI tests and shorter runtime than the baselines. This advantage becomes more pronounced in denser graphs, where global methods and many local baselines suffer from larger conditioning-set search space. EEMBI-PC is omitted from this experiment due to excessive runtime. 

\subsection{Real-World Benchmark Structures}
\label{sec:performance-real-world-structures}

We further evaluate the compared methods on four real-world benchmark structures, MILDEW, BARLEY, ANDES, and LINK, drawn from the Bayesian Network Repository\footnote{\url{https://www.bnlearn.com/bnrepository/}; detailed descriptions of these structures are available therein. Summary statistics of the underlying structures in this experiment are reported in \cref{tab:real-structure-degree} of \cref{app:supplementary-experiments-real-world-structures}.}. Their dimensionalities range from $35$ to $724$ variables. This experiment complements the random-graph evaluations in \cref{sec:performance-varying-dimensions,sec:performance-varying-density} by testing whether the observed advantages persist on non-random real-world structures. We vary the sample size over $\{1000,2000,3000,4000\}$ and keep the remaining experimental settings the same as those in \cref{sec:performance-varying-dimensions}.

\cref{fig:benchmark-results} reports the main results on the two larger structures, ANDES and LINK. 
Complete results on MILDEW and BARLEY, as well as the Mark-Precision, Mark-Recall, and Mark-F1 curves for all four structures, are provided in \cref{app:supplementary-experiments-real-world-structures}.
Note that some methods are excluded from certain plots due to excessive runtime, and GraN-LCS is reported with zero CI tests because it does not perform CI tests.
The results show that \method scales well on real-world benchmark structures. 
On ANDES, \method requires only around $10^3$ CI tests and finishes within one second across all sample sizes. 
On LINK, which contains $724$ variables, \method still requires only a few thousand CI tests and finishes within a few seconds. 
In terms of accuracy, \method maintains low and stable Local-SHD on both ANDES and LINK as the sample size increases, indicating that its efficiency gain does not come at the cost of target-specific structural accuracy.
The supplementary Mark-Precision, Mark-Recall, and Mark-F1 curves further show that \method is consistently among the top-performing methods. 
Similar trends are observed on MILDEW and BARLEY.

\subsection{Sensitivity to the Number of Latent Variables} \label{sec:sensitivity-latent} 

We evaluate how sensitive different methods are to the number of latent variables. This experiment is designed to isolate the effect of latent confounding while keeping the observed problem size fixed. Specifically, we first generate an observed backbone from $\mathrm{ER}(n,d)$ models with $n=100$ and $d=2$, and then augment the underlying DAG with additional latent and selection variables. The latent ratio $r_L$ is varied over $\{0.05,0.10,0.15,0.20\}$, where $r_L=0.05$ corresponds to adding $5=r_L\cdot n$ latent variables. To focus on the effect of latent variables, the selection ratio is fixed at $r_S=0.05$. All other experimental settings follow \cref{sec:performance-varying-dimensions}. 

All results are provided in \cref{app:supplementary-sensitivity-latent}. 
As the latent ratio increases, 
all methods show some degree of degradation in structural accuracy, either through lower Mark-Precision, Mark-Recall, or Mark-F1, or through larger Local-SHD.
Notably, as the latent ratio increases, \method retains its advantage in requiring substantially fewer CI tests and shorter runtime than the baselines, while maintaining higher Mark-Precision, Mark-Recall, Mark-F1, and lower Local-SHD.

\subsection{Sensitivity to the Number of Selection Variables}
\label{sec:sensitivity-selection}

We further evaluate the sensitivity of different methods to the number of selection variables. 
This experiment complements \cref{sec:sensitivity-latent} by isolating the effect of selection bias while keeping the observed problem size fixed. 
Specifically, we first generate an observed backbone from $\mathrm{ER}(n,d)$ models with $n=100$ and $d=2$, and then augment the underlying DAG with additional latent and selection variables. 
The selection ratio $r_S$ is varied over $\{0.05,0.10,0.15,0.20\}$, where $r_S=0.05$ corresponds to adding $5=r_S\cdot n$ selection variables. 
To focus on the effect of selection variables, the latent ratio is fixed at $r_L=0.05$. 
All other experimental settings follow \cref{sec:performance-varying-dimensions}.

All results are provided in \cref{app:supplementary-sensitivity-selection}. 
Consistent with the latent-variable sensitivity experiment, increasing the selection ratio leads to some degradation in structural accuracy for all methods.
As expected, \method still keeps the advantage of requiring substantially fewer CI tests and shorter runtime than the baselines, while maintaining higher Mark-Precision, Mark-Recall, Mark-F1, and lower Local-SHD, even as the selection ratio increases.

\section{Applications to Gene Expression Datasets}
\label{sec:real-datasets}

In this section, we apply \method to two gene expression datasets: a low-dimensional \emph{Arabidopsis thaliana} dataset with 33 genes and a high-dimensional \emph{melanoma} dataset with over 1,000 genes. 
Together, these applications evaluate the effectiveness of \method in both low- and high-dimensional real-world biological settings.

\subsection{Application to \emph{Arabidopsis thaliana} Gene Expression Data}
\label{sec:real-data-Arabidopsis}

In this subsection, we apply \method to the gene expression dataset from Wille et al.~\citep{wille2004sparse}. 
The dataset contains expression measurements of \emph{Arabidopsis thaliana} collected under 118 experimental conditions, including light and dark treatments and exposure to growth hormones. 
In \emph{Arabidopsis thaliana}, isoprenoids are synthesized through two major pathways located in distinct cellular compartments: the cytosolic mevalonate (MVA) pathway and the plastidial non-mevalonate, or MEP, pathway. 
The dataset includes the 33 genes analyzed in~\citep{wille2004sparse}, which are involved in these two pathways or localized in the mitochondrion. 
Although the true causal graph is not available, biological studies and previous graphical-model analyses suggest a modular organization: genes within the MEP and MVA pathways tend to form dense pathway-specific groups, while cross-talk between the two compartments is mediated by a smaller number of bridge loci~\citep{laule2003crosstalk,rodrieguez2004distinct,wille2004sparse,frot2019robust}.

\begin{figure}[hpt!]
    \centering
    \captionsetup[subfigure]{font=small, skip=2pt}

    \begin{subfigure}[t]{0.48\linewidth}
        \centering
        \begin{tikzpicture}

            \draw[densely dashed, rounded corners=8pt, line width=0.5pt]
                (-2.2,-1.5) -- (2.1,-1.5) -- (2.1,-0.50) --
                (-0.90,1.55) -- (-2.2,1.55) --
                (-2.2,0.35) -- cycle;
            \node[font=\scriptsize] at (-1.5,1.75) {MEP};


            \draw (0.0, 0.0) node(MCT) [genetarget, inner sep=2pt, font=\footnotesize] {$\mathrm{MCT}$};

            \draw (-1.5, -1.0) node(DXR) [geneobs, inner sep=2pt, font=\footnotesize] {$\mathrm{DXR}$};
            \draw (1.5, -1.0) node(CMK) [geneobs, inner sep=2pt, font=\footnotesize] {$\mathrm{CMK}$};
            \draw (-1.5, 1.0) node(MECPS) [geneobs, inner sep=2pt, font=\scriptsize] {$\mathrm{MECPS}$};
            \draw (1.5, 1.0) node(FPPS1) [geneobs, inner sep=2pt, font=\scriptsize] {$\mathrm{FPPS1}$};

            \draw[-arcsq]       (MCT) -- (DXR);
            \draw[-arcsq]       (MCT) -- (CMK);
            \draw[arcsq-arcsq]  (MCT) -- (MECPS);
            \draw[arcsq-arcsq]  (MCT) -- (FPPS1);

        \end{tikzpicture}
        \caption{$\target = \mathrm{MCT}$}
        \label{fig:local-pag-MCT}
    \end{subfigure}
    \hfill
    \begin{subfigure}[t]{0.48\linewidth}
        \centering
        \begin{tikzpicture}


            \draw[densely dashed, rounded corners=8pt, line width=0.5pt]
                (2.2,1.5) -- (-2.1,1.5) -- (-2.1,0.50) --
                (0.90,-1.55) -- (2.2,-1.55) --
                (2.2,-0.35) -- cycle;
            \node[font=\scriptsize] at (1.50,1.75) {MVA};

            \draw (0.0, 0.0) node(HMGS) [genetarget, inner sep=2pt, font=\footnotesize] {$\mathrm{HMGS}$};

            \draw (-1.5, -1.0) node(PPDS2) [geneobs, inner sep=2pt, font=\scriptsize] {$\mathrm{PPDS2}$};
            \draw (1.5, -1.0) node(FPPS1) [geneobs, inner sep=2pt, font=\scriptsize] {$\mathrm{FPPS1}$};
            \draw (-1.5, 1.0) node(FPPS2) [geneobs, inner sep=2pt, font=\scriptsize] {$\mathrm{FPPS2}$};
            \draw (1.5, 1.0) node(DPPS1) [geneobs, inner sep=2pt, font=\scriptsize] {$\mathrm{DPPS1}$};

            \draw[-arcsq]       (HMGS) -- (PPDS2);
            \draw[arcsq-arcsq]  (HMGS) -- (FPPS1);
            \draw[-arcsq]       (HMGS) -- (FPPS2);
            \draw[arcsq-arcsq]  (HMGS) -- (DPPS1);
        \end{tikzpicture}
        \caption{$\target = \mathrm{HMGS}$}
        \label{fig:local-pag-HMGS}
    \end{subfigure}

    \caption{
    Local causal structures learned for representative \emph{Arabidopsis thaliana} target genes.
    Dashed contours indicate genes from the same pathway group.
    (a) Target $\mathrm{MCT}$;
    (b) target $\mathrm{HMGS}$.
    }
    \label{fig:main-text-Arabidopsis}
\end{figure}
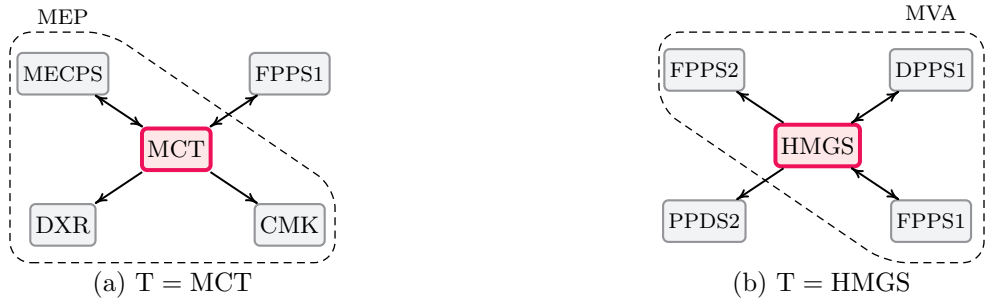

\textbf{Results.}
\cref{fig:main-text-Arabidopsis} reports the learned local structure for two representative target genes, $\mathrm{MCT}$ and $\mathrm{HMGS}$. 
Comparing our results with the known isoprenoid metabolic organization and the graphical models reported by Wille et al.~\citep[Figure~3]{wille2004sparse}, we observe several biologically meaningful patterns.
\begin{itemize}[leftmargin=10pt,itemsep=2pt,topsep=2pt,parsep=0pt]
    \item 
    \emph{(1) Intra-pathway Connections:} 
    The local neighborhoods recovered by \method are mainly concentrated within the corresponding biosynthetic pathways.
    For the MEP target $\mathrm{MCT}$, the recovered neighbors include $\mathrm{DXR}$, $\mathrm{CMK}$, and $\mathrm{MECPS}$, which form a closely connected segment of the MEP pathway and agree with the pathway structure reported in~\citep{wille2004sparse}.
    For the MVA-related target $\mathrm{HMGS}$, the recovered neighbors include $\mathrm{FPPS1}$, $\mathrm{FPPS2}$, and $\mathrm{DPPS1}$, which are located in the cytosolic MVA pathway. The recovered edges are also consistent with~\citep{wille2004sparse}, except for the edge involving $\mathrm{DPPS1}$.

    \item
    \emph{(2) Cross-pathway and Orientation Patterns:}
    Beyond pathway-local connections, \method also identifies sparse links across pathway regions, such as $\mathrm{MCT}$--$\mathrm{FPPS1}$ and $\mathrm{HMGS}$--$\mathrm{PPDS2}$. These links are compatible with previously reported cross-talk between plastidial and cytosolic isoprenoid biosynthesis~\citep{laule2003crosstalk,rodrieguez2004distinct,wille2004sparse}. Moreover, the recovered edge-mark information between $\mathrm{MCT}$ and $\mathrm{CMK}$, and between $\mathrm{HMGS}$ and $\mathrm{FPPS2}$, is consistent with known biological relationships~\citep{wille2004sparse}. The remaining bi-directed edges indicate potential latent confounding, possibly arising from unmeasured genes or regulatory factors outside the observed gene set.

\end{itemize}
Overall, these results show that \method can recover compact target-specific structures that preserve the main modular organization of the isoprenoid network and retain biologically plausible cross-pathway information. 
Additional local structures for $\mathrm{DXR}$, $\mathrm{PPDS1}$, and $\mathrm{MECPS}$ are provided in \cref{app:supplementary-real-data-Arabidopsis}.

\subsection{Application to \emph{Melanoma} Gene Expression Data}
\label{sec:real-data-melanoma}

To further evaluate the performance of \method on high-dimensional real-world data, we apply it to the melanoma (cancer) gene expression dataset from Frangieh et al.~\citep{frangieh2021multimodal}. The dataset contains $218{,}331$ melanoma cells and $23{,}712$ genes, with perturbations targeting $249$ genes. For each cell, the dataset records the perturbation identity, i.e., the perturbation target or no perturbation, together with a genome-wide gene expression count vector. Cells were measured under three experimental conditions: co-culture with patient-derived tumor-infiltrating T cells, which can recognize and kill melanoma cells; interferon (IFN)-$\gamma$ stimulation; and a control condition. We treat these three conditions as separate datasets.

Following~\citep{versteeg2022local,chevalley2025large,li2026root}, we use only the non-perturbed cells in each condition to learn the local causal structure around a target gene, and evaluate the learned structure using perturbation-induced distributional changes, since no ground-truth causal graph is available for this dataset.
We retain genes that show nonconstant expression and are expressed in more than $60\%$ of cells within each condition. After preprocessing, the three condition-specific datasets contain $5{,}039$ cells and $1{,}494$ genes for the control condition, $10{,}403$ cells and $1{,}139$ genes for IFN-$\gamma$ stimulation, and $7{,}586$ cells and $1{,}640$ genes for co-culture. We then learn target-specific local causal structures separately for each condition.

For evaluation, we use the two-sided Mann--Whitney U rank test~\citep{mann1947test,nachar2008mann,chevalley2025large} to assess whether each predicted definite causal mark in the learned local structures is supported by perturbation-induced distributional changes. 
Specifically, a predicted mark is counted as a true positive if perturbing the corresponding source gene leads to a significant distributional change in the target gene. 
We define the Interventional Validation Rate (IVR) as the fraction of predicted local definite causal marks validated by perturbation data. 
A higher IVR indicates stronger support from interventional evidence.

\textbf{Target genes.}
We report the main results for $\mathrm{B2M}$ in \cref{tab:gene-target-performance-B2M}. 
$\mathrm{B2M}$ is a key component of MHC class I antigen presentation and is directly related to tumor immune recognition; moreover, it is a biologically meaningful target in the cancer immune-evasion setting studied by Frangieh et al.~\citep{frangieh2021multimodal}. 
Additional results for $\mathrm{HLA}$-$\mathrm{A}$, another antigen-presentation gene relevant to immune response, and $\mathrm{CD59}$, a complement-regulatory gene related to immune escape, are provided in \cref{app:supplementary-melanoma-gene-expression}.

\begin{table}[hpt!]
\centering
\caption{Comparison on melanoma gene expression data for target $\mathrm{B2M}$.}
\label{tab:gene-target-performance-B2M}
\begin{threeparttable}
\begin{tabular}{l l C{2.5cm} C{2.5cm} C{2.0cm}}
\toprule
\textbf{Condition} & \textbf{Algorithm} & \textbf{Runtime $\downarrow$} & 
\textbf{\#CI Tests $\downarrow$} & \textbf{IVR $\uparrow$} \\
\midrule
\multirow{3}{*}{Co-culture} 
    & L-MARVEL & 1717.59 & 13064931 & 0.54 \\
    & GraN-LCS & 2172.13 & -- & 0.33 \\
    & \cellcolor{gray!10}LoCaLS (Ours) & \cellcolor{gray!10}\textbf{13.35} & \cellcolor{gray!10}\textbf{177727} & \cellcolor{gray!10}\textbf{0.91} \\
\midrule
\multirow{3}{*}{IFN-$\gamma$} 
    & L-MARVEL & 810.05 & 5632375 & 0.52 \\
    & GraN-LCS & 1888.14 & -- & 0.33 \\
    & \cellcolor{gray!10}LoCaLS (Ours) & \cellcolor{gray!10}\textbf{52.10} & \cellcolor{gray!10}\textbf{886062} & \cellcolor{gray!10}\textbf{0.70} \\
\midrule
\multirow{3}{*}{Control} 
    & L-MARVEL & 19.58 & 1194846 & 0.50 \\
    & GraN-LCS & 1909.46 & -- & 0.47 \\
    & \cellcolor{gray!10}LoCaLS (Ours) & \cellcolor{gray!10}\textbf{4.76} & \cellcolor{gray!10}\textbf{50286} & \cellcolor{gray!10}\textbf{0.57} \\
\bottomrule
\end{tabular}
\begin{tablenotes}[flushleft]
\footnotesize
\item Note: 
The best result in each group is highlighted in bold.
\end{tablenotes}
\end{threeparttable}
\end{table}

\textbf{Results.}
Due to the high dimensionality of this dataset, we report results only for L-MARVEL and GraN-LCS as representative global and local baselines, respectively. The remaining compared methods are omitted because their runtime exceeds two hours.
\cref{tab:gene-target-performance-B2M} shows that \method achieves the highest IVR for $\mathrm{B2M}$ under all three conditions. These results indicate that the local structures learned by \method are not only computationally efficient, but also better aligned with perturbation-supported regulatory effects in this single-cell immune-evasion dataset. In contrast, GraN-LCS attains the lowest IVR, which is consistent with its inability to account for latent variables and selection bias. The supplementary results in \cref{app:supplementary-melanoma-gene-expression} provide additional evidence on other immune-relevant targets, further supporting the effectiveness of \method on high-dimensional biological data.

\section{Conclusion and Future Work}

In this work, we studied local causal structure learning in the presence of latent variables and selection bias. 
We proposed \method, a local causal discovery algorithm for identifying the direct causes and effects of a target variable. 
Compared with existing global methods, \method substantially reduces computational cost while remaining sound and complete, under standard assumptions, with respect to the target-specific information identifiable from global causal discovery. 
Compared with existing local methods, \method accommodates both latent variables and selection bias, thereby addressing a more realistic and challenging setting.

Extensive experiments on random structures with varying dimensionalities, graph densities, latent-variable ratios, and selection-variable ratios, demonstrate that \method achieves strong target-specific structural accuracy while substantially reducing runtime and the number of conditional independence tests. 
Experiments on four real-world benchmark structures further confirm its efficiency and reliability. 
Moreover, applications to one low-dimensional and one high-dimensional gene expression datasets with more than $1000$ genes provide additional evidence for the practical utility of \method in real-world biological data analysis.

Several directions remain open for future work. 
First, some causal directions are inherently unidentifiable from purely observational data without additional assumptions, and can only be characterized up to a Markov equivalence class. 
An important direction is therefore to incorporate additional information, such as knowledge of data-generation mechanisms~\citep{kaltenpoth2023nonlinear} or expert knowledge~\citep{wang2023sound,zheng2026local}, to further refine local causal structures. 
Second, combining observational and interventional data~\citep{hauser2015jointly} is a promising direction, as interventional information may help resolve ambiguities that cannot be eliminated using observational conditional independence information alone.

\bibliographystyle{plainnat}
\bibliography{reference.bib}

\clearpage
\appendix

\crefalias{section}{appendix}
\crefalias{subsection}{appendix}
\crefalias{subsubsection}{appendix}

\section*{Appendix Contents}
\addcontentsline{toc}{section}{Appendix Contents}
\startcontents[appendix]
\printcontents[appendix]{}{1}{}

\section{Supplementary Preliminaries}
\label{Appendix:Preliminaries}

\begin{center}
  \begin{table}[H]
    \centering
    \small
    \renewcommand{\arraystretch}{1.2}
    \caption{Main symbols and notation used in this paper.}
    \label{table:list-symbols}
    \begin{tabular}{@{}p{0.20\textwidth}p{0.74\textwidth}@{}}
    \toprule
    \textbf{Symbol} & \textbf{Description}\\ 
    \midrule
    $\vars[V]$ & The full set of variables in the underlying causal system.\\
    $\vars[O]$, $\vars[L]$, $\vars[S]$ & The sets of observed, latent, and selection variables, respectively, with $\vars[V]=\vars[O]\cup\vars[L]\cup\vars[S]$.\\
    $\target\in\vars[O]$ & The target observed variable whose local causal structure is to be recovered.\\
    $|\vars[K]|$ & The cardinality of a variable set $\vars[K]$.\\
    $\DAG$ & The underlying directed acyclic graph (DAG) over $\vars[V]$.\\
    $\MAG$ & The maximal ancestral graph (MAG) induced over the observed variables $\vars[O]$.\\
    $\PAG$ & The ground-truth partial ancestral graph (PAG) over $\vars[O]$.\\
    $\hat{\PAG}$ & The PAG estimate maintained and updated by \method.\\
    $\g[L]_{X}$ & The locally learned graph over the MB-based region $\MBplus(X)$.\\
    $\MB(X,\MAG)$, $\MB(X)$ & The Markov blanket of $X$ in $\MAG$; after identifiability is established, $\MB(X)$ denotes the common MB.\\
    $\MBplus(X)$ & The MB-based local region $\MB(X)\cup\{X\}$.\\
    $\An(X,\g)$, $\Anplus(X,\g)$ & The ancestors of $X$ in $\g$, and the ancestors including $X$ itself.\\
    $\Ant(X,\g)$, $\Antplus(X,\g)$ & The anterior vertices of $X$ in $\g$, and the anterior vertices including $X$ itself.\\
    $\Sepset(X,Y)$, $\Sepsets$ & A separating set for $X$ and $Y$, and the collection of separating sets recorded during learning.\\
    $\vars[X]\CI \vars[Y]\mid \vars[Z]$ & $\vars[X]$ and $\vars[Y]$ are conditionally independent given $\vars[Z]$; graphically, this corresponds to m-separation under the stated assumptions.\\
    $\vars[X]\nCI \vars[Y]\mid \vars[Z]$ & $\vars[X]$ and $\vars[Y]$ are conditionally dependent given $\vars[Z]$.\\
    $n,d,p,s,r_L,r_S$ & Experimental parameters: number of variables, expected degree, edge probability, density scaling factor, latent ratio, and selection ratio.\\
    \bottomrule
    \end{tabular}
    \end{table}  
\end{center}

We first recall several definitions used throughout the paper.

We refer to graphs with three possible edge marks---tail `$-$', arrowhead `$>$', and circle `$\circ$'---as \emph{partial mixed graphs} (PMGs). Hence, a PMG may contain six types of edges:
$
\rightarrow,\ \leftrightarrow,\ \circ\!\rightarrow,\ \circ\!\--,\ \circ\!\--\!\circ,\ \--
$.

\begin{definition}[\textbf{Potentially Anterior Path}]
\label{def:potentially-anterior-path}
A path $\pi$ between $X$ and $Y$ in a PMG is a \emph{potentially anterior path} from $X$ to $Y$ if there is no arrowhead on the path pointing toward $X$.
\end{definition}

\begin{definition}[\textbf{Causal Markov Condition}~\citep{spirtes2000causation}]
\label{def:markov-condition}
    Let $\g$ be a causal graph with vertex set $\mathbf{V}$, and let $P(\mathbf{V})$ denote a probability distribution over $\mathbf{V}$ generated by the causal structure represented by $\g$. 
    We say that $P(\mathbf{V})$ satisfies the \emph{Causal Markov Condition} with respect to $\g$ if, for any triplet of disjoint subsets $\mathbf{X}, \mathbf{Y}, \mathbf{Z} \subseteq \mathbf{V}$, 
    whenever $\mathbf{X}$ and $\mathbf{Y}$ are m-separated by $\mathbf{Z}$ in $\g$, then $\mathbf{X}$ and $\mathbf{Y}$ are conditionally independent given $\mathbf{Z}$ in $P(\mathbf{V})$. 
\end{definition}

\begin{definition}[\textbf{Causal Faithfulness Condition}~\citep{spirtes2000causation}]
\label{def:faithfulness-condition}
    Let $\g$ be a causal graph with vertex set $\mathbf{V}$, and let $P(\mathbf{V})$ denote a probability distribution over $\mathbf{V}$ generated by the causal structure represented by $\g$. 
    We say that $P(\mathbf{V})$ satisfies the \emph{Causal Faithfulness Condition} with respect to $\g$ if, for any triplet of disjoint subsets $\mathbf{X}, \mathbf{Y}, \mathbf{Z} \subseteq \mathbf{V}$, 
    whenever $\mathbf{X}$ and $\mathbf{Y}$ are conditionally independent given $\mathbf{Z}$ in $P(\mathbf{V})$, then $\mathbf{X}$ and $\mathbf{Y}$ are m-separated by $\mathbf{Z}$ in $\g$. 
\end{definition}

Under these two conditions, the conditional independence relations among the observed variables correspond exactly to m-separation relations in the causal graph. 
Accordingly, when the context is clear, we use the notation $\CI$ interchangeably to denote conditional independence in the distribution and m-separation in MAGs.

The following notions are mainly used in the proofs.

\begin{definition}[\textbf{Ancestors and Anterior Vertices}]
\label{def:ancestors-anterior}
Let $\g$ be a mixed graph and let $X,Y$ be vertices in $\g$. We say that $X$ is an \emph{ancestor} of $Y$ if there exists a directed path
$
X\rightarrow\cdots\rightarrow Y
$
in $\g$. The set of all ancestors of $X$ in $\g$ is denoted by $\An(X,\g)$.

We say that $X$ is \emph{anterior} to $Y$ if there exists a path from $X$ to $Y$ consisting only of undirected ($\--$) and directed ($\rightarrow$) edges, such that no edge on the path is oriented toward $X$. The set of all anterior vertices of $X$ in $\g$ is denoted by $\Ant(X,\g)$.
\end{definition}

For a vertex set $\vars[X]$, we define 
\[ \An(\vars[X],\g)=\bigcup_{X\in\vars[X]}\An(X,\g), \qquad \Anplus(\vars[X],\g)=\An(\vars[X],\g)\cup\vars[X], \] 
and similarly, 
\[ \Ant(\vars[X],\g)=\bigcup_{X\in\vars[X]}\Ant(X,\g), \qquad \Antplus(\vars[X],\g)=\Ant(\vars[X],\g)\cup\vars[X]. 
\]

\begin{definition}[\textbf{Induced Subgraph}]
    \label{def:induced-graph}
    Given a graph $\g = (\mathbf{V}, \mathbf{E})$ and a subset $\mathbf{V}' \subseteq \mathbf{V}$, the \emph{induced subgraph} of $\g$ on $\mathbf{V}'$, denoted by $\g[G][\mathbf{V}']$, is the graph whose vertex set is $\mathbf{V}'$ and whose edge set consists of all edges in $\g$ with both endpoints in $\mathbf{V}'$.
\end{definition}

\begin{definition}[\textbf{Augmented Graph}]
    \label{def:augmented-graph}
    Given a mixed graph $\g$, the \emph{augmented graph} of $\g$, denoted by $\aug[\g]$, is an undirected graph over the same set of vertices as $\g$ such that two vertices are adjacent in $\aug[\g]$ if and only if they are collider connected in $\g$.
\end{definition}

Two vertices $X$ and $Y$ are \emph{collider connected} in a mixed graph $\g$ if there exists a path between $X$ and $Y$ in which every non-endpoint vertex is a collider, including the case where $X$ and $Y$ are adjacent in $\g$.

In an undirected graph, a vertex $X$ is said to be \emph{separated} from a vertex $Y$ by a set of vertices $\mathbf{Z}$ if every path between $X$ and $Y$ contains at least one vertex in $\mathbf{Z}$.
More generally, two vertex sets $\mathbf{X}$ and $\mathbf{Y}$ are said to be separated by $\mathbf{Z}$ if, for every pair $(X, Y)$ with $X \in \vars[\mathbf{X}]$ and $Y \in \vars[\mathbf{Y}]$, the vertices $X$ and $Y$ are separated by $\mathbf{Z}$.

\section{Additional Background on MAGs Induced by DAGs}
\label{Appendix:DAG-to-MAG}

This appendix provides additional background on how a DAG with latent variables and selection variables induces a MAG over the observed variables. This construction complements the discussion in \cref{sec:mag-pag}.

Let $\DAG$ be an underlying DAG over $\mathbf{V}=\mathbf{O}\cup\mathbf{L}\cup\mathbf{S}$, where $\mathbf{O}$ denotes observed variables, $\mathbf{L}$ denotes latent variables, and $\mathbf{S}$ denotes selection variables. Since variables in $\mathbf{L}$ are unobserved and variables in $\mathbf{S}$ are conditioned on, the conditional independence information available from data is of the form
$
\mathbf{X}\CI \mathbf{Y}\mid \mathbf{Z}
$
under $P_{\mathrm{obs}}(\mathbf{O})=P(\mathbf{O}\mid \mathbf{S}=\mathbf{s})$, for disjoint sets $\mathbf{X},\mathbf{Y},\mathbf{Z}\subseteq\mathbf{O}$. Equivalently, this corresponds to conditional independences among observed variables after marginalizing over $\mathbf{L}$ and conditioning on $\mathbf{S}$.

A key property of MAGs is that they represent these observable conditional independence relations without explicitly including the latent and selection variables. The following construction gives the induced MAG over $\mathbf{O}$.

\begin{definition}[\textbf{Inducing Path Relative to $(\mathbf{L},\mathbf{S})$}]
\label{def:inducing-path-relative}
Let $\DAG$ be a DAG over $\mathbf{O}\cup\mathbf{L}\cup\mathbf{S}$, and let $X,Y\in\mathbf{O}$. A path between $X$ and $Y$ in $\DAG$ is called an \emph{inducing path relative to $(\mathbf{L},\mathbf{S})$} if every non-endpoint vertex on the path is either in $\mathbf{L}$ or is a collider, and every collider on the path is an ancestor of $X$, $Y$, or some variable in $\mathbf{S}$.
\end{definition}

\begin{algorithm}[hpt!]
\caption{\textsc{InduceMAG} from an Underlying DAG}
\label{alg:dag-to-mag}
\begin{algorithmic}[1]
\REQUIRE A DAG $\DAG$ over $\mathbf{V}=\mathbf{O}\cup\mathbf{L}\cup\mathbf{S}$.
\ENSURE The induced MAG $\MAG$ over $\mathbf{O}$.

\STATE Initialize $\MAG$ as an empty mixed graph over vertex set $\mathbf{O}$.

\FORALL{distinct pairs $A,B\in\mathbf{O}$}
    \IF{there exists an inducing path relative to $(\mathbf{L},\mathbf{S})$ between $A$ and $B$ in $\DAG$}
        \STATE Add an edge between $A$ and $B$ in $\MAG$.
    \ENDIF
\ENDFOR

\FORALL{adjacent pairs $A,B$ in $\MAG$}
    \IF{$A\in\An(B\cup\mathbf{S},\DAG)$ and $B\notin\An(A\cup\mathbf{S},\DAG)$}
        \STATE Orient the edge as $A\rightarrow B$ in $\MAG$.
    \ELSIF{$B\in\An(A\cup\mathbf{S},\DAG)$ and $A\notin\An(B\cup\mathbf{S},\DAG)$}
        \STATE Orient the edge as $A\leftarrow B$ in $\MAG$.
    \ELSIF{$A\notin\An(B\cup\mathbf{S},\DAG)$ and $B\notin\An(A\cup\mathbf{S},\DAG)$}
        \STATE Orient the edge as $A\leftrightarrow B$ in $\MAG$.
    \ELSIF{$A\in\An(B\cup\mathbf{S},\DAG)$ and $B\in\An(A\cup\mathbf{S},\DAG)$}
        \STATE Orient the edge as $A\--B$ in $\MAG$.
    \ENDIF
\ENDFOR

\STATE \textbf{return} $\MAG$.
\end{algorithmic}
\end{algorithm}

The MAG $\MAG$ constructed by \cref{alg:dag-to-mag} preserves the observable conditional independence structure of the underlying DAG. More precisely, for any disjoint sets $\mathbf{X},\mathbf{Y},\mathbf{Z}\subseteq\mathbf{O}$, $\mathbf{X}$ and $\mathbf{Y}$ are $m$-separated by $\mathbf{Z}$ in $\MAG$ if and only if they are $d$-separated by $\mathbf{Z}\cup\mathbf{S}$ in $\DAG$ after marginalizing over $\mathbf{L}$. Thus, under the standard causal Markov and faithfulness assumptions, $\MAG$ represents the conditional independence relations in $P_{\mathrm{obs}}(\mathbf{O})$.

The edge marks in the induced MAG also carry qualitative ancestral information inherited from the underlying DAG. Specifically, for adjacent observed variables $A$ and $B$:
\begin{itemize}[leftmargin=16pt,itemsep=1pt,topsep=2pt]
    \item $A\rightarrow B$ indicates that $A$ is an ancestor of $B$ or of some selection variable, whereas $B$ is not an ancestor of $A$ or of any selection variable;
    \item $A\leftrightarrow B$ indicates that neither $A$ nor $B$ is an ancestor of the other or of any selection variable; in the presence of adjacency, this is typically interpreted as evidence of latent confounding between $A$ and $B$;
    \item $A\--B$ indicates that both $A$ and $B$ are ancestors of the other variable or of some selection variable. Since the underlying causal graph is acyclic, such undirected edges arise from selection effects.
\end{itemize}

Therefore, the induced MAG provides a compact representation of both the conditional independence structure and the qualitative ancestral information among observed variables, after latent variables have been marginalized out and selection variables have been conditioned on.

\section{Proofs}
\label{app:proofs}

\newcommand{\proofdir}[1]{\noindent(#1)\;}

\subsection{Proofs of Results on Markov Blanket in MAGs}
\label{app:Markov-blanket}

\begin{proof}[Proof of~\cref{lemma:MB-MAG-property}]
\label{proof:lemma:MB-MAG-property}
    Let $Y\in \vars[O]\setminus \MBplus(X,\MAG)$. Suppose there exists an m-connecting path $\pi = \langle X = V_0, V_1, \dots, V_k = Y \rangle$ between $X$ and $Y$ given $\MB(X,\MAG)$. We analyze cases according to the form of $\pi$:
    \begin{itemize}
    \item[(1)] $X \leftarrow V_1 \cdots Y$ or $X - V_1 \cdots Y$.
    Then $V_1$ is respectively a parent or a neighbor of $X$, hence $V_1\in \MB(X,\MAG)$. Since $V_1$ is a non-collider on $\pi$, the path cannot be m-connecting given $\MB(X,\MAG)$.

    \item[(2)]$X \leftrightarrow V_1 \cdots Y$.
    If all edges on $\pi$ are bidirected, then $Y$ lies in the bidirected district of $X$, hence $Y\in \MB(X,\MAG)$, contradicting $Y\notin \MBplus(X,\MAG)$. Otherwise, let $i$ be the largest index with $0<i<k$ such that all edges between $X$ and $V_i$ are bidirected. The next edge is either $V_i\!\to\! V_{i+1}$ or $V_i\!\leftarrow\! V_{i+1}$; an undirected edge $V_i - V_{i+1}$ would violate the ancestral property of MAGs. If $V_i\!\to\! V_{i+1}$, then $V_i\in\MB(X,\MAG)$ and, as a non-collider, blocks $\pi$. If $V_i\!\leftarrow\! V_{i+1}$, then $V_{i+1}\in\MB(X,\MAG)$ and, as a non-collider, blocks $\pi$.

    \item[(3)]$X \to V_1 \cdots Y$.
    Let $i$ be the smallest index with $0<i<k$ such that the edge between $V_i$ and $V_{i+1}$ is not $V_i\!\to\! V_{i+1}$. 
    If $i>1$, then $V_1$ is a non-collider and $V_1\in\MB(X,\MAG)$, so $\pi$ is blocked.
    If $i=1$. $V_1\--V_2$ contradicts the ancestral property by $X \to V_1$. 
    If $V_1\!\leftarrow\! V_{2}$, then $V_2\in\MB(X,\MAG)$ and, as a non-collider, blocks $\pi$. 
    If $V_1\!\leftrightarrow\! V_{2}$, we reduce to case~(2), which we have shown cannot yield an m-connecting path given $\MB(X,\MAG)$.
    \end{itemize}
    Thus, in all cases, $\pi$ cannot be m-connecting given $\MB(X,\MAG)$.
\end{proof}

\begin{proof}[Proof of~\cref{lemma:MB-MAG-minimality}]
\label{proof:lemma:MB-MAG-minimality}
    Suppose there is a nonempty set $\vars[Z]\subset\MB(X,\MAG)$ such that $X \CI Y \mid \MB^{\prime}(X,\MAG)$ for all $Y \in \vars[O] \setminus \{\MB^{\prime}(X,\MAG)\cup\{X\}\}$, where $\MB^{\prime}(X,\MAG)=\MB(X,\MAG)\setminus \vars[Z]$. We argue by cases, according to the role of $\vars[Z]$ in the graphical structure. 
    \begin{itemize}
        \item[(1)] If $Z\in \vars[Z]$ is an adjacent vertex of $X$, then $X \CI Z \mid \MB^{\prime}(X,\MAG)$ does not hold, due to the presence of the direct edge between $X$ and $Z$. 
        \item[(2)] If $Z\in \vars[Z]$ is a non-adjacent vertex of $X$, by the definition of $\MB(X,\MAG)$, there exists a collider path $\pi$ between $X$ and $Z$. Without loss of generality, we suppose non-endpoint vertices on $\pi$ are not in $\vars[Z]$, otherwise, we can find a shorter collider path between $X$ and another $Z^{\prime} \in \vars[Z]$. Since all non-endpoint vertices on $\pi$ are colliders and in $\MB^{\prime}(X,\MAG)$, $\pi$ is m-connecting given $\MB^{\prime}(X,\MAG)$. Thus $X \CI Z \mid \MB^{\prime}(X,\MAG)$ not holds.
    \end{itemize}
    In either case, $\MB^{\prime}(X,\MAG)$ fails to separate $X$ from at least one vertex outside $\MB^{\prime}(X,\MAG)\cup \{X\}$, proving minimality.
\end{proof}

\begin{proof}[Proof of~\cref{proposition:MB-MAG}]
\label{proof:proposition:MB-MAG}
    If $Y \in \MB(X,\MAG)$, by the definition of $\MB(X,\MAG)$, $\vars[O]\setminus \{X,Y\}$ does not m-separate $X$ and $Y$. Thus, $X \nCI Y \mid \vars[O] \setminus \{X, Y\}$.

    If $Y \notin \MB(X,\MAG)$, $X$ is m-separated from $Y$ given $\vars[O] \setminus \{X, Y\}$. Thus, $X \CI Y \mid \vars[O] \setminus \{X, Y\}$.
\end{proof}

\subsection{Proof of~\cref{theorem:local-learning-edge}}
\label{app:proof-theo-local-edge}

In this section, we present the proof of~\cref{theorem:local-learning-edge}. The proof relies on several lemmas that establish key properties of m-separation and ancestral relationships in MAGs. The overall structure of the proof is illustrated in~\cref{fig:proof-structure-local-edge}.

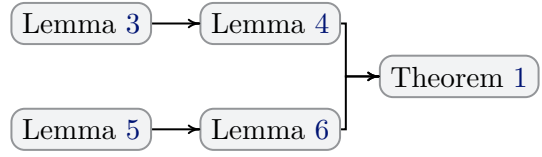
\begin{figure}[hpt!]
    \centering
    \begin{tikzpicture}
    \draw (0.5, 1.4) node(4) [rectangle, draw=Gray, fill=Gray!10, rounded corners=5pt, inner sep=4pt] {\cref{lemma:moral-m-separates}};

    \draw (3.0, 1.4) node(5) [rectangle, draw=Gray, fill=Gray!10, rounded corners=5pt, inner sep=4pt] {\cref{lemma:ancestral-m-separates}};

    \draw (0.5, 0.0)  node(6) [rectangle, draw=Gray, fill=Gray!10, rounded corners=5pt, inner sep=4pt] {\cref{lemma:m-connecting-path}};

    \draw (3.0, 0.0) node(7) [rectangle, draw=Gray, fill=Gray!10, rounded corners=5pt, inner sep=4pt] {\cref{lemma:ancestral-separator}};

    \draw (5.5, 0.7)  node(T) [rectangle, draw=Gray, fill=Gray!10, rounded corners=5pt, inner sep=4pt] {\cref{theorem:local-learning-edge}};

    \draw[-arcsq] (4) -- (5);
    \draw[-arcsq] (6) -- (7);
    \draw[-arcsq] (7) -- (4.0,0.0) --(4.0, 0.7)-- (T);
    \draw[-arcsq] (5) -- (4.0,1.4) --(4.0, 0.7)-- (T);

    \end{tikzpicture}
    
    \caption{Proof structure of~\cref{theorem:local-learning-edge}.}
    \label{fig:proof-structure-local-edge}
\end{figure}


\begin{lemma}
    \label{lemma:moral-m-separates}
    Let $\vars[A]$, $\vars[B]$, and $\vars[C]$ be three disjoint subsets of vertices in a MAG $\MAG$. Then $\vars[A]$ and $\vars[B]$ are m-separated by $\vars[C]$ in $\MAG$ if and only if $\vars[A]$ and $\vars[B]$ are separated by $\vars[C]$ in $\aug[ {\MAG[\Antplus(\vars[A]\cup\vars[B]\cup \vars[C],\MAG)]}]$.
\end{lemma}

\begin{proof}[Proof of~\cref{lemma:moral-m-separates}]
    This is a direct consequence of Theorem 3.18 in~\citep{richardson2002ancestral}.
\end{proof}


\begin{lemma}
    \label{lemma:ancestral-m-separates}
    Let $\vars[S]$ be a subset of vertices in a MAG $\MAG$. For any two vertices $X$ and $Y$ in $\vars[S]$, $X$ and $Y$ are m-separated by a subset of $\vars[S]\setminus\{X,Y\}$ in $\MAG$ if and only if they are m-separated by $\Ant(\{X,Y\},\MAG)\cap \vars[S]$.
\end{lemma}

\begin{proof}[Proof of~\cref{lemma:ancestral-m-separates}]
    Let 
    \[
    \vars[S]^{\prime} = \Ant(\{X,Y\},\MAG)\cap \vars[S].
    \]
    We prove the two directions separately.

    \proofdir{$\Rightarrow$}
    Suppose $X$ and $Y$ are not m-separated by $\vars[S]^{\prime}$ in $\MAG$. 
    Since $\Antplus(\{X,Y\},\MAG) \cup \vars[S]^{\prime} = \Antplus(\{X,Y\},\MAG)$, \cref{lemma:moral-m-separates} implies that $X$ and $Y$ are not separated by $\vars[S]^{\prime}$ in $\aug[\MAG{[\Antplus(\{X,Y\},\MAG)]}]$.
    Thus, there exists a path $\pi$ between $X$ and $Y$ in $\aug[\MAG{[\Antplus(\{X,Y\},\MAG)]}]$ such that no non-endpoint vertex of $\pi$ lies in $\vars[S]^{\prime}$.
    Every non-endpoint vertex of $\pi$ is contained in $\Ant(\{X,Y\},\MAG)$, and $\vars[S]^{\prime} = \Ant(\{X,Y\},\MAG)\cap \vars[S]$. Hence, no non-endpoint vertex of $\pi$ lies in $\vars[S]$.
    Now consider an arbitrary subset $\vars[Z]\subseteq \vars[S]\setminus\{X,Y\}$. Because the same path $\pi$ also exists in each augmented graph $\aug[\MAG{[\Antplus(\{X,Y\} \cup \vars[Z],\MAG)]}]$, the vertices $X$ and $Y$ are not separated by $\vars[Z]$ in any such graph.
    Applying \cref{lemma:moral-m-separates} again, we conclude that $X$ and $Y$ are not m-separated by $\vars[Z]$ in $\MAG$.

    \proofdir{$\Leftarrow$}
    The reverse direction is immediate, since $\vars[S]^{\prime}\subseteq \vars[S]\setminus\{X,Y\}$.

    This completes the proof.
\end{proof}

\begin{lemma}
    \label{lemma:m-connecting-path}
    If $\pi$ is a path m-connecting $X$ and $Y$ given $\vars[Z]$ in a MAG $\MAG$, then every vertex on $\pi$ is in $\Antplus(\{X,Y\}\cup \vars[Z],\MAG)$.
\end{lemma}

\begin{proof}[Proof of~\cref{lemma:m-connecting-path}]
    This is a direct consequence of Lemma 3.13 in~\citep{richardson2002ancestral}.
\end{proof}

\begin{lemma}
    \label{lemma:ancestral-separator}
    Suppose that $\pi$ is a path that connects two nonadjacent vertices $X$ and $Y$. If $\pi$ is not contained completely in $\Antplus(\{X,Y\},\MAG)$, then $\pi$ is m-separated by any subset of $\Ant(\{X,Y\},\MAG)$.
\end{lemma}

\begin{proof}[Proof of~\cref{lemma:ancestral-separator}]
    Since $\pi$ is not contained completely in $\Antplus(\{X,Y\},\MAG)$ and has length $|\pi|>1$, there exists a vertex $V$ on $\pi$ such that
    \[
        V \notin \Antplus(\{X,Y\},\MAG).
    \]
    Let $\mathbf{Z} \subseteq \Ant(\{X,Y\},\MAG)$ be arbitrary. We will show that $\pi$ is not m-connecting $X$ and $Y$ given $\mathbf{Z}$. 

    First note that in a MAG $\MAG$, $\Ant(\cdot,\MAG)$ is closed under taking anteriors, so in particular
    \[
        \Ant(\Ant(\{X,Y\},\MAG),\MAG) 
        \;=\; \Ant(\{X,Y\},\MAG).
    \]
    Because $\mathbf{Z} \subseteq \Ant(\{X,Y\},\MAG)$, monotonicity of $\Ant(\cdot,\MAG)$ implies
    \[
    \begin{aligned}
        &\Ant(\mathbf{Z},\MAG) 
        \subseteq \Ant(\Ant(\{X,Y\},\MAG),\MAG), \\
        &\Ant(\Ant(\{X,Y\},\MAG),\MAG) 
        = \Ant(\{X,Y\},\MAG).
    \end{aligned}
    \]
    Therefore,
    \[
        \Antplus(\{X,Y\} \cup \mathbf{Z},\MAG)
        = \Antplus(\{X,Y\},\MAG).
    \]
    By construction, the vertex $V$ satisfies
    \[
        V \notin \Antplus(\{X,Y\},\MAG)
        = \Antplus(\{X,Y\} \cup \mathbf{Z},\MAG).
    \]
    Consequently, by~\cref{lemma:m-connecting-path}, $\pi$ cannot be m-connecting $X$ and $Y$ given $\mathbf{Z}$. Thus, $\pi$ is m-separated by any subset of $\Ant(\{X,Y\},\MAG)$.
\end{proof}

\begin{proof}[Proof of~\cref{theorem:local-learning-edge}]
    We prove the two directions separately.

    \proofdir{$\Rightarrow$}
    Suppose $X$ and $Y$ are m-separated by a subset of $\vars[O]$. By~\cref{lemma:ancestral-m-separates}, they are m-separated by 
    \[
        \vars[Z] = \Ant(\{X,Y\},\MAG) \cap \vars[O].
    \]
    Define 
    \[
        \vars[S] = \Ant(\{X,Y\},\MAG) \cap \MB(X), \vars[B] = \vars[O]\setminus \MBplus(X)
    \]
    We will show that $X$ and $Y$ are m-separated by $\vars[S]$.

    Consider an arbitrary path $\pi$ between $X$ and $Y$. 
    If $\pi$ is not contained completely in $\Antplus(\{X,Y\},\MAG)$, then by~\cref{lemma:ancestral-separator}, $\pi$ is m-separated by any subset of $\Ant(\{X,Y\},\MAG)$, including $\vars[S]$. 

    Then all vertices on $\pi$ lie in $\Antplus(\{X,Y\},\MAG)$. Suppose $\pi$ is not m-separated by $\vars[S]$. Since $\vars[Z]$ m-separates $\pi$ and $\vars[S] \subseteq \vars[Z]$, there must be at least one vertex $V \in\{ \pi \cap (\vars[Z]\setminus\vars[S])\} \subseteq \vars[B]$. Let $V$ be the closest such vertex to $X$ along $\pi$. Then the subpath $\pi(X,V)$ is also not m-separated by $\vars[S]$. Since all non-endpoint vertices on $\pi(X,V)$ are not in $\vars[B]$, by $\pi(X,V)\subseteq \pi \subseteq \Ant(\{X,Y\},\MAG)$ and the definition of $\vars[S]$, all non-endpoint vertices on $\pi(X,V)$ belong to $\vars[S]$. Thus $\pi(X,V)$ m-connects $X$ and $V$ given $\vars[S]$, and therefore also given $\MB(X)\cup\vars[S] = \MB(X)$. 
    However, the condition $X \CI \vars[B] \mid \MB(X)$, implies $X  \CI B \mid \MB(X)$ for any $B \in \vars[B]$. This contradicts the existence of the m-connecting subpath $\pi(X,V)$ given $\MB(X)$. Hence $\pi$ must be m-separated by $\vars[S]$.
    In conclusion, there is a subset of $\MB(X)$ that m-separates $X$ and $Y$.

    \proofdir{$\Leftarrow$}
    The reverse direction is immediate, since $\MB(X)\subseteq \vars[O]$.

    The proof is complete.

\end{proof}

\subsection{Proof of~\cref{theorem:local-learning-collider-triples,theorem:local-learning-collider-paths}}
\label{app:proof-theo-local-colliders-triples-paths}

In this section, we present the proofs of~\cref{theorem:local-learning-collider-triples} and~\cref{theorem:local-learning-collider-paths}. The proofs rely on the following lemmas that establish key properties of edge marks in MAGs and PAGs, as well as the concept of arrow-collider paths.

\begin{lemma}
    \label{lemma:local-PAG-marks-correct}
    Let $\DAG$ be a DAG over the variable set $\mathbf{V} = \mathbf{O} \cup \mathbf{L} \cup \mathbf{S}$, where $\mathbf{O}$ denotes the set of observed variables, $\mathbf{L}$ denotes the set of latent variables, and $\mathbf{S}$ denotes the set of selection biases. Let $\MAG$ be the MAG that causally represents~\footnote{The MAG obtained by applying the construction procedure described in~\citep{zhang2008completeness} (see \cref{alg:dag-to-mag}), which preserves the causal relations among the observed variables.} $\DAG$ over $\mathbf{O}$, and let $\MAG^{\prime}$ be the MAG that causally represents $\DAG$ over any subset $\mathbf{O}^{\prime} \subset \mathbf{O}$. Then, every edge mark in $\MAG^{\prime}$ encodes the same causal information with respect to its endpoint as in $\MAG$. In particular, for any two adjacent vertices $X, Y \in \mathbf{O}^{\prime}$,
    \begin{itemize}
        \item[(1)] $X \--\!\!* Y$ in $\MAG^{\prime}$ means that $X$ is an ancestor of ${Y}\cup \vars[S]$ in $\DAG$, and hence also in $\MAG$;
        \item[(2)] $X \leftarrow\!\!* Y$ in $\MAG^{\prime}$ means that $X$ is not an ancestor of ${Y}\cup \vars[S]$ in $\DAG$, and hence also in $\MAG$.
    \end{itemize} 
\end{lemma}

\begin{proof}[Proof of~\cref{lemma:local-PAG-marks-correct}]
    The edge marks in a MAG admit causal interpretations~\citep{zhang2008completeness}: a tail at a vertex $X$ on an edge indicates that $X$ is an ancestor of the adjacent vertex or selection biases in $\DAG$, whereas an arrowhead indicates that $X$ is not its ancestor. 
    Constructing $\MAG^{\prime}$ over a subset $\mathbf{O}^{\prime}$ of observed variables 
    corresponds to treating the remaining variables $\mathbf{O}\setminus \mathbf{O}^{\prime}$ as latent. 
    By the definition of MAGs, the causal semantics of edge marks are preserved under such a projection. 
    Therefore, if $X \--\!\!* Y$ appears in $\MAG^{\prime}$, the tail at $X$ implies that $X$ is an ancestor of $Y$ in the underlying DAG and hence also in $\MAG$. 
    Conversely, if $X \leftarrow\!\!* Y$ appears in $\MAG^{\prime}$, the arrowhead at $X$ implies that $X$ is not an ancestor of $Y$ in the underlying DAG and thus also not an ancestor of $Y$ in $\MAG$. 
    This completes the proof.
\end{proof}

Then, we can get a straightforward corollary as follows.
\begin{corollary}
    \label{corollary:local-PAG-marks-correct}
    Let $\MAG$ be a MAG over $\mathbf{O}$, and let $\LocalP[\mathbf{O}^{\prime}]$ denote the local PAG over any subset $\mathbf{O}^{\prime} \subset \mathbf{O}$. Then, every non-circle edge mark in $\LocalP[\mathbf{O}^{\prime}]$ encodes the same causal information as in $\MAG$. Specifically, for any two adjacent vertices $X, Y \in \mathbf{O}^{\prime}$,
    \begin{itemize}[leftmargin=26pt,itemsep=0pt,topsep=0pt,parsep=0pt]
        \item[(1)] $X \--\!\!* Y$ in $\LocalP[\mathbf{O}^{\prime}]$ means that $X$ is an ancestor of ${Y}\cup \vars[S]$,
        \item[(2)] $X \leftarrow\!\!* Y$ in $\LocalP[\mathbf{O}^{\prime}]$ means that $X$ is not an ancestor of ${Y}\cup \vars[S]$.
    \end{itemize}
\end{corollary}

\begin{definition}[\textbf{Arrow-Collider Path}~\citep{li2025local}]
    \label{def:arrow-collider-path}
    In a mixed graph, a path $\pi= \left \langle V_0,\dots,V_n  \right \rangle $ is called an arrow-collider path from $V_0$ to $V_n$
    if every non-endpoint vertex is a collider on $\pi$, and the edge between $V_0$ and $V_1$ points into $V_0$, \ie, $V_0 \leftrightarrow V_1 \dots \leftarrow \!\!* V_n$.  If $n=1$, $\pi$ simplifies to $V_0 \leftarrow \!\!* V_1$.
\end{definition}

\begin{lemma}
\label{lemma:local-define-sepset}
Let $\MAG$ be a MAG over $\mathbf{O}$, and let
$V_1$ and $V_2$ be two non-adjacent vertices in $\MAG$ such that
$V_1 \notin \Ant(V_2,\MAG)$.
Let
\[
\mathbf{Z}
=
\{\ArrColl(V_1,\MAG)
\cap
\Ant(\{V_1,V_2\},\MAG)\}\setminus \{V_1, V_2\},
\]
where $\ArrColl(V_1,\MAG)$ denotes the set of vertices reachable from $V_1$ via
arrow-collider paths in $\MAG$.
Then $\mathbf{Z}$ $m$-separates $V_1$ and $V_2$ in $\MAG$.
\end{lemma}

\begin{proof}[Proof of~\cref{lemma:local-define-sepset}]
    We show that $\mathbf{Z}$ blocks every path between $V_1$ and $V_2$ in $\MAG$.
Let \[\pi = \langle V_1 = A_0, A_1, \dots, A_i, \dots, V_2 = A_n \rangle\] be an arbitrary path between $V_1$ and $V_2$ in $\MAG$.  
We analyze cases according to the form of $\pi$:
\begin{itemize}
    \item If $\pi$ forms $V_1 \leftarrow A_1 \dots V_2$, then $\mathbf{Z}$ blocks the path because $A_1 \in \mathbf{Z}$ and $A_1$ is a non-collider on $\pi$.
    \item If $\pi$ forms $V_1 \leftrightarrow A_1 \dots A_i \dots V_2$, then $\mathbf{Z}$ also blocks the path.  
    Specifically, if all edges on $\pi$ are bi-directed, there must exist some $A_i \notin \Ant(\{V_1,V_2\}, \MAG)$; otherwise, $\pi$ would be an inducing path between $V_1$ and $V_2$ but they are not adjacent.  
    Thus, $\mathbf{Z}$ blocks $\pi$ since $A_i$ is a collider and $A_i \notin \Antplus(\mathbf{Z}, \MAG)$. 
    
    Otherwise, let $i$ ($0 < i < n$) be the largest index such that all edges between $V_1$ and $A_i$ are bi-directed. Then the edge between $A_i$ and $A_{i+1}$ is either $A_i \rightarrow A_{i+1}$ or $A_i \leftarrow A_{i+1}$. The case $A_i \-- A_{i+1}$ is impossible, as it would contradict the ancestral property due to the presence of $A_{i-1} \leftrightarrow A_i$. We proceed by considering the direction of this edge:  
    \begin{itemize}
        \item If $A_i \rightarrow A_{i+1}$ and \[A_i \in \Ant(\{V_1,V_2\}, \MAG),\] then $A_i \in \mathbf{Z}$, so $\mathbf{Z}$ blocks $\pi$ as $A_i$ is a non-collider and $A_i \in \mathbf{Z}$.
          
        \item If $A_i \rightarrow A_{i+1}$ and \[A_i \notin \Ant(\{V_1,V_2\}, \MAG),\] then there must exist a collider on $\pi(A_i, V_2)$, let $A_j$ be the collider that is the nearest to $A_i$ with $i < j < n$. Since \[A_i\in\Ant(A_j,\MAG),\] we have $A_j \notin \Ant(\{V_1,V_2\}, \MAG)$; otherwise, $A_i$ would belong to $\Ant(\{V_1,V_2\}, \MAG)$. Therefore, $\mathbf{Z}$ blocks this path because $A_j$ is a collider on $\pi$ and $A_j \notin \Antplus(\mathbf{Z}, \MAG)$.  
        
        \item If $A_i \leftarrow A_{i+1}$ and \[A_{i+1} \in \Ant(\{V_1,V_2\}, \MAG),\] then $A_{i+1} \in \mathbf{Z}$, so $\mathbf{Z}$ blocks $\pi$ as $A_{i+1}$ is a non-collider and $A_{i+1} \in \mathbf{Z}$.  
        \item If $A_i \leftarrow A_{i+1}$ and \[A_{i+1} \notin \Ant(\{V_1,V_2\}, \MAG),\] then $A_i \notin \Ant(\{V_1,V_2\}, \MAG)$. Hence, $\mathbf{Z}$ blocks this path because $A_i$ is a collider on $\pi$ and $A_i \notin \Antplus(\mathbf{Z}, \MAG)$.
    \end{itemize}
    \item If $\pi$ forms $V_1 \--\!\!* A_1 \dots A_i \dots V_2$. Since $V_1\notin \Ant(V_2,\MAG)$, there must exist at least one collider on $\pi$. Let $A_i$ be the collider nearest to $V_1$. Since $V_1 \in \Ant(A_i, \MAG)$, we have $A_i \notin \Ant(\{V_1,V_2\}, \MAG)$.  
    Therefore, $\mathbf{Z}$ blocks $\pi$ because $A_i$ is a collider and $A_i \notin \Ant(\mathbf{Z}, \MAG)$.
\end{itemize}
Consequently, every path $\pi$ between $V_1$ and $V_2$ is blocked by $\vars[Z]$. Therefore, $\mathbf{Z}$ $m$-separates $V_1$ and $V_2$, which completes the proof.
\end{proof}

We then prove \cref{theorem:local-learning-collider-paths} and \cref{theorem:local-learning-collider-triples} respectively. For notational convenience, let $\Sepset(X,Y)$ denote the set of vertices that m-separates $X$ and $Y$. Let $\PAG$ denote the ground-truth PAG over $\mathbf{O}$.

\begin{proof}[Proof of~\cref{theorem:local-learning-collider-paths}]
    We prove the two directions separately.

\proofdir{$\Rightarrow$}
Without loss of generality, suppose that an uncovered collider path 
\[
    X \ast\!\!\rightarrow V_1 \leftrightarrow \dots \leftarrow\!\!\ast V_i
\]
is identified in $\LocalP[X]$, with $i\ge 2$.   
By definition, all intermediate vertices along this path are colliders, with every consecutive triplet forming an unshielded collider.  

Consider the first triplet $\langle X, V_1, V_2 \rangle$. Since this triplet is identified over $\MBplus(X)$, the following hold:
\begin{enumerate}[label=(\arabic*)]
    \item $\forall \mathbf{S} \subset \MB(X)$, we have $X \nCI V_1 \mid \mathbf{S}$.
    \item $\exists\, \Sepset(X, V_2) \subset \MB(X)$ such that \[X \CI V_2 \mid \Sepset(X, V_2).\]
    \item $\forall \mathbf{S} \subset \MBplus(X)$, we have $V_1 \nCI V_2 \mid \mathbf{S}$.
    \item $V_1 \notin \Sepset(X, V_2)$.
\end{enumerate}
From~\cref{theorem:local-learning-edge}, (1) ensures that the edge between $X$ and $V_1$ in $\LocalP[X]$ is consistent with $\PAG$, and (2) ensures that no edge between $X$ and $V_2$ exists in both.  
Assume for contradiction that the edge between $X$ and $V_1$ in $\PAG$ is $X \ast\!\!- V_1$. Then, together with (3) and (4), we would have $\forall \mathbf{S} \subseteq \MB(X)$, $X \nCI V_2 \mid \mathbf{S}$, contradicting (2).  
Therefore, the edge must be $X \ast\!\!\rightarrow V_1$ in $\PAG$ (this can also be inferred from Corollary~\ref{corollary:local-PAG-marks-correct}).

Next, suppose that the edge between $V_1$ and $V_2$ in $\LocalP[X]$ is spurious, meaning there is no direct edge between $V_1$ and $V_2$ in the underlying $\MAG$, but at least one active path connects them. 
Define \[\mathbf{Z} = \{\ArrColl(V_1, \MAG) \cap \Ant(\{V_1,V_2\}, \MAG)\}\setminus\{V_1,V_2\},\] since $X \ast\!\!\rightarrow V_1$, it follows that \[\ArrColl(V_1, \MAG) \subset \MBplus(X),\] hence $\mathbf{Z} \subset \MBplus(X)$.  
Since $V_1 \leftarrow\!\!* V_2$ in $\LocalP[X]$, by~\cref{corollary:local-PAG-marks-correct}, $V_1$ is not an ancestor of $V_2$ in $\MAG$.
Then, by~\cref{lemma:local-define-sepset}, $\vars[Z] \subset \MBplus(X)$ $m$-separates $V_1$ and $V_2$.
We have thus confirmed that the edge $V_1 \leftarrow\!\!* V_2$ in $\LocalP[X]$ is consistent with $\PAG$.

Next, consider the second triplet $\langle V_1, V_2, V_3 \rangle$. The following holds: 
\begin{enumerate}[label=(\arabic*)]
    \item[(1)] $\forall \mathbf{S} \subset \MBplus(X),\; V_{1} \nCI V_{2} \mid \mathbf{S}$, 
    \item[(2)] $\exists \Sepset(V_{1}, V_{3}) \subset \MB(X),\;$ \[ V_{1} \CI V_{3} \mid \Sepset(V_{1}, V_{3}),\] 
    \item[(3)] $\forall \mathbf{S} \subset \MBplus(X),\; V_{2} \nCI V_{3} \mid \mathbf{S}$, and 
    \item[(4)] $V_{2} \notin \Sepset(V_{1}, V_{3})$. 
\end{enumerate}
As discussed above, the direct edge between $V_{1}$ and $V_{2}$ in $\LocalP[X]$ is consistent with $\PAG$, and the absence of a direct edge between $V_{2}$ and $V_{3}$ in $\LocalP[X]$ is also consistent with $\PAG$. Assuming the edge between $V_1$ and $V_{2}$ in $\PAG$ is $V_{1} \leftarrow V_{2}$, conditions (3) and (4) lead to $\forall \mathbf{S} \subseteq \MB(X), V_{1} \nCI V_{3} \mid \mathbf{S}$, which contradicts condition (2). Therefore, the edge between $V_1$ and $V_2$ must be $V_1 \leftrightarrow V_2$ in $\PAG$.
Similarly, $\ArrColl(V_2, \MAG) \subset \MB(X)$ since $X*\!\!\rightarrow V_{1} \leftrightarrow V_{2}$. Thus, by~\cref{lemma:local-define-sepset}, the direct edge between $V_2$ and $V_3$ in $\LocalP[X]$ is the same as $\PAG$. The edges along this triplet in $\LocalP[X]$ are consistent with $\PAG$, and the analysis extends iteratively to the remaining unshielded collider triples along the path.  
Hence, all uncovered collider paths of the form 
$X \ast\!\!\rightarrow V_i \leftrightarrow \dots \leftarrow\!\!\ast V_j$
identified in $\LocalP[X]$ are consistent with those in the ground-truth $\PAG$.

($\Leftarrow$) 
Conversely, suppose that an uncovered collider path
\[
\pi : \langle X(=V_0) \ast\!\!\rightarrow V_1 \leftrightarrow \dots \leftarrow\!\!\ast V_i \rangle
\]
exists in the ground-truth $\PAG$, where $i \ge 2$. 
Then, for every adjacent pair of vertices on $\pi$, there exists no $\vars[Z]\subseteq \vars[O]$ that m-separates them; consequently, no such $\vars[Z]\subseteq \MB(X)$ m-separates them either.

We consider triples $\langle V_j, V_{j+1}, V_{j+2} \rangle$ along $\pi$ according to the index $j$.

\begin{enumerate}[label=(\arabic*)]

\item If $j=0$, then by~\cref{theorem:local-learning-edge}, there exists a set $\vars[Z]\subset \MB(X)$ that m-separates $V_0$ and $V_2$. Hence, there is no edge between $V_0$ and $V_2$ in $\LocalP[X]$, and $V_1 \notin \vars[Z]$.

\item If $0 < j < i-2$, then $ V_j\leftrightarrow V_{j+1}\leftrightarrow V_{j+2}$. 

If $V_{j+2} \in \Ant(V_j,\MAG)$, then necessarily \[V_j \notin \Ant(V_{j+2},\MAG).\] 
By~\cref{lemma:local-define-sepset}, the set
\[
\vars[Z] = \{\ArrColl(V_j,\MAG) \cap \Ant(\{V_j,V_{j+2}\},\MAG)\} \setminus \{V_j,V_{j+2}\}
\]
m-separates $V_j$ and $V_{j+2}$, and $V_{j+1} \notin \vars[Z]\subset \MB(X)$.

If
\[
V_j \in \Ant(V_{j+2},\MAG) \text{ and } V_{j+2} \notin \Ant(V_j,\MAG),
\]
or
\[
V_j \notin \Ant(V_{j+2},\MAG) \text{ and } V_{j+2} \notin \Ant(V_j,\MAG),
\]
then, by an argument analogous to the above and by~\cref{lemma:local-define-sepset}, there also exists a set $\vars[Z]\subset \MB(X)$ that m-separates $V_j$ and $V_{j+2}$, with $V_{j+1}\notin \vars[Z]$.

Therefore, there is no edge between $V_j$ and $V_{j+2}$ in $\LocalP[X]$, and $V_{j+1}$ does not belong to the separating set of $V_j$ and $V_{j+2}$.

\item If $j=i-2$.

Suppose $V_{i-1} \leftarrow V_i$. Then \[V_{i-2} \notin \Ant(V_i,\MAG)\] must hold; otherwise this would contradict the edge $V_{i-2}\leftrightarrow V_{i-1}$. 
By~\cref{lemma:local-define-sepset}, the set
\[
\vars[Z] = \{\ArrColl(V_{i-2},\MAG) \cap \Ant(\{V_{i-2},V_i\},\MAG)\}\setminus \{V_{i-2},V_i\}
\]
m-separates $V_{i-2}$ and $V_i$, and $V_{i-1} \notin \vars[Z]$.

If $V_{i-1} \leftrightarrow V_i$, then by reasoning analogous to case (2), there exists a set $\vars[Z]\subset \MB(X)$ that m-separates $V_{i-2}$ and $V_i$, with $V_{i-1}\notin \vars[Z]$.

Hence, there is no edge between $V_{i-2}$ and $V_i$ in $\LocalP[X]$, and $V_{i-1}$ does not belong to the separating set of $V_{i-2}$ and $V_i$.

\end{enumerate}

Consequently, all unshielded triples along $\pi$ are identified in $\LocalP[X]$, and for each $j$, the vertex $V_{j+1}$ is not contained in the separating set of $V_j$ and $V_{j+2}$. 
Therefore, the uncovered collider path $\pi$ is correctly identified in $\LocalP[X]$. 
This completes the proof.

\end{proof}

\begin{proof}[Proof of~\cref{theorem:local-learning-collider-triples}]
Without loss of generality, suppose an unshielded collider $V_1 \ast\!\!\rightarrow X \leftarrow\!\!\ast V_2$ is identified in $\LocalP[X]$.  
By~\cref{theorem:local-learning-edge}, the existence of direct edges between $X$ and $V_1$, and between $X$ and $V_2$, is consistent with $\PAG$. Likewise, the absence of a direct edge between $V_1$ and $V_2$ in $\LocalP[X]$ is also consistent with $\PAG$. 

Furthermore, since $\exists\Sepset(V_1, V_2) \subset \MBplus(X)$ and $X \notin \Sepset(V_1, V_2)$, $X$ must be a collider in the triple $\langle V_1, X, V_2 \rangle$ in $\PAG$ (this conclusion also follows from~\cref{corollary:local-PAG-marks-correct}). Hence, the identified collider $V_1 \ast\!\!\rightarrow X \leftarrow\!\!\ast V_2$ in $\LocalP[X]$ agrees with that in $\PAG$.
\end{proof}

\subsection{Proof of~\cref{theorem:completeness-stop}}
\label{app:proof-theo-completeness-stop}

We prove that each stopping rule in \cref{theorem:completeness-stop} is sufficient for terminating the sequential learning procedure without losing target-specific causal information. 
The proof relies on the following elementary observation: removing variables that are irrelevant leaf variables with respect to a retained set does not change the joint distribution over the retained variables.

\begin{lemma}
\label{lemma:relative-leaf-deletion}
Let $\MAG$ be the underlying MAG over the observed variable set $\mathbf{O}$ induced by a DAG $\DAG$, and let $\mathbf{X}\subseteq \mathbf{O}$ be a set of variables to be retained. 
Let $\mathbf{Y}\subseteq \mathbf{O}\setminus \mathbf{X}$ be a set of variables that can be removed in an order $Y_1,\ldots,Y_m$ such that, at each step, $Y_k$ is a leaf relative to the variables retained at that step. 
Then removing $\mathbf{Y}$ does not change the marginal distribution over $\mathbf{X}$, i.e.,
\[
P_{\MAG}(\mathbf{X}) = P_{\MAG\setminus \mathbf{Y}}(\mathbf{X}).
\]
\end{lemma}

\cref{lemma:relative-leaf-deletion} implies that deleting variables outside the retained set does not affect the distribution of interest, provided that they are removed as relative leaves. 
In particular, when a variable $Y\notin\mathbf{X}$ has no remaining descendants that are relevant to $\mathbf{X}$, marginalizing out $Y$ only removes its own local factor and leaves the marginal distribution over $\mathbf{X}$ unchanged. 
Thus, the distribution over $\mathbf{X}$ in the reduced MAG $\MAG\setminus\mathbf{Y}$ is identical to that in the original MAG $\MAG$.

\begin{proof}
It is sufficient to prove the claim for the removal of a single variable. 
Let $Y\in \mathbf{O}\setminus \mathbf{X}$ be a variable to be removed, and let
\[
\mathbf{R}=\mathbf{O}\setminus\{Y\}.
\]
Since $Y$ is a leaf relative to the currently retained variables, it does not appear in the anterior set of any variable in $\mathbf{R}$ under the corresponding underlying causal structure representation. 
Therefore, the factor associated with $Y$ can be separated from the remaining factors when marginalizing over $Y$. 
Specifically,
\begin{align}
P_{\MAG}(\mathbf{X})
&= \sum_{\mathbf{O}\setminus \mathbf{X}} P_{\MAG}(\mathbf{O}) \notag\\
&= \sum_{\mathbf{R}\setminus \mathbf{X}} \sum_y P_{\MAG}(\mathbf{R},y) \notag\\
&= \sum_{\mathbf{R}\setminus \mathbf{X}} P_{\MAG\setminus Y}(\mathbf{R}) \notag\\
&= P_{\MAG\setminus Y}(\mathbf{X}).
\end{align}
Thus, removing such a relative leaf preserves the marginal distribution over $\mathbf{X}$.

Applying the same argument sequentially to $Y_1,\ldots,Y_m$ gives
\[
P_{\MAG}(\mathbf{X})
=
P_{\MAG\setminus Y_1}(\mathbf{X})
=
\cdots
=
P_{\MAG\setminus \mathbf{Y}}(\mathbf{X}),
\]
which proves the claim.
\end{proof}

\begin{proof}[Proof of~\cref{theorem:completeness-stop}]
Throughout the sequential procedure, \method only inserts adjacencies and endpoint marks whose consistency with the global PAG is guaranteed by the local learning results in \cref{sec:foundations-local-learning} (\ie, \cref{corollary:local-learning-edge,corollary:local-PAG-marks-correct,theorem:local-learning-collider-paths,theorem:local-learning-collider-triples}), and \textsc{OrientPAG} applies sound orientation rules.
Thus, every structural mark currently recorded in $\hat{\PAG}$ is sound with respect to the global PAG.

\paragraph{Rule $\mathcal{R}1$}
If all edge marks incident to $T$ have been determined, then the target-specific information of interest has already been fixed.
Since all determined adjacencies and endpoint marks incident to $T$ in $\hat{\PAG}$ are sound with respect to the global PAG, no later local learning step can validly change them.
Hence the direct causes and direct effects of $T$ encoded by the target-specific structure have already been recovered.

\paragraph{Rule $\mathcal{R}2$}
If $\mathrm{Waitlist}$ is empty under the update, then there is no unprocessed vertex in the current $T$-connected active region that is still selected for local learning.
Thus, when $\mathrm{Waitlist}$ is empty, the current local structure around $T$ is equivalent to the target-specific structure obtained from global causal discovery.

\paragraph{Rule $\mathcal{R}3$}
Let $\mathbf{Y}$ be the set of remaining unprocessed variables.
Under $\mathcal{R}3$, these variables lie outside the anterior region that does not affect the target-specific structure around $T$.
Relative to the retained region consisting of $T$ and the currently relevant processed vertices, variables in $\mathbf{Y}$ can be treated as removable relative leaves: marginalizing them out does not change the distribution over the retained region by \cref{lemma:relative-leaf-deletion}.
Hence, further local learning around variables in $\mathbf{Y}$ cannot supply additional conditional independence constraints or orientation information for determining the adjacencies and endpoint marks incident to $T$.

This completes the proof that each stopping rule is sufficient for terminating the sequential procedure without losing target-specific causal information.
\end{proof}

\subsection{Proof of~\cref{theorem:locals-correctness}}
\label{app:proof-theorem-locals-correctness}

Let $\PAG$ denote the global PAG that would be obtained by applying a sound and complete global causal discovery procedure to $P_{\mathrm{obs}}(\mathbf{O})$. 
We first note the following invariant of \method: after each iteration, all adjacencies and endpoint marks inserted into $\hat{\PAG}$ are consistent with those in $\PAG$. 
Indeed, for each processed variable $X$, \cref{corollary:local-learning-edge,corollary:local-PAG-marks-correct} guarantees that the adjacencies incident to $X$ in the locally learned structure $\g[L]_X$ coincide with those in the global PAG. 
Moreover, \method preserves only the local orientation information justified by \cref{theorem:local-learning-collider-paths,theorem:local-learning-collider-triples}, and the subsequent call to \textsc{OrientPAG} applies sound PAG orientation rules. 
Thus, no adjacency or endpoint mark added to $\hat{\PAG}$ contradicts the global PAG.

Since the target variable $T$ is processed first, the adjacencies incident to $T$ are correctly recovered from $\g[L]_T$ by \cref{corollary:local-learning-edge,corollary:local-PAG-marks-correct}. 
Subsequent iterations may further orient endpoint marks incident to $T$, but by the invariant above, every such orientation remains consistent with $\PAG$. 
Therefore, throughout the sequential procedure, the learned target-specific structure in $\hat{\PAG}$ is always sound with respect to the global PAG.

It remains to justify that the algorithm does not stop before all target-specific information has been recovered. 
This follows from \cref{theorem:completeness-stop}. 
If $\mathcal{R}1$ holds, then all endpoint marks incident to $T$ have been determined. 
If $\mathcal{R}2$ holds, then there is no remaining variable selected for further local learning. 
If $\mathcal{R}3$ holds, then every path from an unprocessed variable to $T$ is blocked by a non-potentially-anterior mark, so no remaining variable can provide additional information for resolving the local structure around $T$. 
Hence, upon termination, the local structure around $T$ in $\hat{\PAG}$ contains exactly the adjacencies and identifiable endpoint marks incident to $T$ that would be obtained from the global PAG $\PAG$.

Therefore, \method identifies the direct causes and direct effects of $T$ that are identifiable from $\PAG$, and the returned local structure is consistent with the target-specific structure obtained by global causal discovery.

\section{Additional Technical Details for LoCaLS}

\subsection{Orientation Rules}
\label{app:orientation-rules}
Let $\Sepsets$ denote the collection of all conditional independence relations identified during the local learning process. A separator $\Sepset(X, Y)$ between two variables $X$ and $Y$ recorded in $\Sepsets$ implies that there is no edge between $X$ and $Y$ in the learned graph.

\citep{zhang2008completeness} presented sound and complete orientation rules for learning a PAG. These rules are summarized as~\cref{alg:orient-pag}.

\begin{algorithm}[hpt!]
\caption{\textsc{OrientPAG}}
\label{alg:orient-pag}
\begin{algorithmic}[1]
\REQUIRE A PAG $\PAG$ and a collection of separating sets $\Sepsets$.
\ENSURE The oriented PAG $\PAG$.

\STATE Initialize $\PAG$ with the current skeleton and existing edge marks. \\
\algcomment{Apply any of the following rules whenever its antecedent is satisfied.}
\REPEAT

    \STATE \textbf{Rule 1.}
    If $X *\!\!\rightarrow Y \circ\!\!-\!\!* Z$, $\Sepset(X,Z)\in\Sepsets$, and $Y\in\Sepset(X,Z)$, then orient the triple as $X *\!\!\rightarrow Y \rightarrow Z$.

    \STATE \textbf{Rule 2.}
    If $X\rightarrow Y *\!\!\rightarrow Z$ or $X *\!\!\rightarrow Y\rightarrow Z$, and $X *\!\!-\!\circ Z$, then orient $X *\!\!-\!\circ Z$ as $X *\!\!\rightarrow Z$.

    \STATE \textbf{Rule 3.}
    If $X *\!\!\rightarrow Y \leftarrow\!\!* Z$, $X *\!\!-\!\circ W \circ\!\!-\!\!* Z$, $W *\!\!-\!\circ Y$, $\Sepset(X,Z)\in\Sepsets$, and $W\in\Sepset(X,Z)$, then orient $W *\!\!-\!\circ Y$ as $W *\!\!\rightarrow Y$.

    \STATE \textbf{Rule 4.} \IF{there exists a discriminating path $\pi=\langle W,\ldots,X,Y,Z\rangle$ between $W$ and $Z$ for $Y$, $Y\circ\!\!-\!\!* Z$, and $\Sepset(W,Z)\in\Sepsets$} \IF{$Y\in\Sepset(W,Z)$} \STATE Orient $Y\circ\!\!-\!\!* Z$ as $Y\rightarrow Z$. \ELSE \STATE Orient the triple $\langle X,Y,Z\rangle$ as $X\leftrightarrow Y\leftrightarrow Z$. \ENDIF \ENDIF

    \STATE \textbf{Rule 5.}
    For every remaining edge $X\circ\!\!-\!\circ Y$, if there exists an uncovered circle path $p=\langle X,Z,\ldots,W,Y\rangle$ between $X$ and $Y$ such that $\Sepset(X,W)\in\Sepsets$ and $Y\in\Sepset(X,W)$, and $\Sepset(Y,Z)\in\Sepsets$ and $X\in\Sepset(Y,Z)$, then orient $X\circ\!\!-\!\circ Y$ and every edge on $p$ as undirected edges.

    \STATE \textbf{Rule 6.}
    If $X - Y \circ\!\!-\!\!* Z$, then orient $Y\circ\!\!-\!\!* Z$ as $Y -\!\!* Z$.

    \STATE \textbf{Rule 7.}
    If $X -\!\!\circ Y \circ\!\!-\!\!* Z$, $\Sepset(X,Z)\in\Sepsets$, and $Y\in\Sepset(X,Z)$, then orient $Y\circ\!\!-\!\!* Z$ as $Y -\!\!* Z$.

    \STATE \textbf{Rule 8.}
    If $X\rightarrow Y\rightarrow Z$ or $X-\!\circ Y\rightarrow Z$, and $X\circ\!\!\rightarrow Z$, then orient $X\circ\!\!\rightarrow Z$ as $X\rightarrow Z$.

    \STATE \textbf{Rule 9.}
    If $X\circ\!\!\rightarrow Z$, and $p=\langle X,Y,W,\ldots,Z\rangle$ is an uncovered potentially directed path from $X$ to $Z$ such that $\Sepset(Z,Y)\in\Sepsets$ and $X\in\Sepset(Z,Y)$, then orient $X\circ\!\!\rightarrow Z$ as $X\rightarrow Z$.

    \STATE \textbf{Rule 10.}
    Suppose $X\circ\!\!\rightarrow Z$, $Y\rightarrow Z\leftarrow W$, $\pi_1$ is an uncovered potentially directed path from $X$ to $Y$, and $\pi_2$ is an uncovered potentially directed path from $X$ to $W$. Let $U$ be the vertex adjacent to $X$ on $\pi_1$ ($U$ could be $Y$) and let $V$ be the vertex adjacent to $X$ on $\pi_2$ ($V$ could be $W$). If $U$ and $V$ are distinct, $\Sepset(U,V)\in\Sepsets$, and $X\in\Sepset(U,V)$, then orient $X\circ\!\!\rightarrow Z$ as $X\rightarrow Z$.

\UNTIL{no rule can further orient any edge mark in $\PAG$}

\STATE \textbf{return} $\PAG$.
\end{algorithmic}
\end{algorithm}

\subsection{Update Waitlist}
\label{app:update-waitlist}

\begin{algorithm}[hpt!] 
    \caption{\textsc{UpdateWaitlist}} 
    \label{alg:update-waitlist} 
    \begin{algorithmic}[1] 
        \REQUIRE Target variable $T$, current graph $\hat{\PAG}$, processed set $\mathrm{Donelist}$. 
        \ENSURE Updated set $\mathrm{Waitlist}$. 
        \STATE Initialize $\mathrm{Waitlist}\gets \emptyset$. 
        \FOR{each variable $V$ in $\hat{\PAG}$} 
        \IF{$V$ has a potentially anterior path to $T$ in $\hat{\PAG}$} 
        \STATE Add $V$ to $\mathrm{Waitlist}$. 
        \ENDIF 
        \ENDFOR 
        \STATE Remove all variables in $\mathrm{Donelist}$ from $\mathrm{Waitlist}$. 
        \STATE \textbf{return} $\mathrm{Waitlist}$. 
\end{algorithmic} 
\end{algorithm}

The selective update in \cref{alg:update-waitlist} does not change the completeness argument above. 
It is a refinement of the coarse waitlist update: instead of retaining all unprocessed vertices connected to $T$, it keeps only those that still have a potentially anterior path to $T$ in the current graph $\hat{\PAG}$. 
Vertices discarded by this rule are exactly the vertices covered by the structural redundancy condition in $\mathcal{R}3$, and by \cref{theorem:completeness-stop}, learning their local structures cannot further affect the target-specific structure around $T$. 
Therefore, replacing the coarse update with \cref{alg:update-waitlist} preserves completeness, while making $\mathcal{R}3$ implicit in the waitlist update; in practice, the algorithm only needs to explicitly check $\mathcal{R}1$ and $\mathcal{R}2$.

\FloatBarrier
\section{Supplementary Examples}
\label{app:example-R3}

This example illustrates a learning process that terminates by $\mathcal{R}3$ under a coarse waitlist update and, when using \textsc{UpdateWaitlist}, terminates instead by $\mathcal{R}2$.

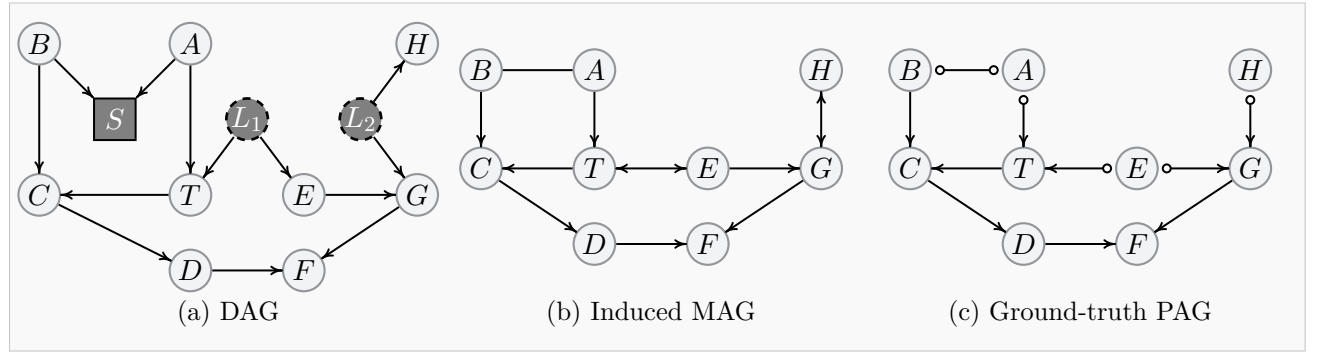
\begin{figure}[H]
    \centering
    \captionsetup[subfigure]{font=small, skip=3pt}
    \setlength{\fboxsep}{8pt}
    \setlength{\fboxrule}{0.4pt}

    {
        \renewcommand\fbox[1]{\fcolorbox{gray!45}{gray!05}{#1}}%

    \fbox{
    \begin{subfigure}[c]{0.30\linewidth}
        \centering
        \fixedtikzbox{%
        \begin{tikzpicture}[x=1.0cm,y=1.0cm]

            \draw (0.0, 2.0) node(B)  [obs] {$B$};
            \draw (2.0, 2.0) node(A)  [obs] {$A$};
             \draw (5.0, 2.0) node(H)   [obs] {$H$};
            
            \draw (1.0, 1.0) node(S) [selection] {$S$};
            
            \draw (2.0, 0.0) node(T)  [obs] {$T$};
            \draw (3.5, 0.0) node(E)   [obs] {$E$};
            \draw (5.0, 0.0) node(G)   [obs] {$G$};

            \draw (2.75, 1.0) node(L1) [latent] {$L_1$};
            \draw (4.25, 1.0) node(L2) [latent] {$L_2$};

            \draw (0.0, 0.0) node(C)  [obs] {$C$};

            \draw (2.0, -1.0) node(D)   [obs] {$D$};
            \draw (3.5, -1.0) node(F)   [obs] {$F$};

            \draw[-arcsq] (A) -- (S);
            \draw[-arcsq] (B) -- (S);
            \draw[-arcsq] (B) -- (C);

            \draw[-arcsq] (A) -- (T);
            \draw[-arcsq] (L1) -- (T);
            \draw[-arcsq] (L1) -- (E);

            \draw[-arcsq] (T) -- (C);
            \draw[-arcsq] (C) -- (D);
            \draw[-arcsq] (D) -- (F);

            \draw[-arcsq] (E) -- (G);
            \draw[-arcsq] (L2) -- (G);
            \draw[-arcsq] (L2) -- (H);
            \draw[-arcsq] (G) -- (F);

        \end{tikzpicture}%
        }
        \caption{DAG}
        \label{fig:exampleR3-dag}
    \end{subfigure}
    \hspace{1em}
    \begin{subfigure}[c]{0.30\linewidth}
        \centering
        \fixedtikzbox{%
        \begin{tikzpicture}[x=1.0cm,y=1.0cm]

            \draw (0.5, 1.3) node(B)  [obs] {$B$};
            \draw (2.0, 1.3) node(A)  [obs] {$A$};
            \draw (5.0, 1.3) node(H)   [obs] {$H$};

            \draw (2.0, 0.0) node(T)  [obs] {$T$};
            \draw (3.5, 0.0) node(E)   [obs] {$E$};
            \draw (5.0, 0.0) node(G)   [obs] {$G$};
            \draw (0.5, 0.0) node(C)  [obs] {$C$};
            
            \draw (2.0, -1.0) node(D)   [obs] {$D$};
            \draw (3.5, -1.0) node(F)   [obs] {$F$};

            \draw[-] (A) -- (B);

            \draw[-arcsq] (A) -- (T);
            \draw[-arcsq] (B) -- (C);
            \draw[-arcsq] (T) -- (C);

            \draw[-arcsq] (C) -- (D);
            \draw[-arcsq] (D) -- (F);

            \draw[arcsq-arcsq] (T) -- (E);
            \draw[-arcsq] (E) -- (G);

            \draw[arcsq-arcsq] (G) -- (H);
            \draw[-arcsq] (G) -- (F);

        \end{tikzpicture}%
        }
        \caption{Induced MAG}
        \label{fig:exampleR3-mag}
    \end{subfigure}
    \hspace{1em}
    \begin{subfigure}[c]{0.31\linewidth}
        \centering
        \fixedtikzbox{%
        \begin{tikzpicture}[x=1.0cm,y=1.0cm]

            \draw (0.5, 1.3) node(B)  [obs] {$B$};
            \draw (2.0, 1.3) node(A)  [obs] {$A$};
            \draw (5.0, 1.3) node(H)   [obs] {$H$};

            \draw (2.0, 0.0) node(T)  [obs] {$T$};
            \draw (3.5, 0.0) node(E)   [obs] {$E$};
            \draw (5.0, 0.0) node(G)   [obs] {$G$};
            \draw (0.5, 0.0) node(C)  [obs] {$C$};
            
            \draw (2.0, -1.0) node(D)   [obs] {$D$};
            \draw (3.5, -1.0) node(F)   [obs] {$F$};

            \draw[circle-circle] (A) -- (B);

            \draw[circle-arcsq] (A) -- (T);
            \draw[-arcsq] (B) -- (C);
            \draw[-arcsq] (T) -- (C);

            \draw[-arcsq] (C) -- (D);
            \draw[-arcsq] (D) -- (F);

            \draw[circle-arcsq] (E) -- (T);
            \draw[circle-arcsq] (E) -- (G);

            \draw[circle-arcsq] (H) -- (G);
            \draw[-arcsq] (G) -- (F);

        \end{tikzpicture}%
        }
        \caption{Ground-truth PAG}
        \label{fig:exampleR3-pag}
    \end{subfigure}

    }}

    \caption{Underlying causal graphs for \cref{example:exampleR3}: (a) the DAG, where $L_1$ and $L_2$ are latent variables and $S$ is an unobserved selection variable; (b) the induced MAG over the observed variables; and (c) the corresponding ground-truth PAG.}
    \label{fig:exampleR3-graphs}
\end{figure}

\begin{figure}[H]
    \centering
    \captionsetup[subfigure]{font=small, skip=2pt}

    {\renewcommand\fbox[1]{\fcolorbox{gray!45}{gray!05}{#1}}%
    \setlength{\seqfigbodyheight}{2.0cm}

    \fbox{%
    \begin{minipage}[c]{0.92\linewidth}
        \centering

        \begin{subfigure}[c]{0.45\linewidth}
            \centering
            \fixedseqtikzbox{%
            \begin{tikzpicture}[x=1.0cm,y=1.0cm]
                \draw[fill=red!80, fill opacity=0.35, draw=none]
                    (2.0, 0.0) ellipse [x radius=0.5cm, y radius=0.4cm];

                \draw (0.5, 1.3) node(B)  [obs] {$B$};
                \draw (2.0, 1.3) node(A)  [obs] {$A$};

                \draw (2.0, 0.0) node(T)  [obs] {$T$};
                \draw (3.5, 0.0) node(E)   [obs] {$E$};
                \draw (0.5, 0.0) node(C)  [obs] {$C$};

                \draw[circle-circle,red] (A) -- (B);
                \draw[circle-arcsq] (A) -- (T);
                \draw[circle-arcsq] (E) -- (T);
                \draw[circle-arcsq] (T) -- (C);
                \draw[circle-arcsq] (B) -- (C);

            \end{tikzpicture}%
            }
            \caption{$\g[L]_{T}$}
            \label{fig:exampleR3-inferred-local-pag-T}
        \end{subfigure}
        \hfill
        \begin{subfigure}[c]{0.45\linewidth}
            \centering
            \fixedseqtikzbox{%
            \begin{tikzpicture}[x=1.0cm,y=1.0cm]
                \draw (0.5, 1.3) node(B)  [obs] {$B$};
                \draw (2.0, 1.3) node(A)  [obs] {$A$};

                \draw (2.0, 0.0) node(T)  [target] {$T$};
                \draw (3.5, 0.0) node(E)   [obs] {$E$};
                \draw (0.5, 0.0) node(C)  [obs] {$C$};

                \draw[circle-arcsq] (A) -- (T);
                \draw[circle-arcsq] (E) -- (T);
                \draw[-arcsq] (T) -- (C);
                \draw[circle-arcsq] (B) -- (C);
            \end{tikzpicture}%
            }
            \caption{$\hat{\PAG}$ after learning $\g[L]_{T}$}
            \label{fig:exampleR3-hat-pag-after-T}
        \end{subfigure}

    \end{minipage}}

    \vspace{0mm}

    \begin{tikzpicture}
        \draw[-stealth, line width=1.5pt] (0,0.35) -- (0,-0.35);
        \node[font=\small, fill=white, inner sep=4pt, align=center,
              draw=black!25, rounded corners=4pt]
              at (-2.4,0) {$\mathrm{Waitlist}=\{A,E\}$};
        \node[font=\small, fill=white, inner sep=4pt, align=center,
              draw=black!25, rounded corners=4pt]
              at (2.4,0) {$\mathrm{Donelist}=\{T\}$};
    \end{tikzpicture}

    \vspace{0mm}

    \fbox{%
    \begin{minipage}[c]{0.92\linewidth}
        \centering

        \begin{subfigure}[c]{0.45\linewidth}
            \centering
            \fixedseqtikzbox{%
            \begin{tikzpicture}[x=1.0cm,y=1.0cm]
                \draw[fill=red!80, fill opacity=0.35, draw=none]
                    (2.0, 1.3) ellipse [x radius=0.5cm, y radius=0.4cm];

                \draw (0.5, 1.3) node(B)  [obs] {$B$};
                \draw (2.0, 1.3) node(A)  [obs] {$A$};

                \draw (2.0, 0.0) node(T)  [obs] {$T$};
                \draw (3.5, 0.0) node(E)   [obs] {$E$};

     
                \draw[circle-circle] (A) -- (B);
                \draw[circle-arcsq] (A) -- (T);
                \draw[circle-arcsq] (E) -- (T);
            \end{tikzpicture}%
            }
            \caption{$\g[L]_{A}$}
            \label{fig:exampleR3-inferred-local-pag-A}
        \end{subfigure}
        \hfill
        \begin{subfigure}[c]{0.45\linewidth}
            \centering
            \fixedseqtikzbox{%
            \begin{tikzpicture}[x=1.0cm,y=1.0cm]
                \draw (0.5, 1.3) node(B)  [obs] {$B$};
                \draw (2.0, 1.3) node(A)  [obs] {$A$};

                \draw (2.0, 0.0) node(T)  [target] {$T$};
                \draw (3.5, 0.0) node(E)   [obs] {$E$};
                \draw (0.5, 0.0) node(C)  [obs] {$C$};

                \draw[circle-circle] (A) -- (B);
                \draw[circle-arcsq] (A) -- (T);
                \draw[circle-arcsq] (E) -- (T);
                \draw[-arcsq] (T) -- (C);
                \draw[-arcsq] (B) -- (C);
   
            \end{tikzpicture}%
            }
            \caption{$\hat{\PAG}$ after learning $\g[L]_{A}$}
            \label{fig:exampleR3-hat-pag-after-A}
        \end{subfigure}

    \end{minipage}}

    \vspace{0mm}

    \begin{tikzpicture}
        \draw[-stealth, line width=1.5pt] (0,0.35) -- (0,-0.35);
        \node[font=\small, fill=white, inner sep=4pt, align=center,
              draw=black!25, rounded corners=4pt]
              at (-2.4,0) {$\mathrm{Waitlist}=\{E,B\}$};
        \node[font=\small, fill=white, inner sep=4pt, align=center,
              draw=black!25, rounded corners=4pt]
              at (2.4,0) {$\mathrm{Donelist}=\{T,A\}$};
    \end{tikzpicture}

    \vspace{0mm}

    \fbox{%
    \begin{minipage}[c]{0.92\linewidth}
        \centering

        \begin{subfigure}[c]{0.45\linewidth}
            \centering
            \fixedseqtikzbox{%
            \begin{tikzpicture}[x=1.0cm,y=1.0cm]
                \draw[fill=red!80, fill opacity=0.35, draw=none]
                    (3.5, 0.0) ellipse [x radius=0.5cm, y radius=0.4cm];

         
                \draw (2.0, 1.3) node(A)  [obs] {$A$};
                \draw (5.0, 1.3) node(H)   [obs] {$H$};

                \draw (2.0, 0.0) node(T)  [obs] {$T$};
                \draw (3.5, 0.0) node(E)   [obs] {$E$};
                \draw (5.0, 0.0) node(G)   [obs] {$G$};

                \draw[circle-arcsq] (E) -- (T);
                \draw[circle-arcsq] (E) -- (G);
                \draw[circle-arcsq] (H) -- (G);
                \draw[circle-arcsq] (A) -- (T);
                
            \end{tikzpicture}%
            }
            \caption{$\g[L]_{E}$}
            \label{fig:exampleR3-inferred-local-pag-E}
        \end{subfigure}
        \hfill
        \begin{subfigure}[c]{0.45\linewidth}
            \centering
            \fixedseqtikzbox{%
            \begin{tikzpicture}[x=1.0cm,y=1.0cm]
                \draw (0.5, 1.3) node(B)  [obs] {$B$};
                \draw (2.0, 1.3) node(A)  [obs] {$A$};
                \draw (5.0, 1.3) node(H)   [obs] {$H$};

                \draw (2.0, 0.0) node(T)  [target] {$T$};
                \draw (3.5, 0.0) node(E)   [obs] {$E$};
                \draw (0.5, 0.0) node(C)  [obs] {$C$};
                \draw (5.0, 0.0) node(G)   [obs] {$G$};

                \draw[circle-circle] (A) -- (B);
                \draw[circle-arcsq] (A) -- (T);
                \draw[circle-arcsq] (E) -- (T);
                \draw[-arcsq] (T) -- (C);
                \draw[-arcsq] (B) -- (C);
                 \draw[circle-arcsq] (E) -- (G);
                \draw[circle-arcsq] (H) -- (G);
            \end{tikzpicture}%
            }
            \caption{$\hat{\PAG}$ after learning $\g[L]_{E}$}
            \label{fig:exampleR3-hat-pag-after-E}
        \end{subfigure}

    \end{minipage}}

    \vspace{0mm}

    \begin{tikzpicture}
        \draw[-stealth, line width=1.5pt] (0,0.35) -- (0,-0.35);
        \node[font=\small, fill=white, inner sep=4pt, align=center,
              draw=black!25, rounded corners=4pt]
              at (-2.4,0) {$\mathrm{Waitlist}=\{B\}$};
        \node[font=\small, fill=white, inner sep=4pt, align=center,
              draw=black!25, rounded corners=4pt]
              at (2.4,0) {$\mathrm{Donelist}=\{T,A,E\}$};
    \end{tikzpicture}

    \vspace{0mm}

    \fbox{%
    \begin{minipage}[c]{0.92\linewidth}
        \centering

        \begin{subfigure}[c]{0.45\linewidth}
            \centering
            \fixedseqtikzbox{%
            \begin{tikzpicture}[x=1.0cm,y=1.0cm]
                \draw[fill=red!80, fill opacity=0.35, draw=none]
                    (0.5, 1.3) ellipse [x radius=0.5cm, y radius=0.4cm];

                \draw (0.5, 1.3) node(B)  [obs] {$B$};
                \draw (2.0, 1.3) node(A)  [obs] {$A$};

                \draw (2.0, 0.0) node(T)  [obs] {$T$};
                \draw (0.5, 0.0) node(C)  [obs] {$C$};

                \draw[circle-circle] (A) -- (B);
                \draw[circle-circle,red] (A) -- (T);
                \draw[circle-arcsq] (T) -- (C);
                \draw[circle-arcsq] (B) -- (C);
                
            \end{tikzpicture}%
            }
            \caption{$\g[L]_{B}$}
            \label{fig:exampleR3-inferred-local-pag-B}
        \end{subfigure}
        \hfill
        \begin{subfigure}[c]{0.45\linewidth}
            \centering
            \fixedseqtikzbox{%
            \begin{tikzpicture}[x=1.0cm,y=1.0cm]
                \draw (0.5, 1.3) node(B)  [obs] {$B$};
                \draw (2.0, 1.3) node(A)  [obs] {$A$};
                \draw (5.0, 1.3) node(H)   [obs] {$H$};

                \draw (2.0, 0.0) node(T)  [target] {$T$};
                \draw (3.5, 0.0) node(E)   [obs] {$E$};
                \draw (0.5, 0.0) node(C)  [obs] {$C$};
                \draw (5.0, 0.0) node(G)   [obs] {$G$};

                \draw[circle-circle] (A) -- (B);
                \draw[circle-arcsq] (A) -- (T);
                \draw[circle-arcsq] (E) -- (T);
                \draw[-arcsqblue] (T) -- (C);
                \draw[-arcsqblue] (B) -- (C);
                 \draw[circle-arcsqblue] (E) -- (G);
                \draw[circle-arcsq] (H) -- (G);
            \end{tikzpicture}%
            }
            \caption{$\hat{\PAG}$ after learning $\g[L]_{B}$}
            \label{fig:exampleR3-hat-pag-after-B}
        \end{subfigure}

    \end{minipage}}

    \vspace{0mm}

    \begin{tikzpicture}
        \draw[-stealth, line width=1.5pt] (0,0.35) -- (0,-0.35);
        \node[font=\small, fill=white, inner sep=4pt, align=center,
              draw=black!25, rounded corners=4pt]
              at (-2.4,0) {$\mathrm{Waitlist}=\{\}$};
        \node[font=\small, fill=white, inner sep=4pt, align=center,
              draw=black!25, rounded corners=4pt]
              at (2.4,0) {$\mathrm{Donelist}=\{T,A,E,B\}$};
        \node[font=\small, fill=OrangeRed!5, inner sep=4pt, align=center,
              draw=OrangeRed, rounded corners=4pt]
              at (0,-0.70) {Terminate: $\mathrm{Waitlist}=\emptyset$};
    \end{tikzpicture}
    }

    \caption{Sequential execution of \method in \cref{example:exampleR3}. Each row shows the learned local structure (left) and the updated $\hat{\PAG}$ (right). Red shading marks the currently processed variable, red edges are discarded because their global consistency is not guaranteed, and the blue arrows in the final update block all remaining potentially anterior paths to $T$. Thus, \textsc{UpdateWaitlist} returns an empty waitlist and the procedure terminates by $\mathcal{R}2$.}
    \label{fig:exampleR3-inferred-local-pag-process}
\end{figure}
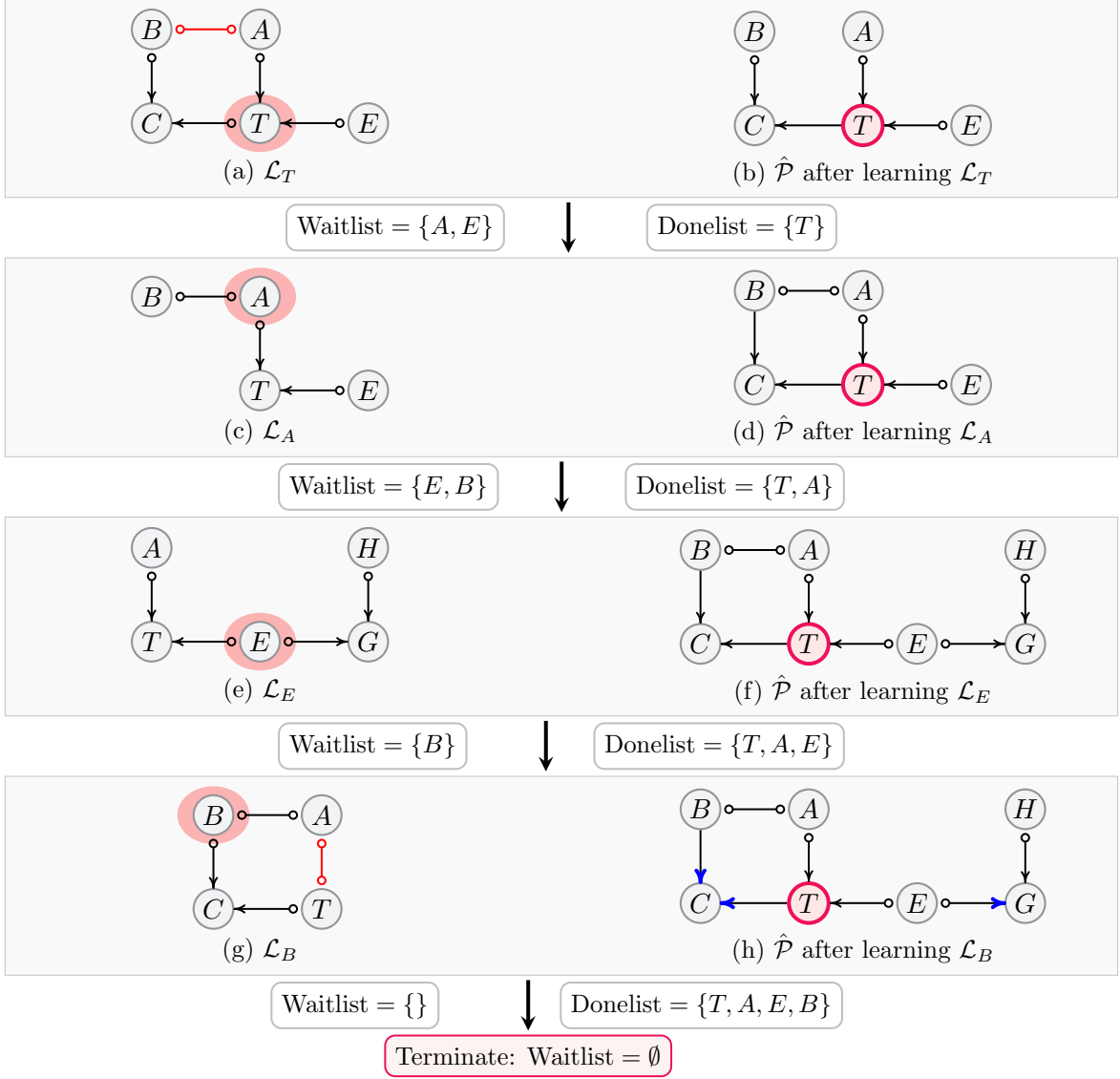

\begin{example}
\label{example:exampleR3}
Consider the causal graphs in \cref{fig:exampleR3-graphs}, and assume oracle conditional independence tests. 
Let $T$ be the target variable. 
The algorithm starts with 
$
\mathrm{Waitlist}=\{T\},
\mathrm{Donelist}=\emptyset,
$
and $\hat{\PAG}=\emptyset$. 
Its execution is illustrated in \cref{fig:exampleR3-inferred-local-pag-process} and proceeds as follows:
\begin{itemize}
    \item \textbf{Processing $T$.}
    The algorithm first identifies
    $
    \MBplus(T)=\{A,B,C,E,T\}
    $
    and learns the local graph $\g[L]_{T}$ shown in \cref{fig:exampleR3-inferred-local-pag-T}. 
    It then incorporates the locally learned information that is guaranteed to be globally consistent and discards the red edge. 
    The resulting graph is shown in \cref{fig:exampleR3-hat-pag-after-T}. 
    Since $A$ and $E$ have potentially anterior paths to $T$, the waitlist update gives
    $
    \mathrm{Donelist}=\{T\},
    \mathrm{Waitlist}=\{A,E\}.
    $

    \item \textbf{Processing $A$.}
    Next, the algorithm identifies
    $
    \MBplus(A)=\{A,B,E,T\}
    $
    and learns $\g[L]_{A}$, shown in \cref{fig:exampleR3-inferred-local-pag-A}. 
    After preserving the globally consistent information and applying the orientation rules, the algorithm obtains the graph in \cref{fig:exampleR3-hat-pag-after-A}. 
    The waitlist update yields
    $
    \mathrm{Donelist}=\{T,A\},
    \mathrm{Waitlist}=\{E,B\}.
    $

    \item \textbf{Processing $E$.}
    The algorithm then identifies
    $
    \MBplus(E)=\{A,E,G,H,T\}.
    $
    The learned local graph $\g[L]_{E}$ and the updated graph are shown in \cref{fig:exampleR3-inferred-local-pag-E,fig:exampleR3-hat-pag-after-E}, respectively. 
    After preserving the globally consistent information and applying the orientation rules, only $B$ remains on the waitlist:
    $
    \mathrm{Donelist}=\{T,A,E\},
    \mathrm{Waitlist}=\{B\}.
    $

    \item \textbf{Processing $B$.}
    Finally, the algorithm identifies
    $
    \MBplus(B)=\{A,B,C,T\}
    $
    and learns $\g[L]_{B}$, shown in \cref{fig:exampleR3-inferred-local-pag-B}. 
    After the update, $\hat{\PAG}$ becomes the graph shown in \cref{fig:exampleR3-hat-pag-after-B}. 
    At this point, every remaining unprocessed variable has no potentially anterior path to $T$. 
    Therefore, \textsc{UpdateWaitlist} returns
    $
    \mathrm{Donelist}=\{T,A,E,B\},
    \mathrm{Waitlist}=\emptyset.
    $
    The algorithm consequently terminates by $\mathcal{R}2$. 
    Under a coarser waitlist update, the same structural redundancy would instead be detected by $\mathcal{R}3$.
\end{itemize}
\end{example}

\FloatBarrier
\section{Supplementary Experiments} 
\label{app:supplementary-experiments}

\subsection{Additional Details on Baseline Methods}
\label{app:supplementary-baseline-methods}

For all compared methods, we use the default hyperparameter settings provided by their official or publicly available implementations, unless otherwise stated. 
For conditional-independence-test-based methods, the same significance level $\alpha=0.01$ is used throughout all experiments to ensure comparability across methods. 
All experiments are conducted on a machine equipped with an Intel i9-14900K CPU and 64 GB RAM.

\cref{tab:baseline-methods} summarizes the baseline methods used in our experiments, including their learning scope, supported assumptions, and code sources.

\begin{table}[hpt!]
    \centering
    \caption{Overview of compared causal structure learning methods.}
    \label{tab:baseline-methods}
    \small
    \renewcommand{\arraystretch}{1.15}
    \setlength{\tabcolsep}{4.2mm}
    \resizebox{\textwidth}{!}{%
    \begin{tabular}{@{}lcccp{0.50\textwidth}@{}}
        \toprule
        \textbf{Method} & \textbf{Local} & \textbf{Latent Variables} & \textbf{Selection Bias} & \textbf{Code Source} \\
        \midrule

        PC~\citep{spirtes1991PC} 
        & \xmark & \xmark & \xmark 
        & \url{https://github.com/kuiy/pyCausalFS} \\

        EEMBI-PC~\citep{dong2025intersecting} 
        & \xmark & \xmark & \xmark 
        & \url{https://github.com/ronedong/EEMBI} \\

        FCI~\citep{spirtes1995fci} 
        & \xmark & \cmark & \cmark 
        & \url{https://github.com/cmu-phil/causal-learn} \\

        RFCI~\citep{colombo2012rfci} 
        & \xmark & \cmark & \cmark 
        & \url{https://pcalg.r-forge.r-project.org/} \\

        FCI$^+$~\citep{claassen2013learning} 
        & \xmark & \cmark & \cmark 
        & \url{https://pcalg.r-forge.r-project.org/} \\

        L-MARVEL~\citep{akbari2021recursive} 
        & \xmark & \cmark & \cmark 
        & \url{https://github.com/ban-epfl/rcd} \\

        ICD~\citep{rohekar2021iterative} 
        & \xmark & \cmark & \cmark 
        & \url{https://github.com/IntelLabs/causality-lab} \\

        \midrule

        PCD-by-PCD~\citep{yin2008partial} 
        & \cmark & \xmark & \xmark 
        & \url{https://github.com/kuiy/pyCausalFS} \\

        MB-by-MB~\citep{wang2014discovering} 
        & \cmark & \xmark & \xmark 
        & \url{https://github.com/kuiy/pyCausalFS} \\

        CMB~\citep{gao2015local} 
        & \cmark & \xmark & \xmark 
        & \url{https://github.com/kuiy/pyCausalFS} \\

        PSL~\citep{ling2022psl} 
        & \cmark & \xmark & \xmark 
        & \url{https://github.com/z-dragonl/Causal-Learner} \\

        GraN-LCS~\citep{liang2023gradient} 
        & \cmark & \xmark & \xmark 
        & \url{https://www.sdu-idea.cn/codes.php?name=GraN-LCS} \\

        LatentLCD~\citep{ling2025local} 
        & \cmark & \cmark & \xmark 
        & \url{https://github.com/z-dragonl/Causal-Learner} \\

        \hdashline

        \rowcolor{gray!10}
        \textbf{\method (Ours)} 
        & \cmark & \cmark & \cmark 
        & \url{https://github.com/zhengli0060/LoCaLS} \\

        \bottomrule
    \end{tabular}
    }
    \vspace{1mm}
    \begin{minipage}{0.95\textwidth}
    \footnotesize
    \textit{Note.} 
    ``Local'' indicates whether the method is designed to learn target-specific local causal structures rather than the entire causal graph. 
    ``Latent Variables'' and ``Selection Bias'' indicate whether the method can handle latent confounders and selection bias, respectively.
    \end{minipage}
\end{table}

\FloatBarrier
\subsection{Additional Details on Evaluation Metrics}
\label{app:supplementary-evaluation-metrics}

Let $\PAG$ and $\hat{\PAG}$ denote the true PAG and the learned PAG, respectively. 
For a target variable $T$, we evaluate the target row and target column of the corresponding PAG matrices, excluding the diagonal entry $(T,T)$. 
Formally, define
\[
\Omega_T
=
\{(T,V):V\in\vars[O]\setminus\{T\}\}
\cup
\{(V,T):V\in\vars[O]\setminus\{T\}\}.
\]
The set $\Omega_T$ contains all matrix entries corresponding to the endpoint marks of edges incident to $T$. 
Each nonzero entry encodes a PAG endpoint mark, while a zero entry indicates the absence of an edge.

\paragraph{Mark-Precision, Mark-Recall, and Mark-F1}
A predicted mark is counted as a true positive only when it is nonzero and exactly matches the corresponding true mark:
\[
\mathrm{TP}_{\mathrm{mark}}
=
\sum_{(i,j)\in\Omega_T}
\mathbb{I}
\left[
\hat{\PAG}(i,j)\neq 0
\ \land\
\hat{\PAG}(i,j)=\PAG(i,j)
\right].
\]
Then Mark-Precision, Mark-Recall, and Mark-F1 are defined as
\[
\mathrm{Mark\mbox{-}Precision}
=
\frac{
\mathrm{TP}_{\mathrm{mark}}
}{
\sum_{(i,j)\in\Omega_T}
\mathbb{I}\!\left[\hat{\PAG}(i,j)\neq 0\right]
},
\]
\[
\mathrm{Mark\mbox{-}Recall}
=
\frac{
\mathrm{TP}_{\mathrm{mark}}
}{
\sum_{(i,j)\in\Omega_T}
\mathbb{I}\!\left[\PAG(i,j)\neq 0\right]
},
\]
and
\[
\mathrm{Mark\mbox{-}F1}
=
\frac{
2\cdot
\mathrm{Mark\mbox{-}Precision}
\cdot
\mathrm{Mark\mbox{-}Recall}
}{
\mathrm{Mark\mbox{-}Precision}
+
\mathrm{Mark\mbox{-}Recall}
}.
\]
If the denominator of a metric is zero, the corresponding score is set to zero.

\paragraph{Local-SHD}
Structural Hamming Distance (SHD) is a widely used metric in causal structure learning for measuring the discrepancy between a learned graph and the ground-truth graph. In this work, we also report a target-specific version of Structural Hamming Distance, termed Local-SHD, defined as
\[
\mathrm{Local\mbox{-}SHD}
=
\sum_{(i,j)\in\Omega_T}
\mathbb{I}
\left[
\hat{\PAG}(i,j)
\neq
\PAG(i,j)
\right].
\]
Local-SHD counts the number of mismatched PAG entries associated with the target variable, including missing adjacencies, false adjacencies, and incorrect endpoint marks.

\subsection{Supplementary Results on Random Structures with Varying Dimensions}
\label{app:supplementary-experiments-varying-dimensions}

This section provides supplementary results for the experiment in \cref{sec:performance-varying-dimensions}, including summary statistics of the underlying induced MAGs and detailed efficiency results.

Because marginalizing latent variables and conditioning on selection variables may change the sparsity of the observed-variable structure, \cref{tab:vary-dimension-degree} reports the average degree of the underlying induced MAGs used in this experiment. The induced MAGs remain sparse across all settings, with average degrees ranging from approximately $2.05$ to $2.16$. Therefore, this experiment primarily evaluates scalability with respect to the number of variables rather than changes in graph density.

\cref{tab:vary-dimension-runtime,tab:vary-dimension-ci-test} provide the detailed runtime and CI-test results corresponding to \cref{fig:ER-varying-vertices-results}.

\begin{table}[H]
\centering
\caption{Summary statistics for random structures with varying dimensionality.} 
\label{tab:vary-dimension-degree}
\begin{tabular}{lC{0.5cm}C{0.5cm}C{0.5cm}C{0.5cm}C{5.0cm}}
\toprule
\textbf{Underlying Structure} & $|\vars[V]|$ & $|\vars[O]|$ & $|\vars[L]|$ & $|\vars[S]|$ & 
\textbf{Avg. Degree in Induced MAGs}  \\
\midrule
ER$(20,2)$ graphs  & 20  & 18  & 1  & 1  & 2.05$\pm$0.25  \\
ER$(40,2)$ graphs  & 40  & 36  & 2  & 2  & 2.11$\pm$0.21  \\
ER$(60,2)$ graphs  & 60  & 54  & 3  & 3  & 2.13$\pm$0.16  \\
ER$(80,2)$ graphs  & 80  & 72  & 4  & 4  & 2.11$\pm$0.17  \\
ER$(120,2)$ graphs & 120 & 108 & 6  & 6  & 2.16$\pm$0.19  \\
ER$(160,2)$ graphs & 160 & 144 & 8  & 8  & 2.16$\pm$0.17  \\
ER$(200,2)$ graphs & 200 & 180 & 10 & 10 & 2.13$\pm$0.13  \\
\bottomrule
\end{tabular}
\end{table}

\begin{table}[H]
\centering
\caption{Runtime on random structures with varying dimensionality.} 
\label{tab:vary-dimension-runtime}
\small
\setlength{\tabcolsep}{2.5pt}
\renewcommand{\arraystretch}{1.08}

\begin{adjustbox}{max width=\textwidth}
\begin{tabular}{lccccccc}
\toprule
\multirow{2}{*}{\textbf{Algorithm}} 
& \multicolumn{7}{c}{\textbf{Number of variables}} \\
\cmidrule(lr){2-8}
& \textbf{20} & \textbf{40} & \textbf{60} & \textbf{80} & \textbf{120} & \textbf{160} & \textbf{200} \\
\midrule
PC & 0.09 $\pm$ 0.02 & 0.37 $\pm$ 0.08 & 0.85 $\pm$ 0.17 & 1.48 $\pm$ 0.27 & 3.83 $\pm$ 0.61 & 17.57 $\pm$ 3.69 & 36.02 $\pm$ 7.69 \\
EEMBI-PC & 1120.61 $\pm$ 237.41 & 6008.58 $\pm$ 1716.68 & 13733.03 $\pm$ 595.24 & -- & -- & -- & --\\
FCI & 0.07 $\pm$ 0.02 & 0.29 $\pm$ 0.09 & 0.56 $\pm$ 0.13 & 0.92 $\pm$ 0.26 & 1.79 $\pm$ 0.38 & 2.86 $\pm$ 0.41 & 4.17 $\pm$ 0.45 \\
RFCI & 0.90 $\pm$ 0.40 & 0.93 $\pm$ 0.40 & 0.98 $\pm$ 0.41 & 1.04 $\pm$ 0.42 & 1.21 $\pm$ 0.41 & 1.39 $\pm$ 0.43 & 1.60 $\pm$ 0.43 \\
FCI$^{+}$ & 1.03 $\pm$ 0.42 & 1.17 $\pm$ 0.41 & 1.46 $\pm$ 0.47 & 1.69 $\pm$ 0.44 & 2.41 $\pm$ 0.56 & 3.40 $\pm$ 0.66 & 4.82 $\pm$ 0.60 \\
L-MARVEL & 0.03 $\pm$ 0.02 & 0.08 $\pm$ 0.04 & 0.15 $\pm$ 0.12 & 0.16 $\pm$ 0.07 & 0.29 $\pm$ 0.08 & 0.55 $\pm$ 0.19 & 0.72 $\pm$ 0.15 \\
ICD & 0.10 $\pm$ 0.07 & 4.12 $\pm$ 7.27 & 67.13 $\pm$ 131.27 & 108.39 $\pm$ 184.79 & 235.84 $\pm$ 296.21 & 314.58 $\pm$ 355.24 & 454.36 $\pm$ 393.13 \\
\midrule
PCD-by-PCD & 0.14 $\pm$ 0.06 & 0.68 $\pm$ 0.27 & 1.56 $\pm$ 0.73 & 3.11 $\pm$ 0.89 & 8.73 $\pm$ 2.54 & 30.43 $\pm$ 9.73 & 65.60 $\pm$ 15.14 \\
MB-by-MB & 0.23 $\pm$ 0.11 & 0.90 $\pm$ 0.61 & 1.51 $\pm$ 1.30 & 2.26 $\pm$ 1.33 & 4.16 $\pm$ 2.86 & 11.32 $\pm$ 7.44 & 16.58 $\pm$ 9.87 \\
CMB & 0.22 $\pm$ 0.21 & 0.70 $\pm$ 0.64 & 1.09 $\pm$ 1.06 & 1.34 $\pm$ 1.57 & 2.92 $\pm$ 3.07 & 15.08 $\pm$ 12.83 & 27.99 $\pm$ 32.40 \\
PSL & 0.10 $\pm$ 0.06 & 0.38 $\pm$ 0.19 & 0.67 $\pm$ 0.45 & 0.96 $\pm$ 0.58 & 1.71 $\pm$ 1.28 & 4.74 $\pm$ 2.96 & 9.37 $\pm$ 5.38 \\
GraN-LCS & 98.04 $\pm$ 40.84 & 103.93 $\pm$ 53.11 & 136.32 $\pm$ 83.38 & 126.59 $\pm$ 69.37 & 136.55 $\pm$ 89.04 & 139.10 $\pm$ 77.14 & 161.75 $\pm$ 128.57 \\
LatentLCD & 0.10 $\pm$ 0.04 & 0.33 $\pm$ 0.17 & 0.55 $\pm$ 0.36 & 1.08 $\pm$ 0.67 & 2.07 $\pm$ 1.45 & 3.38 $\pm$ 2.73 & 5.51 $\pm$ 3.89 \\
\hdashline
\rowcolor{gray!10}
LoCaLS (Ours) & \textbf{0.02 $\pm$ 0.01} & \textbf{0.04 $\pm$ 0.02} & \textbf{0.04 $\pm$ 0.03} & \textbf{0.05 $\pm$ 0.03} & \textbf{0.08 $\pm$ 0.05} & \textbf{0.10 $\pm$ 0.05} & \textbf{0.11 $\pm$ 0.04} \\
\bottomrule
\end{tabular}
\end{adjustbox}
\vspace{1mm}
\begin{tablenotes}
\footnotesize
\item Note: ``--'' indicates that the method was excluded due to excessive runtime.
The best result in each group is highlighted in bold.
\end{tablenotes}
\end{table}

\begin{table}[H]
\centering
\caption{Number of CI tests on random structures with varying dimensionality.} 
\label{tab:vary-dimension-ci-test}
\small
\setlength{\tabcolsep}{2.5pt}
\renewcommand{\arraystretch}{1.08}

\begin{adjustbox}{max width=\textwidth}
\begin{tabular}{lccccccc}
\toprule
\multirow{2}{*}{\textbf{Algorithm}} 
& \multicolumn{7}{c}{\textbf{Number of variables}} \\
\cmidrule(lr){2-8}
& \textbf{20} & \textbf{40} & \textbf{60} & \textbf{80} & \textbf{120} & \textbf{160} & \textbf{200} \\
\midrule
PC & 269.9 $\pm$ 44.7 & 1067.0 $\pm$ 155.4 & 2325.4 $\pm$ 311.5 & 3976.2 $\pm$ 463.9 & 8101.7 $\pm$ 618.9 & 13569.1 $\pm$ 696.6 & 20702.8 $\pm$ 1189.3 \\
FCI & 520.8 $\pm$ 137.7 & 1986.5 $\pm$ 511.3 & 3830.0 $\pm$ 848.2 & 6225.8 $\pm$ 1670.8 & 11495.5 $\pm$ 2255.4 & 17679.7 $\pm$ 2022.3 & 25342.9 $\pm$ 2301.2 \\
RFCI & 463.1 $\pm$ 99.9 & 1705.8 $\pm$ 277.8 & 3353.2 $\pm$ 570.5 & 5364.0 $\pm$ 670.9 & 10243.6 $\pm$ 1331.7 & 16113.0 $\pm$ 1062.4 & 23557.7 $\pm$ 1583.9 \\
FCI$^{+}$ & 578.1 $\pm$ 157.6 & 2262.3 $\pm$ 444.6 & 4550.9 $\pm$ 1071.6 & 7245.1 $\pm$ 1373.1 & 13256.9 $\pm$ 2475.7 & 20215.7 $\pm$ 2254.6 & 29245.8 $\pm$ 3874.8 \\
L-MARVEL & 244.7 $\pm$ 76.1 & 938.0 $\pm$ 313.2 & 1966.7 $\pm$ 794.7 & 3065.1 $\pm$ 291.2 & 6464.2 $\pm$ 330.4 & 11270.2 $\pm$ 571.7 & 17116.7 $\pm$ 286.6 \\
ICD & 346.9 $\pm$ 56.4 & 1370.1 $\pm$ 169.1 & 2769.0 $\pm$ 403.9 & 4652.0 $\pm$ 564.3 & 8983.4 $\pm$ 680.1 & 15118.3 $\pm$ 1233.4 & 22359.0 $\pm$ 2143.0 \\
\midrule
PCD-by-PCD & 391.7 $\pm$ 153.3 & 1620.0 $\pm$ 602.7 & 3592.9 $\pm$ 1676.2 & 6665.7 $\pm$ 1889.8 & 13479.3 $\pm$ 3529.9 & 21081.2 $\pm$ 5269.9 & 34744.4 $\pm$ 4194.0 \\
MB-by-MB & 667.7 $\pm$ 340.4 & 2344.5 $\pm$ 1603.0 & 3887.3 $\pm$ 4633.0 & 4848.7 $\pm$ 3131.5 & 6748.4 $\pm$ 4928.5 & 7744.9 $\pm$ 6971.3 & 9511.9 $\pm$ 6029.1 \\
CMB & 682.4 $\pm$ 776.5 & 1960.2 $\pm$ 2025.8 & 3214.6 $\pm$ 4013.1 & 3760.0 $\pm$ 5119.9 & 6321.6 $\pm$ 7113.8 & 9874.0 $\pm$ 9466.8 & 19788.4 $\pm$ 30942.8 \\
PSL & 307.9 $\pm$ 173.6 & 1017.3 $\pm$ 502.4 & 1695.0 $\pm$ 1226.6 & 2311.6 $\pm$ 1678.2 & 3684.9 $\pm$ 3465.7 & 3514.7 $\pm$ 2459.0 & 5651.3 $\pm$ 4162.4 \\
LatentLCD & 983.2 $\pm$ 442.7 & 3287.8 $\pm$ 1827.2 & 5398.3 $\pm$ 3844.3 & 9460.4 $\pm$ 5985.3 & 15194.2 $\pm$ 11109.6 & 20657.9 $\pm$ 17861.9 & 36316.1 $\pm$ 26655.3 \\
\hdashline
\rowcolor{gray!10}
LoCaLS (Ours) & \textbf{159.0 $\pm$ 87.1} & \textbf{357.1 $\pm$ 190.8} & \textbf{433.4 $\pm$ 243.6} & \textbf{558.8 $\pm$ 268.9} & \textbf{706.0 $\pm$ 431.7} & \textbf{757.4 $\pm$ 478.2} & \textbf{978.1 $\pm$ 449.6} \\
\bottomrule
\end{tabular}
\end{adjustbox}
\vspace{1mm}
\begin{tablenotes}
\footnotesize
\item Note:
The best result in each group is highlighted in bold.
\end{tablenotes}
\end{table}

\FloatBarrier
\subsection{Supplementary Results on Random Structures with Varying Density} 
\label{app:supplementary-experiments-varying-density}

This section provides supplementary numerical results for the experiment in \cref{sec:performance-varying-density}. Specifically, \cref{tab:vary-density-runtime,tab:vary-density-ci-test} report the detailed runtime and CI-test results corresponding to \cref{fig:ER-varying-density-results}. EEMBI-PC is omitted due to excessive runtime.

\begin{table}[H]
\centering
\caption{Runtime on random structures with varying graph density.} \label{tab:vary-density-runtime}
\small
\setlength{\tabcolsep}{2.5pt}
\renewcommand{\arraystretch}{1.08}

\begin{adjustbox}{max width=\textwidth}
\begin{tabular}{lccccccc}
\toprule
\multirow{2}{*}{\textbf{Algorithm}} 
& \multicolumn{7}{c}{\textbf{Density}} \\
\cmidrule(lr){2-8}
& \textbf{0.05} & \textbf{0.1} & \textbf{0.2} & \textbf{0.4} & \textbf{0.6} & \textbf{0.8} & \textbf{1.0} \\
\midrule
PC & 1.96 $\pm$ 0.34 & 2.35 $\pm$ 0.45 & 3.39 $\pm$ 0.75 & 4.90 $\pm$ 0.95 & 5.93 $\pm$ 1.20 & 6.48 $\pm$ 1.17 & 6.69 $\pm$ 1.17 \\
FCI & 1.66 $\pm$ 0.36 & 2.16 $\pm$ 0.50 & 2.96 $\pm$ 0.75 & 4.67 $\pm$ 1.49 & 5.33 $\pm$ 0.97 & 6.67 $\pm$ 1.17 & 6.82 $\pm$ 1.44 \\
RFCI & 1.50 $\pm$ 0.62 & 1.79 $\pm$ 0.64 & 1.69 $\pm$ 0.60 & 2.01 $\pm$ 0.65 & 2.05 $\pm$ 0.63 & 2.74 $\pm$ 0.80 & 2.61 $\pm$ 0.72 \\
FCI$^{+}$ & 2.73 $\pm$ 0.69 & 3.35 $\pm$ 0.89 & 3.32 $\pm$ 0.87 & 4.17 $\pm$ 1.03 & 4.46 $\pm$ 1.09 & 6.05 $\pm$ 1.26 & 5.63 $\pm$ 1.16 \\
L-MARVEL & 0.28 $\pm$ 0.09 & 0.53 $\pm$ 0.84 & 1.53 $\pm$ 1.69 & 17.42 $\pm$ 18.09 & 32.09 $\pm$ 28.93 & 88.85 $\pm$ 61.77 & 150.31 $\pm$ 113.99 \\
ICD & 37.10 $\pm$ 114.49 & 1704.42 $\pm$ 3036.36 & 3129.47 $\pm$ 1864.21 & 1555.47 $\pm$ 1213.07 & 2508.14 $\pm$ 1044.65 & 2693.20 $\pm$ 1204.20 & 3143.88 $\pm$ 1775.54 \\

\midrule
PCD-by-PCD & 2.30 $\pm$ 1.12 & 5.54 $\pm$ 2.95 & 11.05 $\pm$ 3.41 & 18.51 $\pm$ 4.06 & 22.76 $\pm$ 5.60 & 27.88 $\pm$ 7.26 & 27.86 $\pm$ 5.45 \\
MB-by-MB & 1.37 $\pm$ 1.41 & 3.00 $\pm$ 3.31 & 7.71 $\pm$ 5.81 & 12.87 $\pm$ 7.67 & 16.60 $\pm$ 10.49 & 20.62 $\pm$ 12.40 & 23.16 $\pm$ 16.24 \\
CMB & 2.58 $\pm$ 1.79 & 4.45 $\pm$ 3.76 & 8.77 $\pm$ 5.50 & 23.77 $\pm$ 19.21 & 32.09 $\pm$ 24.40 & 23.73 $\pm$ 14.02 & 55.92 $\pm$ 36.06 \\
PSL & 1.02 $\pm$ 0.59 & 1.85 $\pm$ 1.50 & 3.68 $\pm$ 1.87 & 5.98 $\pm$ 3.37 & 7.52 $\pm$ 3.23 & 8.44 $\pm$ 4.30 & 9.82 $\pm$ 4.46 \\
GraN-LCS & 155.33 $\pm$ 102.98 & 123.29 $\pm$ 77.04 & 125.30 $\pm$ 56.99 & 164.29 $\pm$ 101.18 & 182.34 $\pm$ 123.93 & 179.46 $\pm$ 126.89 & 204.91 $\pm$ 160.07 \\
LatentLCD & 1.37 $\pm$ 0.65 & 2.75 $\pm$ 1.49 & 4.41 $\pm$ 3.52 & 7.93 $\pm$ 7.14 & 5.73 $\pm$ 5.51 & 8.26 $\pm$ 8.99 & 9.08 $\pm$ 9.86 \\
\hdashline
\rowcolor{gray!10}
LoCaLS (Ours) & \textbf{0.12 $\pm$ 0.05} & \textbf{0.15 $\pm$ 0.07} & \textbf{0.18 $\pm$ 0.07} & \textbf{0.20 $\pm$ 0.08} & \textbf{0.28 $\pm$ 0.19} & \textbf{0.46 $\pm$ 0.29} & \textbf{0.47 $\pm$ 0.40} \\
\bottomrule
\end{tabular}
\end{adjustbox}
\vspace{1mm}
\begin{tablenotes}
\footnotesize
\item Note: The best result in each group is highlighted in bold.
\end{tablenotes}
\end{table}

\begin{table}[H]
\centering
\caption{Number of CI tests on random structures with varying graph density.} 
\label{tab:vary-density-ci-test}
\small
\setlength{\tabcolsep}{2.5pt}
\renewcommand{\arraystretch}{1.08}

\begin{adjustbox}{max width=\textwidth}
\begin{tabular}{lccccccc}
\toprule
\multirow{2}{*}{\textbf{Algorithm}} 
& \multicolumn{7}{c}{\textbf{Density}} \\
\cmidrule(lr){2-8}
& \textbf{0.05} & \textbf{0.1} & \textbf{0.2} & \textbf{0.4} & \textbf{0.6} & \textbf{0.8} & \textbf{1.0} \\
\midrule
PC & 4955.3 $\pm$ 367.0 & 6397.0 $\pm$ 942.2 & 9299.2 $\pm$ 1722.7 & 13147.1 $\pm$ 1992.1 & 15475.9 $\pm$ 2489.4 & 17081.3 $\pm$ 2420.6 & 17638.3 $\pm$ 2539.4 \\
FCI & 6436.2 $\pm$ 931.6 & 10701.7 $\pm$ 2628.1 & 17559.6 $\pm$ 4267.4 & 29473.9 $\pm$ 7846.5 & 35188.3 $\pm$ 6540.5 & 40727.4 $\pm$ 7262.5 & 42031.6 $\pm$ 8354.8 \\
RFCI & 5984.7 $\pm$ 698.7 & 9235.1 $\pm$ 2037.7 & 14705.3 $\pm$ 3187.4 & 23463.6 $\pm$ 4114.1 & 28652.4 $\pm$ 5000.5 & 32858.9 $\pm$ 5628.6 & 34058.3 $\pm$ 6224.9 \\
FCI$^{+}$ & 6971.4 $\pm$ 1171.5 & 12733.1 $\pm$ 4164.3 & 24253.2 $\pm$ 7199.3 & 42573.2 $\pm$ 9145.9 & 54814.1 $\pm$ 12403.1 & 63872.5 $\pm$ 12577.2 & 67575.8 $\pm$ 15107.2 \\
L-MARVEL & 4381.1 $\pm$ 243.4 & 6448.2 $\pm$ 7022.9 & 12607.7 $\pm$ 13010.5 & 107921.2 $\pm$ 125798.5 & 191199.7 $\pm$ 209637.3 & 415518.9 $\pm$ 314416.8 & 678769.8 $\pm$ 630760.4 \\
ICD & 5314.4 $\pm$ 509.3 & 7468.7 $\pm$ 1223.3 & 12180.0 $\pm$ 2467.3 & 18060.4 $\pm$ 2891.2 & 24911.2 $\pm$ 3926.0 & 27085.8 $\pm$ 3549.3 & 27360.4 $\pm$ 3283.2 \\
\midrule
PCD-by-PCD & 5353.3 $\pm$ 2766.7 & 12113.7 $\pm$ 6587.6 & 25001.8 $\pm$ 7602.9 & 41267.8 $\pm$ 8217.4 & 50762.8 $\pm$ 10952.1 & 59922.7 $\pm$ 13524.6 & 60804.2 $\pm$ 11682.4 \\
MB-by-MB & 3795.9 $\pm$ 5408.7 & 7957.1 $\pm$ 11382.0 & 19009.5 $\pm$ 19035.0 & 27815.6 $\pm$ 16671.5 & 34630.6 $\pm$ 21526.2 & 41856.1 $\pm$ 26189.3 & 48282.1 $\pm$ 35115.3 \\
CMB & 6812.5 $\pm$ 5954.4 & 11527.7 $\pm$ 11732.8 & 22616.2 $\pm$ 15742.7 & 39088.9 $\pm$ 34534.5 & 46467.1 $\pm$ 32405.8 & 53827.5 $\pm$ 35395.8 & 59007.0 $\pm$ 42992.4 \\
PSL & 2638.0 $\pm$ 1819.9 & 4794.3 $\pm$ 4771.1 & 8829.4 $\pm$ 4765.7 & 13581.0 $\pm$ 8804.0 & 16611.1 $\pm$ 7713.0 & 18843.0 $\pm$ 10145.5 & 22048.9 $\pm$ 11609.0 \\
LatentLCD & 9010.3 $\pm$ 4878.1 & 24738.3 $\pm$ 15588.3 & 44342.4 $\pm$ 39417.8 & 85422.0 $\pm$ 82438.2 & 65808.3 $\pm$ 78557.0 & 100936.1 $\pm$ 129444.6 & 108625.3 $\pm$ 144386.5 \\
\hdashline
\rowcolor{gray!10}
LoCaLS (Ours) & \textbf{407.0 $\pm$ 304.2} & \textbf{459.1 $\pm$ 436.2} & \textbf{563.9 $\pm$ 390.3} & \textbf{632.8 $\pm$ 435.9} & \textbf{823.7 $\pm$ 744.1} & \textbf{1158.7 $\pm$ 857.8} & \textbf{1109.8 $\pm$ 857.0} \\
\bottomrule
\end{tabular}
\end{adjustbox}
\vspace{1mm}
\begin{tablenotes}
\footnotesize
\item Note: The best result in each group is highlighted in bold.
\end{tablenotes}
\end{table}

\FloatBarrier
\subsection{Supplementary Results on Real-World Structures} \label{app:supplementary-experiments-real-world-structures}

This section provides supplementary results for the real-world benchmark experiments in \cref{sec:performance-real-world-structures}. Specifically, \cref{tab:real-structure-degree} summarizes the causal structures used in this experiment. \cref{fig:small-benchmark-results} reports the runtime, number of CI tests, and Local-SHD results on the smaller structures MILDEW and BARLEY. \cref{fig:benchmark-prf-results} reports the Mark-Precision, Mark-Recall, and Mark-F1 curves for all four benchmark structures.

\begin{table}[H] 
\centering 
\caption{Summary statistics for real-world benchmark structures.} 
\label{tab:real-structure-degree} 
\begin{tabular}{lC{0.25cm}C{0.25cm}C{0.25cm}C{0.25cm}C{3.0cm}C{4.0cm}} 
\toprule 
\textbf{Real-World Structure} 
& $|\vars[V]|$ 
& $|\vars[O]|$ 
& $|\vars[L]|$ 
& $|\vars[S]|$ 
& \textbf{Avg. Degree in DAGs} 
& \textbf{Avg. Degree in Induced MAGs} \\ 
\midrule 
MILDEW & 35  & 31  & 2  & 2  & 2.63 & 3.37$\pm$0.55 \\ 
BARLEY & 48  & 42  & 3  & 3  & 3.5 & 3.86$\pm$0.30 \\ 
ANDES  & 223 & 213 & 5  & 5  & 3.03 & 3.31$\pm$0.20 \\ 
LINK   & 724 & 694 & 15 & 15 & 3.11 & 3.19$\pm$0.04 \\ 
\bottomrule 
\end{tabular} 
\end{table}

\begin{figure}[H]
    \centering
    \begin{subfigure}{1.0\linewidth}
        \centering
        \includegraphics[width=\linewidth]{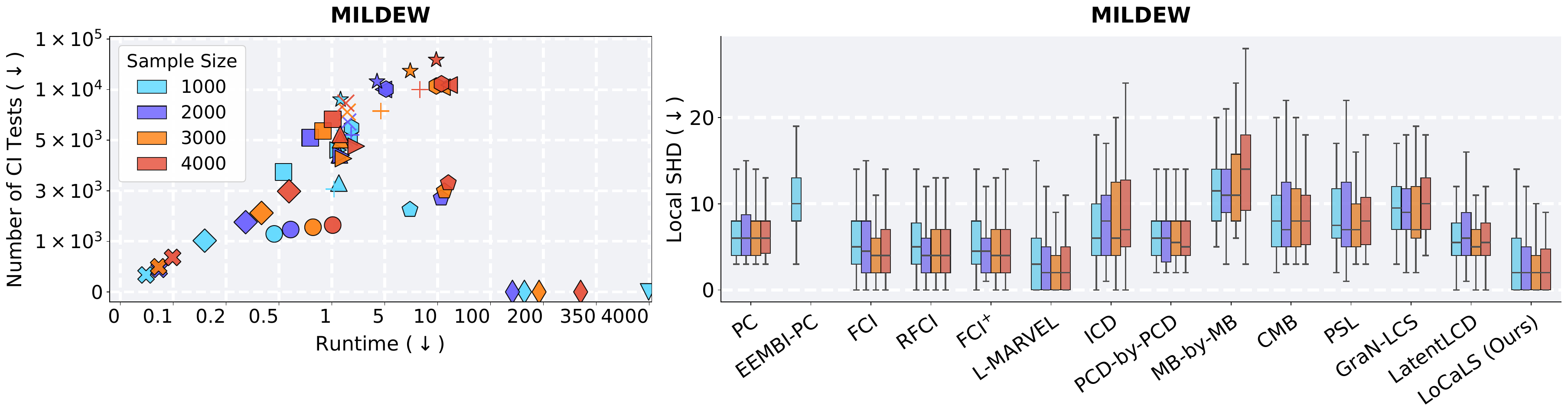}
        \label{fig:mildew-runtime-ci}
        
    \end{subfigure}
        \begin{subfigure}{1.0\linewidth}
        \centering
        \includegraphics[width=\linewidth]{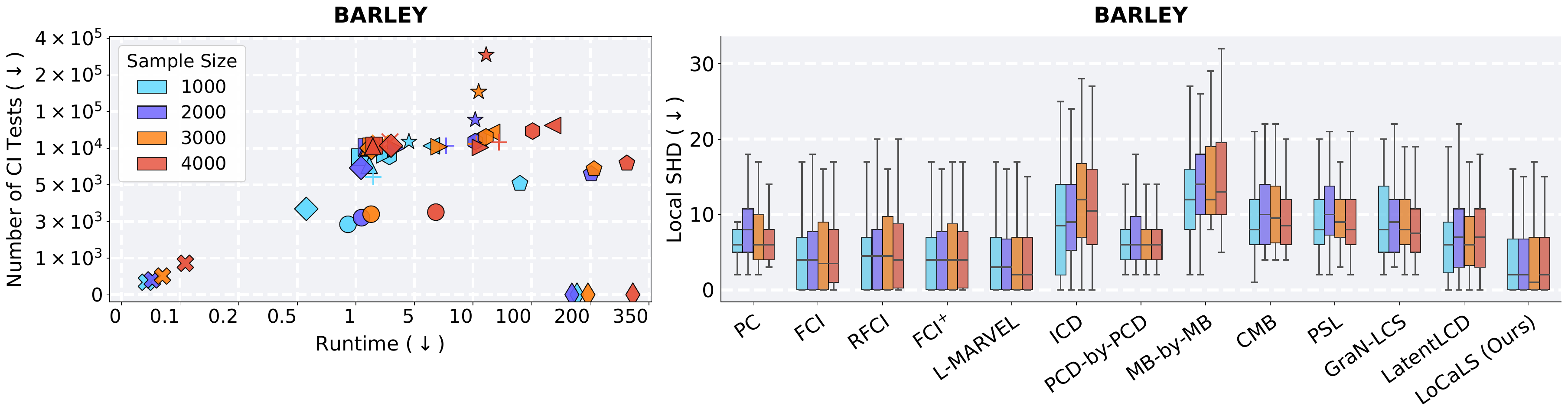}
        \label{fig:barley-runtime-ci}
    \end{subfigure}
    
        \begin{subfigure}{0.9\linewidth}
        \centering
        \includegraphics[width=\linewidth]{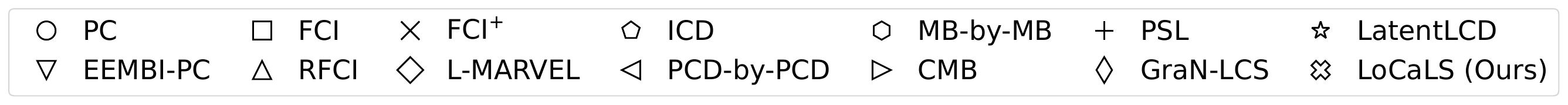}
    \end{subfigure}
    \caption{Results on the smaller real-world benchmark structures MILDEW and BARLEY under varying sample sizes. The panels report runtime, CI-test counts, and Local-SHD. The symbol $\downarrow$ indicates that lower values are better.} 
\label{fig:small-benchmark-results}
\end{figure}

\begin{figure}[H]
    \centering
    \begin{subfigure}{1.0\linewidth}
        \centering
        \includegraphics[width=\linewidth]{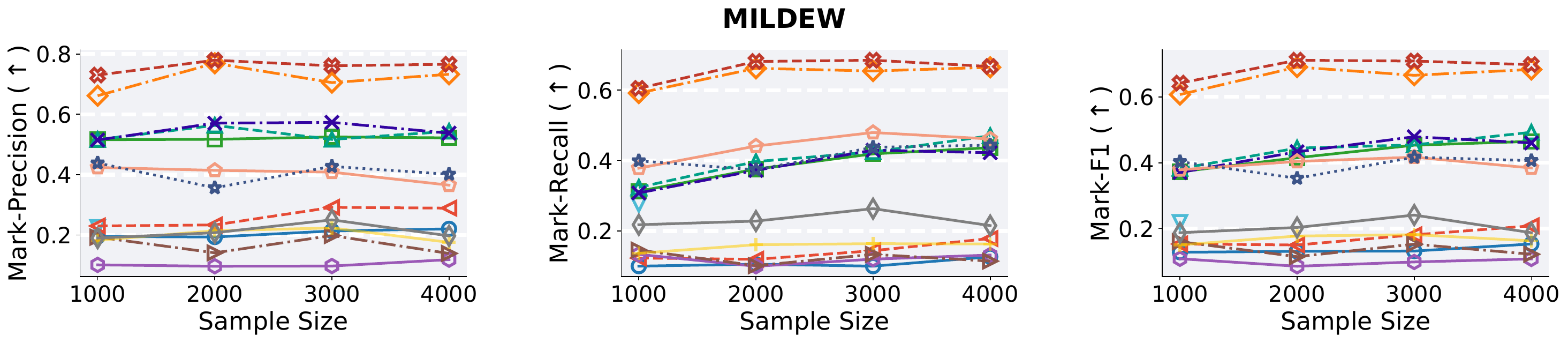}
        \label{fig:mildew-prf}
        
    \end{subfigure}
        \begin{subfigure}{1.0\linewidth}
        \centering
        \includegraphics[width=\linewidth]{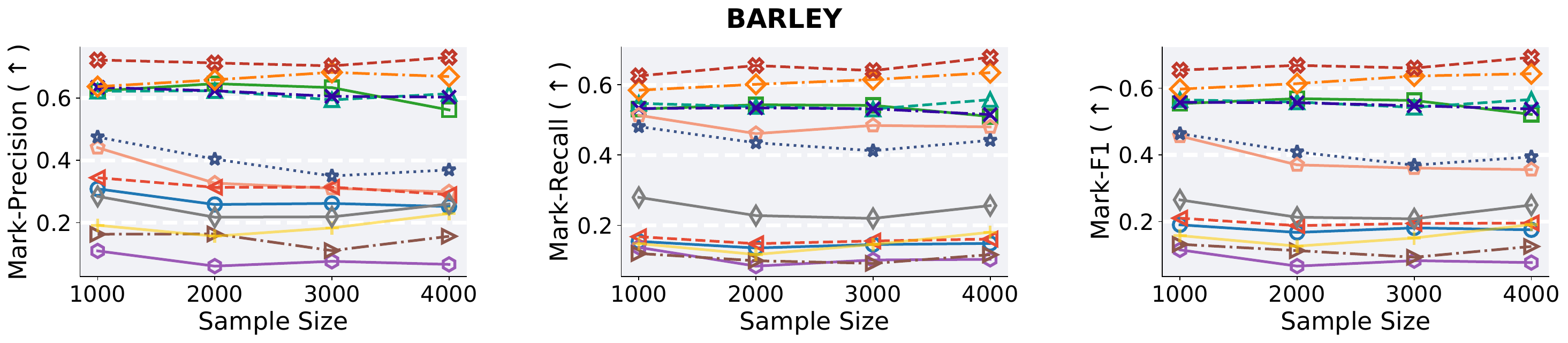}
        \label{fig:barley-prf}
    \end{subfigure}
    \begin{subfigure}{1.0\linewidth}
        \centering
        \includegraphics[width=\linewidth]{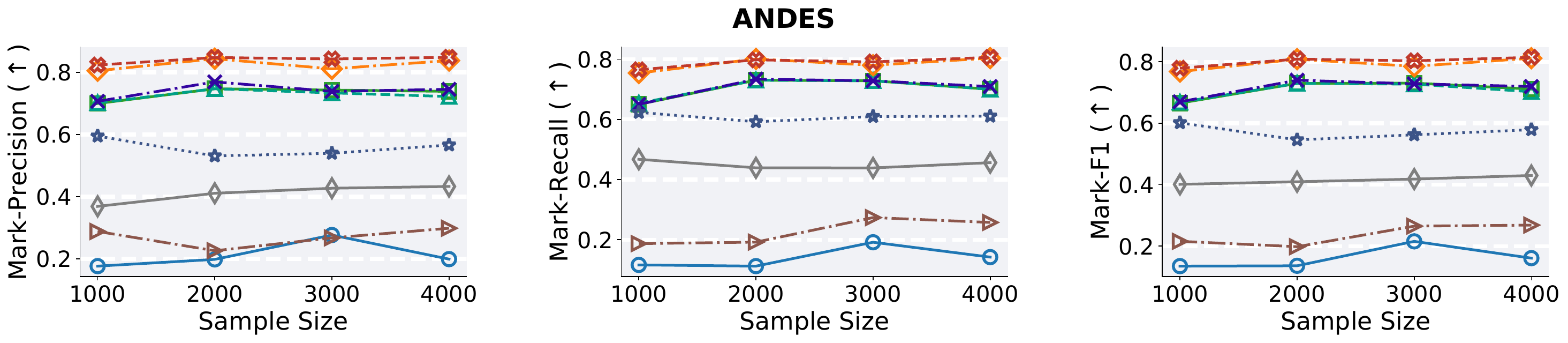}
        \label{fig:andes-prf}
    \end{subfigure}

    \begin{subfigure}{1.0\linewidth}
        \centering
        \includegraphics[width=\linewidth]{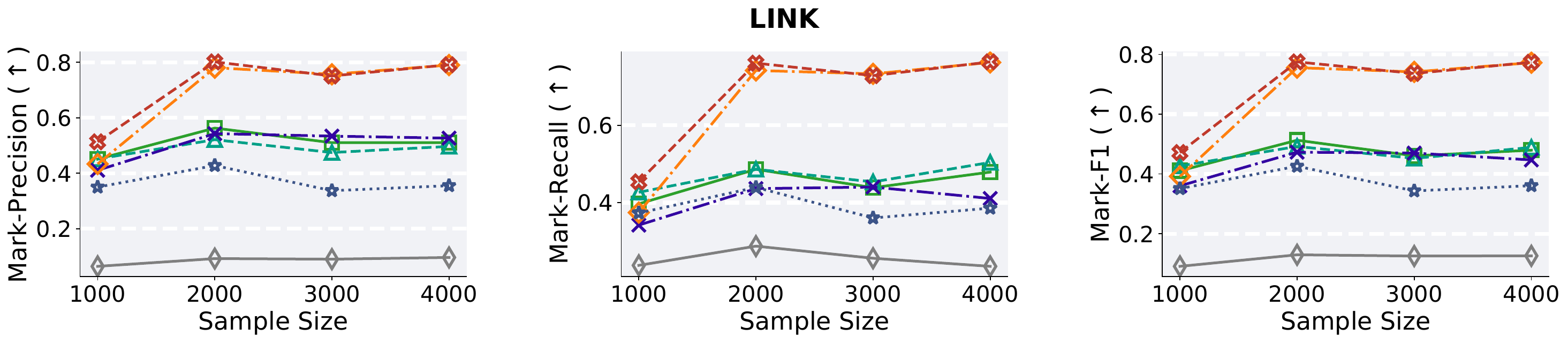}
        \label{fig:link-prf}
    \end{subfigure}
        \begin{subfigure}{0.9\linewidth}
        \centering
        \includegraphics[width=\linewidth]{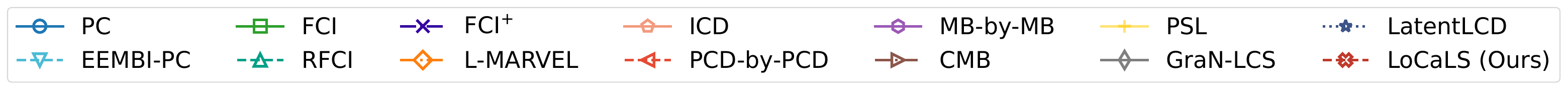}

    \end{subfigure}
    \caption{Mark-based accuracy on real-world benchmark structures under varying sample sizes. The panels report Mark-Precision, Mark-Recall, and Mark-F1 for MILDEW, BARLEY, ANDES, and LINK. The symbol $\uparrow$ indicates that higher values are better.} 
    \label{fig:benchmark-prf-results}
\end{figure}

\FloatBarrier
\subsection{Supplementary Results on Latent-Variable Sensitivity} \label{app:supplementary-sensitivity-latent} 

This section provides supplementary results for the latent-variable sensitivity experiment in \cref{sec:sensitivity-latent}. As shown in \cref{tab:vary-latent-ratio}, increasing the latent ratio only slightly increases the average degree of the induced MAGs, indicating that the main varying factor in this experiment is the amount of latent confounding rather than the overall graph density. \cref{tab:vary-latent-ratio-efficiency} reports the detailed runtime and CI-test results, and \cref{fig:ER-varying-latent-results} reports the Mark-Precision, Mark-Recall, Mark-F1, and Local-SHD results.

\begin{table}[H]
\centering
\caption{Summary statistics under varying latent ratios. The selection ratio is fixed at $0.05$.} 
\label{tab:vary-latent-ratio}
\begin{tabular}{lC{0.3cm}C{0.3cm}C{0.3cm}C{7.0cm}}
\toprule
\textbf{Latent Ratio} & $|\vars[O]|$ & $|\vars[L]|$ & $|\vars[S]|$ & 
\textbf{Avg. Degree in Induced MAGs}  \\
\midrule
0.05  & 100  & 5  & 5  & 2.38$\pm$0.11  \\
0.10  & 100  & 10  & 5  & 2.44$\pm$0.06  \\
0.15  & 100  & 15  & 5  & 2.45$\pm$0.04  \\
0.20  & 100  & 20  & 5  & 2.46$\pm$0.04  \\
\bottomrule
\end{tabular}
\end{table}

\begin{table}[H]
\centering
\caption{Runtime and number of CI tests under varying latent ratios.}
\label{tab:vary-latent-ratio-efficiency}
\scriptsize
\setlength{\tabcolsep}{3pt}
\renewcommand{\arraystretch}{1.08}
\resizebox{\textwidth}{!}{
\begin{tabular}{lcccccccc}
\toprule
\multirow{2}{*}{\textbf{Algorithm}}
& \multicolumn{4}{c}{\textbf{Runtime (seconds) $\downarrow$}}
& \multicolumn{4}{c}{\textbf{Number of CI tests $\downarrow$}} \\
\cmidrule(lr){2-5} \cmidrule(lr){6-9}
& $r_L = 0.05$ & $r_L=0.10$ & $r_L=0.15$ & $r_L=0.20$
& $r_L = 0.05$ & $r_L=0.10$ & $r_L=0.15$ & $r_L=0.20$ \\
\midrule
PC
& 3.51 $\pm$ 0.59 & 3.65 $\pm$ 0.58 & 3.70 $\pm$ 0.64 & 3.75 $\pm$ 0.66
& 7349.8 $\pm$ 597.5 & 7789.8 $\pm$ 655.2 & 8081.3 $\pm$ 741.6 & 7988.3 $\pm$ 602.5 \\

FCI
& 1.88 $\pm$ 0.28 & 2.04 $\pm$ 0.30 & 2.22 $\pm$ 0.36 & 2.49 $\pm$ 0.60
& 10187.1 $\pm$ 1271.5 & 10875.5 $\pm$ 1066.0 & 11311.8 $\pm$ 1236.7 & 11368.4 $\pm$ 1785.4 \\

RFCI
& 1.83 $\pm$ 0.73 & 1.91 $\pm$ 0.68 & 1.95 $\pm$ 0.73 & 1.98 $\pm$ 0.69
& 9362.0 $\pm$ 988.7 & 9931.1 $\pm$ 902.9 & 10274.2 $\pm$ 929.0 & 10337.5 $\pm$ 1012.0 \\

FCI$^{+}$
& 3.71 $\pm$ 0.98 & 3.94 $\pm$ 1.11 & 3.82 $\pm$ 1.01 & 4.12 $\pm$ 1.26
& 12994.2 $\pm$ 2304.7 & 14118.0 $\pm$ 2180.9 & 14865.1 $\pm$ 2439.6 & 14427.6 $\pm$ 2270.3 \\

L-MARVEL
& 0.25 $\pm$ 0.06 & 0.27 $\pm$ 0.07 & 0.27 $\pm$ 0.07 & 0.27 $\pm$ 0.08
& 5476.8 $\pm$ 170.4 & 5496.9 $\pm$ 278.2 & 5454.9 $\pm$ 199.4 & 5479.0 $\pm$ 247.8 \\

ICD
& 16.02 $\pm$ 11.22 & 16.18 $\pm$ 8.53 & 24.73 $\pm$ 17.24 & 19.90 $\pm$ 12.89
& 8772.5 $\pm$ 992.7 & 9358.0 $\pm$ 1090.9 & 9645.2 $\pm$ 1049.6 & 9624.6 $\pm$ 1069.4 \\
\midrule
PCD-by-PCD
& 6.24 $\pm$ 2.90 & 6.98 $\pm$ 2.41 & 7.22 $\pm$ 2.12 & 8.16 $\pm$ 1.92
& 12395.7 $\pm$ 5467.3 & 14502.3 $\pm$ 4411.9 & 15474.8 $\pm$ 4219.0 & 16853.5 $\pm$ 3318.8 \\

MB-by-MB
& 2.46 $\pm$ 1.87 & 3.14 $\pm$ 2.11 & 2.58 $\pm$ 1.69 & 3.75 $\pm$ 3.22
& 5208.5 $\pm$ 4508.4 & 5928.6 $\pm$ 4582.5 & 5067.0 $\pm$ 3649.3 & 8132.0 $\pm$ 9952.7 \\

CMB
& 5.33 $\pm$ 5.57 & 5.36 $\pm$ 4.32 & 6.24 $\pm$ 5.94 & 6.54 $\pm$ 5.48
& 12544.6 $\pm$ 15395.1 & 12968.0 $\pm$ 13290.0 & 14293.8 $\pm$ 15664.4 & 14072.0 $\pm$ 12831.9 \\

PSL
& 1.93 $\pm$ 1.40 & 2.45 $\pm$ 1.62 & 2.41 $\pm$ 1.70 & 2.39 $\pm$ 1.40
& 3845.6 $\pm$ 3073.9 & 4809.2 $\pm$ 3860.5 & 5042.7 $\pm$ 4338.8 & 4753.2 $\pm$ 3408.7 \\

GraN-LCS
& 124.11 $\pm$ 67.15 & 137.94 $\pm$ 74.30 & 139.78 $\pm$ 95.75 & 129.08 $\pm$ 70.33
& -- & -- & -- & -- \\

LatentLCD
& 3.08 $\pm$ 1.54 & 3.37 $\pm$ 1.61 & 3.53 $\pm$ 1.74 & 2.92 $\pm$ 1.62
& 12275.5 $\pm$ 6620.6 & 13834.2 $\pm$ 7179.8 & 13616.9 $\pm$ 6988.5 & 12063.5 $\pm$ 6576.6 \\

\hdashline
\rowcolor{gray!10}
LoCaLS (Ours)
& \textbf{0.14 $\pm$ 0.05} & \textbf{0.15 $\pm$ 0.05} & \textbf{0.17 $\pm$ 0.07} & \textbf{0.20 $\pm$ 0.09}
& \textbf{515.0 $\pm$ 365.5} & \textbf{636.9 $\pm$ 452.9} & \textbf{611.7 $\pm$ 433.4} & \textbf{748.4 $\pm$ 580.1} \\
\bottomrule
\end{tabular}
}

\vspace{1mm}
\begin{minipage}{\textwidth}
\footnotesize
Note: ``--'' indicates that the number of CI tests is not applicable for the corresponding method.
The best result in each group is highlighted in bold.
\end{minipage}

\end{table}

\begin{figure}[H]
    \centering
        \begin{subfigure}{1.0\linewidth}
        \centering
        \includegraphics[width=\linewidth]{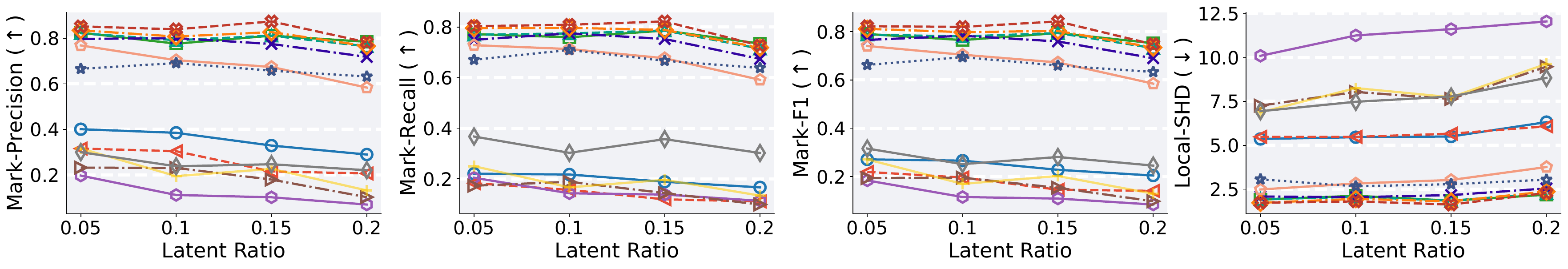}
    \end{subfigure}
        \begin{subfigure}{0.9\linewidth}
        \centering
        \includegraphics[width=\linewidth]{Figures/legend.pdf}
    \end{subfigure}

    \caption{Structural accuracy under varying latent ratios. The panels report Mark-Precision, Mark-Recall, Mark-F1, and Local-SHD, with the selection ratio fixed at $0.05$. The symbols $\uparrow$ and $\downarrow$ indicate that higher and lower values are better, respectively.} 
    \label{fig:ER-varying-latent-results}
\end{figure}

\FloatBarrier
\subsection{Supplementary Results on Selection-Variable Sensitivity} 
\label{app:supplementary-sensitivity-selection} 

This section provides supplementary results for the selection-variable sensitivity experiment in \cref{sec:sensitivity-selection}. Specifically, \cref{tab:vary-selection-ratio} summarizes the induced MAGs under different selection ratios, showing that the average degree changes only mildly as the number of selection variables increases. \cref{tab:vary-selection-ratio-efficiency} reports the detailed runtime and CI-test results, and \cref{fig:ER-varying-selection-results} reports the Mark-Precision, Mark-Recall, Mark-F1, and Local-SHD results.

\begin{table}[H]
\centering
\caption{Summary statistics under varying selection ratios. The latent ratio is fixed at $0.05$.} \label{tab:vary-selection-ratio}
\begin{tabular}{lC{0.3cm}C{0.3cm}C{0.3cm}C{7.0cm}}
\toprule
\textbf{Selection Ratio} & $|\vars[O]|$ & $|\vars[L]|$ & $|\vars[S]|$ & 
\textbf{Avg. Degree in Induced MAGs}  \\
\midrule
0.05  & 100  & 5  & 5  & 2.36$\pm$0.10  \\
0.10  & 100  & 5  & 10  & 2.43$\pm$0.04  \\
0.15  & 100  & 5  & 15  & 2.43$\pm$0.05  \\
0.20  & 100  & 5  & 20  & 2.48$\pm$0.11  \\
\bottomrule
\end{tabular}
\end{table}

\begin{table}[H]
\centering
\caption{Runtime and number of CI tests under varying selection ratios.}
\label{tab:vary-selection-ratio-efficiency}
\scriptsize
\setlength{\tabcolsep}{3pt}
\renewcommand{\arraystretch}{1.08}
\resizebox{\textwidth}{!}{
\begin{tabular}{lcccccccc}
\toprule
\multirow{2}{*}{\textbf{Algorithm}}
& \multicolumn{4}{c}{\textbf{Runtime (seconds) $\downarrow$}}
& \multicolumn{4}{c}{\textbf{Number of CI tests $\downarrow$}} \\
\cmidrule(lr){2-5} \cmidrule(lr){6-9}
& $r_S = 0.05$ & $r_S=0.10$ & $r_S=0.15$ & $r_S=0.20$
& $r_S = 0.05$ & $r_S=0.10$ & $r_S=0.15$ & $r_S=0.20$ \\
\midrule

PC
& 3.65 $\pm$ 0.68 & 3.33 $\pm$ 0.55 & 3.36 $\pm$ 0.57 & 3.39 $\pm$ 0.57
& 7585.6 $\pm$ 714.6 & 7564.8 $\pm$ 808.9 & 7416.0 $\pm$ 576.8 & 7532.3 $\pm$ 576.4 \\

FCI
& 2.12 $\pm$ 0.36 & 2.07 $\pm$ 0.30 & 2.09 $\pm$ 0.33 & 2.13 $\pm$ 0.22
& 10498.7 $\pm$ 1256.0 & 10082.2 $\pm$ 1342.8 & 9866.0 $\pm$ 1143.9 & 9824.1 $\pm$ 1110.7 \\

RFCI
& 2.90 $\pm$ 1.02 & 2.88 $\pm$ 0.98 & 2.89 $\pm$ 1.00 & 2.87 $\pm$ 0.98
& 9618.2 $\pm$ 1006.5 & 9336.8 $\pm$ 1064.6 & 9119.1 $\pm$ 914.1 & 9157.6 $\pm$ 883.5 \\

FCI$^{+}$
& 5.54 $\pm$ 1.35 & 5.49 $\pm$ 1.26 & 5.40 $\pm$ 1.23 & 5.05 $\pm$ 1.63
& 13449.2 $\pm$ 2294.4 & 12949.1 $\pm$ 2477.2 & 12460.5 $\pm$ 2240.3 & 12540.6 $\pm$ 2109.8 \\

L-MARVEL
& 0.25 $\pm$ 0.06 & 0.28 $\pm$ 0.08 & 0.29 $\pm$ 0.12 & 0.28 $\pm$ 0.09
& 5506.4 $\pm$ 239.6 & 5400.3 $\pm$ 168.6 & 5421.2 $\pm$ 250.3 & 5413.3 $\pm$ 187.2 \\

ICD
& 18.00 $\pm$ 12.27 & 13.52 $\pm$ 9.08 & 10.53 $\pm$ 6.84 & 10.37 $\pm$ 6.33
& 8935.0 $\pm$ 1099.4 & 8829.5 $\pm$ 1156.9 & 8595.4 $\pm$ 921.0 & 8792.2 $\pm$ 979.3 \\
\midrule
PCD-by-PCD
& 7.20 $\pm$ 1.95 & 6.30 $\pm$ 2.13 & 6.56 $\pm$ 1.68 & 6.23 $\pm$ 2.46
& 14173.6 $\pm$ 3421.5 & 12499.5 $\pm$ 4040.7 & 13107.1 $\pm$ 3326.5 & 12259.7 $\pm$ 4381.3 \\

MB-by-MB
& 3.78 $\pm$ 3.49 & 2.09 $\pm$ 1.73 & 2.81 $\pm$ 1.54 & 2.70 $\pm$ 2.22
& 8666.0 $\pm$ 11543.0 & 4637.8 $\pm$ 5040.3 & 5873.4 $\pm$ 3953.7 & 5718.4 $\pm$ 5245.6 \\

CMB
& 5.90 $\pm$ 6.20 & 4.73 $\pm$ 4.62 & 4.77 $\pm$ 4.23 & 4.73 $\pm$ 4.33
& 14597.8 $\pm$ 17670.4 & 12176.2 $\pm$ 14096.1 & 12281.2 $\pm$ 13659.3 & 11306.5 $\pm$ 11248.1 \\

PSL
& 2.31 $\pm$ 1.69 & 1.64 $\pm$ 1.11 & 2.01 $\pm$ 1.29 & 2.07 $\pm$ 1.45
& 4772.7 $\pm$ 3891.1 & 3577.7 $\pm$ 3123.2 & 4232.7 $\pm$ 3296.0 & 4163.4 $\pm$ 3217.9 \\

GraN-LCS
& 88.90 $\pm$ 31.59 & 104.78 $\pm$ 41.62 & 130.17 $\pm$ 58.22 & 111.36 $\pm$ 54.73
& -- & -- & -- & -- \\

LatentLCD
& 2.99 $\pm$ 1.47 & 2.93 $\pm$ 1.37 & 2.54 $\pm$ 1.41 & 2.76 $\pm$ 1.43
& 12515.8 $\pm$ 6502.9 & 11285.8 $\pm$ 5542.6 & 9724.4 $\pm$ 5448.3 & 10995.0 $\pm$ 6036.4 \\
\hdashline
\rowcolor{gray!10}
LoCaLS (Ours)
& \textbf{0.14 $\pm$ 0.06} & \textbf{0.13 $\pm$ 0.06} & \textbf{0.16 $\pm$ 0.06} & \textbf{0.15 $\pm$ 0.06}
& \textbf{468.6 $\pm$ 341.4} & \textbf{489.5 $\pm$ 346.9} & \textbf{548.5 $\pm$ 417.3} & \textbf{523.0 $\pm$ 393.3} \\

\bottomrule
\end{tabular}
}

\vspace{1mm}
\begin{minipage}{\textwidth}
\footnotesize
Note: ``--'' indicates that the number of CI tests is not applicable for the corresponding method.
The best result in each group is highlighted in bold.
\end{minipage}

\end{table}

\begin{figure}[H]
    \centering
        \begin{subfigure}{1.0\linewidth}
        \centering
        \includegraphics[width=\linewidth]{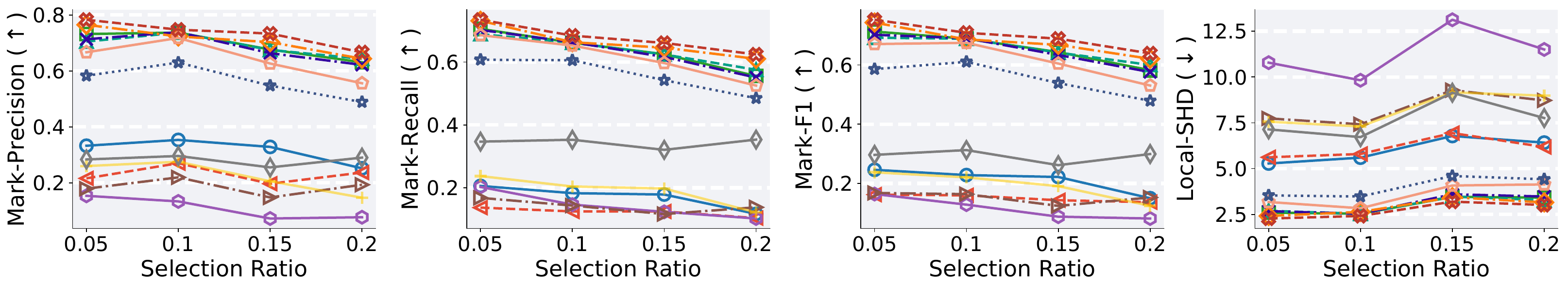}
    \end{subfigure}
        \begin{subfigure}{0.9\linewidth}
        \centering
        \includegraphics[width=\linewidth]{Figures/legend.pdf}
    \end{subfigure}

    \caption{Structural accuracy under varying selection ratios. The panels report Mark-Precision, Mark-Recall, Mark-F1, and Local-SHD, with the latent ratio fixed at $0.05$. The symbols $\uparrow$ and $\downarrow$ indicate that higher and lower values are better, respectively.} 
    \label{fig:ER-varying-selection-results}
\end{figure}

\FloatBarrier
\section{Supplementary Results for Gene Expression Applications}
\label{app:supplementary-real-datasets}

\subsection{Additional Results on \emph{Arabidopsis thaliana} Gene Expression Data}
\label{app:supplementary-real-data-Arabidopsis}

This appendix provides additional target-specific local structures learned from the \emph{Arabidopsis thaliana} gene expression data of Wille et al.~\citep{wille2004sparse}. 
The three targets complement the two representative examples shown in \cref{sec:real-data-Arabidopsis} and further illustrate how \method captures local pathway organization under possible latent regulatory effects and selection bias.

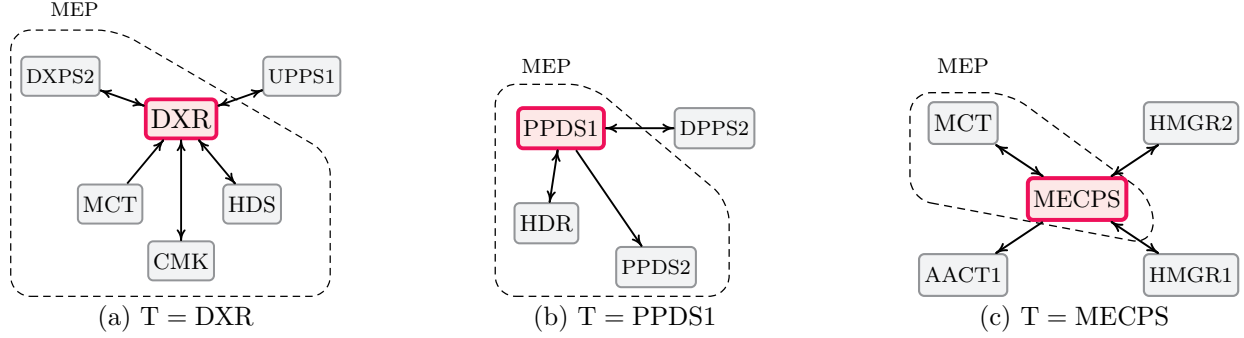
\begin{figure}[H]
    \centering
    \captionsetup[subfigure]{font=small, skip=2pt}

    \begin{subfigure}[t]{0.28\linewidth}
        \centering
        \resizebox{\linewidth}{!}{
        \begin{tikzpicture}
        
            \draw[densely dashed, rounded corners=8pt, line width=0.5pt]
                (-2.4,-2.5) -- (2.1,-2.5) -- (2.1,-0.50) --
                (-0.90,1.25) -- (-2.4,1.25) --
                (-2.4,0.35) -- cycle;
            \node[font=\scriptsize] at (-1.5,1.55) {MEP};

            \draw (0.0, 0.0) node(DXR) [genetarget, inner sep=2pt] {$\mathrm{DXR}$};

            \draw (-1.7, 0.6) node(DXPS2) [geneobs, inner sep=2pt, font=\scriptsize] {$\mathrm{DXPS2}$};
            \draw (-1.0, -1.2) node(MCT) [geneobs, inner sep=2pt, font=\footnotesize] {$\mathrm{MCT}$};
            \draw (0.0, -2.0) node(CMK) [geneobs, inner sep=2pt, font=\footnotesize] {$\mathrm{CMK}$};
            \draw (1.0, -1.2) node(HDS) [geneobs, inner sep=2pt, font=\footnotesize] {$\mathrm{HDS}$};
            \draw (1.7, 0.6) node(UPPS1) [geneobs, inner sep=2pt, font=\scriptsize] {$\mathrm{UPPS1}$};

            \draw[arcsq-arcsq]  (DXR) -- (DXPS2);
            \draw[-arcsq]       (MCT) -- (DXR);
            \draw[arcsq-arcsq]  (CMK) -- (DXR);
            \draw[arcsq-arcsq]  (DXR) -- (HDS);
            \draw[arcsq-arcsq]  (DXR) -- (UPPS1);
    
        \end{tikzpicture}
        }
        \caption{$\target = \mathrm{DXR}$}
        \label{fig:local-pag-DXR}
    \end{subfigure}
    \hfill
    \begin{subfigure}[t]{0.28\linewidth}
        \centering
        \resizebox{0.8\linewidth}{!}{
        \begin{tikzpicture}

            \draw[densely dashed, rounded corners=8pt, line width=0.5pt]
                (-2.2,-1.8) -- (1.0,-1.8) -- (1.0,-0.50) --
                (-0.90,1.1) -- (-2.2,1.1) --
                (-2.2,0.35) -- cycle;
            \node[font=\scriptsize] at (-1.5,1.35) {MEP};

            \draw (-1.3, 0.5) node(PPDS1) [genetarget, inner sep=2pt, font=\footnotesize] {$\mathrm{PPDS1}$};

            \draw (-1.5, -0.8) node(HDR) [geneobs, inner sep=2pt, font=\footnotesize] {$\mathrm{HDR}$};
            \draw (0.0, -1.4) node(PPDS2) [geneobs, inner sep=2pt, font=\scriptsize] {$\mathrm{PPDS2}$};
            \draw (0.8, 0.5) node(DPPS2) [geneobs, inner sep=2pt, font=\scriptsize] {$\mathrm{DPPS2}$};

            \draw[arcsq-arcsq]  (PPDS1) -- (HDR);
            \draw[-arcsq]       (PPDS1) -- (PPDS2);
            \draw[arcsq-arcsq]  (PPDS1) -- (DPPS2);
    
        \end{tikzpicture}
        }
        \caption{$\target = \mathrm{PPDS1}$}
        \label{fig:local-pag-PPDS1}
    \end{subfigure}
    \hfill
    \begin{subfigure}[t]{0.28\linewidth}
        \centering
        \resizebox{\linewidth}{!}{
        \begin{tikzpicture}

            \draw[densely dashed, rounded corners=8pt, line width=0.5pt]
                (-2.2,0.0) -- (1.0,-0.6) -- (1.0,-0.0) --
                (-0.90,1.4) -- (-2.2,1.4) -- cycle;
            \node[font=\scriptsize] at (-1.5,1.75) {MEP};

            \draw (0.0, 0.0) node(MECPS) [genetarget, inner sep=2pt, font=\footnotesize] {$\mathrm{MECPS}$};

            \draw (-1.5, 1.0) node(MCT) [geneobs, inner sep=2pt, font=\footnotesize] {$\mathrm{MCT}$};
            \draw (-1.5, -1.0) node(AACT1) [geneobs, inner sep=2pt, font=\scriptsize] {$\mathrm{AACT1}$};
            \draw (1.5, -1.0) node(HMGR1) [geneobs, inner sep=2pt, font=\scriptsize] {$\mathrm{HMGR1}$};
            \draw (1.5, 1.0) node(HMGR2) [geneobs, inner sep=2pt, font=\scriptsize] {$\mathrm{HMGR2}$};

            \draw[arcsq-arcsq]  (MECPS) -- (MCT);
            \draw[-arcsq]       (MECPS) -- (AACT1);
            \draw[arcsq-arcsq]  (MECPS) -- (HMGR1);
            \draw[arcsq-arcsq]  (MECPS) -- (HMGR2);
        \end{tikzpicture}
        }
        \caption{$\target = \mathrm{MECPS}$}
        \label{fig:local-pag-MECPS}
    \end{subfigure}

    \caption{
    Local causal structures learned for additional \emph{Arabidopsis thaliana} target genes.
    Dashed contours indicate genes from the same pathway group.
    (a) Target $\mathrm{DXR}$;
    (b) target $\mathrm{PPDS1}$;
    (c) target $\mathrm{MECPS}$.
    }
    \label{fig:appendix-Arabidopsis-three-targets}
\end{figure}

Comparing the supplementary local structures in \cref{fig:appendix-Arabidopsis-three-targets} with the known isoprenoid pathway organization and the graphical models reported by Wille et al.~\citep[Figure~3]{wille2004sparse}, we observe the following additional results.
\begin{itemize}[leftmargin=10pt,itemsep=2pt,topsep=2pt,parsep=0pt]
    \item
    \emph{(1) Intra-pathway Connections:}
    The recovered neighborhoods preserve several pathway-local relations reported in~\citep{wille2004sparse}. For the target $\mathrm{DXR}$, \method recovers $\mathrm{DXPS2}$, $\mathrm{MCT}$, $\mathrm{CMK}$, and $\mathrm{HDS}$ as neighbors, which are consistent with the core MEP chain shown in Wille et al.~\citep[Figure~3]{wille2004sparse}. For the downstream MEP-side target $\mathrm{PPDS1}$, the learned neighborhood contains $\mathrm{HDR}$, $\mathrm{PPDS2}$, and $\mathrm{DPPS2}$, which also agrees with the reference graph. For $\mathrm{MECPS}$, the recovered neighbor $\mathrm{MCT}$ lies in the same MEP module and is consistent with the reference structure.

    \item
    \emph{(2) Cross-pathway and Orientation Patterns:}
    The supplementary results also recover sparse links between pathway regions. In particular,
    $\mathrm{DXR}\--\mathrm{UPPS1}$, $\mathrm{PPDS1}\--\mathrm{DPPS2}$, 
    $\mathrm{MECPS}\--\mathrm{HMGR2}$, and $\mathrm{MECPS}\--\mathrm{AACT1}$ are compatible with the cross-compartment connections in Wille et al.~\citep[Figure~3]{wille2004sparse}, where the plastidial MEP pathway, cytosolic MVA pathway, and mitochondrial genes are not completely separated but interact through a small number of bridge genes. 
    This supports the view that \method captures not only pathway-local neighborhoods, but also biologically plausible cross-talk between compartments~\citep{laule2003crosstalk,wille2004sparse}.
    Most recovered edges are bi-directed, suggesting the possible presence of unmeasured common regulators or other latent factors in the gene expression data.
\end{itemize}

\subsection{Additional Results on Melanoma Gene Expression Data}
\label{app:supplementary-melanoma-gene-expression}

This appendix reports additional results for the melanoma gene expression application in \cref{sec:real-data-melanoma}. 
As in the main text, we report L-MARVEL and GraN-LCS as representative baselines; the other compared methods are omitted because a single run exceeds two hours on these datasets.

\begin{table}[H]
\centering
\caption{Comparison on melanoma gene expression data for target $\mathrm{HLA}$-$\mathrm{A}$.}
\label{tab:gene-target-performance-HLA-A}
\begin{threeparttable}
\begin{tabular}{l l C{2.5cm} C{2.5cm} C{2.0cm}}
\toprule
\textbf{Condition} & \textbf{Algorithm} & \textbf{Runtime $\downarrow$} & 
\textbf{\#CI Tests $\downarrow$} & \textbf{IVR $\uparrow$} \\
\midrule
\multirow{3}{*}{Co-culture} 
    & L-MARVEL & 1717.59 & 13064931 & 0.50 \\
    & GraN-LCS & 1897.23 & -- & 0.25 \\
    & \cellcolor{gray!10}LoCaLS (Ours) & \cellcolor{gray!10}\textbf{6.87} & \cellcolor{gray!10}\textbf{68108} & \cellcolor{gray!10}\textbf{0.79} \\
\midrule
\multirow{3}{*}{IFN-$\gamma$} 
    & L-MARVEL & 810.05 & 5632375 & 0.50 \\
    & GraN-LCS & 1916.61 & -- & 0.00 \\
    & \cellcolor{gray!10}LoCaLS (Ours) & \cellcolor{gray!10}\textbf{53.40} & \cellcolor{gray!10}\textbf{993871} & \cellcolor{gray!10}\textbf{0.52} \\
\midrule
\multirow{3}{*}{Control} 
    & L-MARVEL & 19.58 & 1194846 & \textbf{1.00} \\
    & GraN-LCS & 1277.83 & -- & 0.00 \\
    & \cellcolor{gray!10}LoCaLS (Ours) & \cellcolor{gray!10}\textbf{6.01} & \cellcolor{gray!10}\textbf{78693} & \cellcolor{gray!10}0.00 \\
\bottomrule
\end{tabular}
\vspace{1mm}
\begin{tablenotes}[flushleft]
\footnotesize
\item Note: 
The best result in each group is highlighted in bold.
\end{tablenotes}
\end{threeparttable}
\end{table}

\begin{table}[H]
\centering
\caption{Comparison on melanoma gene expression data for target $\mathrm{CD59}$.}
\label{tab:gene-target-performance-CD59}

\begin{threeparttable}
\begin{tabular}{l l C{2.5cm} C{2.5cm} C{2.0cm}}
\toprule
\textbf{Condition} & \textbf{Algorithm} & \textbf{Runtime $\downarrow$} & 
\textbf{\#CI Tests $\downarrow$} & \textbf{IVR $\uparrow$} \\
\midrule
\multirow{3}{*}{Co-culture} 
    & L-MARVEL & 1717.59 & 13064931 & 0.00 \\
    & GraN-LCS & 2129.69 & -- & 0.00 \\
    & \cellcolor{gray!10}LoCaLS (Ours) & \cellcolor{gray!10}\textbf{7.63} & \cellcolor{gray!10}\textbf{95995} & \cellcolor{gray!10}\textbf{0.89} \\
\midrule
\multirow{3}{*}{IFN-$\gamma$} 
    & L-MARVEL & 810.05 & 5632375 & 0.00 \\
    & GraN-LCS & 2056.80 & -- & 0.00 \\
    & \cellcolor{gray!10}LoCaLS (Ours) & \cellcolor{gray!10}\textbf{3.56} & \cellcolor{gray!10}\textbf{43579} & \cellcolor{gray!10}\textbf{1.00} \\
\midrule
\multirow{3}{*}{Control} 
    & L-MARVEL & 19.58 & 1194846 & 0.00 \\
    & GraN-LCS & 1360.88 & -- & 0.00 \\
    & \cellcolor{gray!10}LoCaLS (Ours) & \cellcolor{gray!10}\textbf{10.32} & \cellcolor{gray!10}\textbf{171202} & \cellcolor{gray!10}\textbf{0.73} \\
\bottomrule
\end{tabular}

\vspace{1mm}
\begin{tablenotes}[flushleft]
\footnotesize
\item Note: The best result in each group is highlighted in bold.
\end{tablenotes}
\end{threeparttable}

\end{table}

\textbf{Results.}
\cref{tab:gene-target-performance-HLA-A,tab:gene-target-performance-CD59} provide complementary evidence for two immune-relevant targets beyond $\mathrm{B2M}$. 
For $\mathrm{HLA}$-$\mathrm{A}$ and $\mathrm{CD59}$, \method achieves the highest IVR in nearly all conditions while requiring substantially fewer CI tests and much less runtime than the baselines. Although L-MARVEL obtains the highest IVR in the control condition for $\mathrm{HLA}$-$\mathrm{A}$, its IVR for $\mathrm{CD59}$ is $0.00$ across all three conditions. In contrast, \method obtains nonzero IVR for $\mathrm{CD59}$ under all three conditions, further supporting its robustness on immune-relevant targets.

These supplementary results are consistent with the main $\mathrm{B2M}$ analysis. 
Across antigen-presentation targets ($\mathrm{B2M}$ and $\mathrm{HLA}$-$\mathrm{A}$) and the complement-regulatory target $\mathrm{CD59}$, \method generally recovers local directed relations that are better supported by perturbation-induced distributional changes, while requiring substantially fewer CI tests than L-MARVEL and much less runtime than GraN-LCS. 
This suggests that the proposed \method remains effective for high-dimensional single-cell gene expression data, where latent regulatory factors and selection effects may be present.

\end{document}